\documentclass[sigconf, nonacm, pdfa]{acmart}

\newcommand{\babak}[1]{{\texttt{\color{red} Babak: [{#1}]}}}

\newcommand{\cg}{\mathcal{G}}
\newcommand{\qg}{\mathcal{H}}
\newcommand{\groundedQG}{\mathcal{G}_{\mathcal{H}}}

\newcommand{\edges}{\texttt{E}}

\newcommand{\nodes}{\texttt{V}}
\newcommand{\child}{\texttt{ch}}
\newcommand{\desc}{\texttt{Dsc}}
\newcommand{\nondesc}{\texttt{NDsc}}
\newcommand{\ancs}{\texttt{Ancs}}

\newcommand{\bX}{\boldsymbol{X}}
\newcommand{\bY}{\boldsymbol{Y}}
\newcommand{\bW}{\boldsymbol{W}}
\newcommand{\bZ}{\boldsymbol{Z}}
\newcommand{\bU}{\boldsymbol{U}}

\newcommand{\comp}{\{\mathcal{G}_i\}_{\qg}}
\newcommand{\mb}{\mathbf}
\newcommand{\pr}{\mathbb{P}}
\newcommand{\indep}{\perp\!\!\!\perp}
\newcommand{\probName}{causal DAG summarization}
\newcommand{\groundedDAG}{\text{canonical causal DAG}}
\newcommand{\algoName}{\textsc{CaGreS}}
\newcommand{\ci}{\textsc{CIs}}
\newcommand{\ksnap}{\textcolor{darkgreen}{\textsc{k-Snap}}}
\newcommand{\brutef}{\textcolor{red}{\textsc{Brute-Force}}}
\newcommand{\rand}{\textcolor{darkpurple}{\textsc{Random}}}
\newcommand{\tc}{\textcolor{darkyellow}{\textsc{Transit-Cluster}}}
\newcommand{\cic}{\textcolor{gray}{\textsc{CIC}}}
\newcommand{\flights}{\textsc{Flights}}

\newcommand{\adult}{\textsc{Adult}}
\newcommand{\german}{\textsc{German}}
\newcommand{\accidents}{\textsc{Accidents}}
\newcommand{\urls}{\textsc{Urls}}
\newcommand{\redshift}{\textsc{Redshift}}
\definecolor{darkgreen}{rgb}{0.0, 0.5, 0.0}
\definecolor{darkyellow}{rgb}{0.8, 0.6, 0.0}
\definecolor{darkpurple}{rgb}{0.4, 0.0, 0.4}

\usepackage{hhline} 

\usepackage{xcolor}
\definecolor{yellow}{rgb}{1,1,0.6}
\definecolor{orange}{rgb}{1,0.467,0}
\definecolor{lightyellow}{rgb}{1,1,0.8}

\definecolor{revablue}{RGB}{0, 90, 190}
\newcommand{\reva}[1]{{\leavevmode\color{black}{#1}}}
\definecolor{revbmagenta}{RGB}{220, 60, 150}
\newcommand{\revb}[1]{{\leavevmode\color{black}{#1}}}
\definecolor{revcpurple}{RGB}{140, 30, 160}
\newcommand{\revc}[1]{{\leavevmode\color{black}{#1}}}
\newcommand{\common}[1]{{\leavevmode\color{black}{#1}}}

\usepackage[colorinlistoftodos]{todonotes} 
\newtheorem{theorem}{Theorem}[section]
\newtheorem{corollary}{Corollary}[theorem]
\newtheorem{lemma}[theorem]{Lemma}

\theoremstyle{definition}
\newtheorem{definition}{Definition}
\newtheorem{problem}{Problem}
\newtheorem{example}{Example}
\usepackage{comment}
\usepackage{circledsteps}
\usepackage{listings}
\usepackage[sharp]{easylist}

\usepackage{amsmath}
\usepackage{graphicx}
\usepackage{subcaption}
\usepackage{tikz}
\usepackage{circuitikz}
\usetikzlibrary{arrows,shapes,positioning}
\usepackage{algorithmic}
\usepackage[ngerman,english]{babel}
\usepackage{textcomp}
\usepackage{multirow}
\usepackage{subcaption}
\usepackage{enumitem}
\usepackage{url}
\usepackage{tabularx}
\usepackage[vlined,commentsnumbered,ruled]{algorithm2e}
\usepackage{enumitem}

\usepackage[normalem]{ulem}

\newcommand{\brit}[1]{\textcolor{teal}{\bf (Brit)~[#1]}{\typeout{#1}}}
\newcommand{\anna}[1]{\textcolor{brown}{\bf (Anna)~[#1]}{\typeout{#1}}}

\newcommand{\node}[1]{\uppercase{#1}}

\def\e#1{\emph{#1}}
\newcommand{\eat}[1]{}
\def\set#1{\mathord{\{#1\}}}

\def\sigmarb{\Sigma_{\text{RB}}}
\def\eqdef{\mathrel{\stackrel{\textsf{\tiny def}}{=}}}

\newtheorem{defn}{Definition}[section]
\newtheorem{thm}[defn]{Theorem}

\usepackage{bbold}
 
\newcommand\vldbdoi{10.14778/3725688.3725717}

\renewcommand\footnotetextcopyrightpermission[1]{} 

\newcommand\vldbavailabilityurl{https://github.com/TechnionTDK/causalens}

\begin{document}



\title{Causal DAG Summarization (Full Version)}

\author{Anna Zeng}
\email{annazeng@mit.edu}
\affiliation{
  \institution{CSAIL, MIT}
  \country{USA}
}
\author{Michael Cafarella}
\email{michjc@csail.mit.edu}
\affiliation{
  \institution{CSAIL, MIT}
  \country{USA}
}
\author{Batya Kenig}
\email{batyak@technion.ac.il}
\affiliation{
  \institution{Technion}
  \country{Israel}
}

\author{Markos Markakis}
\email{markakis@mit.edu}
\affiliation{
  \institution{CSAIL, MIT}
  \country{USA}
}

\author{Brit Youngmann}
\email{brity@technion.ac.il}
\affiliation{
  \institution{Technion}
  \country{Israel}
}

\author{Babak Salimi}
\email{bsalimi@ucsd.edu}
\affiliation{
  \institution{University of California, San Diego}
  \country{USA}
}

\begin{abstract}
Causal inference aids researchers in discovering cause-and-effect relationships, leading to scientific insights. Accurate causal estimation requires identifying confounding variables to avoid false discoveries. Pearl’s causal model uses causal DAGs to identify confounding variables, but incorrect DAGs can lead to unreliable causal conclusions. 
However, for high dimensional data, the causal DAGs are often complex beyond human verifiability. Graph summarization is a logical next step, but current methods for general-purpose graph summarization are inadequate for causal DAG summarization. 
This paper addresses these challenges by proposing a causal graph summarization objective that balances graph simplification for better understanding while retaining essential causal information for reliable inference.
We develop an efficient greedy algorithm and show that summary causal DAGs can be directly used for inference and are more robust
to misspecification of assumptions, enhancing robustness for causal
inference.
Experimenting with six real-life datasets, we compared our algorithm to three existing solutions, showing its effectiveness in handling high-dimensional data and its ability to generate summary
DAGs that ensure both reliable causal inference and robustness against misspecifications.

\end{abstract}

\maketitle



\pagenumbering{arabic} 

\section{Introduction}
\label{sec:intro}

Causal inference is central to informed decision-making in economics, sociology, medicine, and in helping analysts unravel complex cause-effect relationships~\cite{varian2016causal,gangl2010causal,kleinberg2011review}. It has become increasingly critical in machine learning, where it supports algorithmic fairness~\cite{salimi2019interventional}, data debiasing~\cite{zhu2023consistent, zhu2024overcoming, zhu2023consistent}, explainable AI~\cite{galhotra2021explaining, miller2019explanation, beckers2022causal, mohan2013graphical}, and enhanced robustness~\cite{magliacane2018domain, sun2021recovering, scholkopf2021toward}. \reva{Causal inference has also become a major theme in recent data management research~\cite{meliou2010causality, bertossi2017causes, meliou2014causality,salimi2020causal,10.14778/3554821.3554902}, integrating causality into data management tasks such as finding input responsibilities toward query answers~\cite{meliou2010causality, meliou2009so, meliou2014causality}, explaining for query results \cite{salimi2018bias,youngmann2024summarized,youngmann2023explaining,roy2014formal}, data discovery~\cite{galhotra2023metam,youngmann2023causal}, data cleaning~\cite{pirhadi2024otclean,salimi2019interventional}, hypothetical reasoning \cite{galhotra2022hyper}, and large system diagnostics~\cite{markakis2024sawmill,causalsim,sage, gudmundsdottir2017demonstration,markakis2024logs}. }

Drawing causal conclusions from data fundamentally hinges on access to background knowledge and assumptions, as data alone cannot establish causality~\cite{pearl_causality_2000,rubin2005causal}. A principled way to encode such background knowledge is through Causal Directed Acyclic Graphs (DAGs)~\cite{pearl_causality_2000}. These graphs explicitly represent assumed causal relationships, enabling systematic reasoning about interventions. Causal DAGs can be used together with graphical criteria such as the backdoor criterion, or in general, Pearl's $do$-calculus~\cite{pearl_causality_2000} to determine whether the effect of interventions can be answered using data and available background knowledge. If so, they help identify the right set of confounding variables to control for, ensuring sound causal inference given the background knowledge.

However, the soundness and robustness of causal inference hinges on the availability of high-quality causal DAGs, which are often not readily available. These DAGs are typically constructed using domain knowledge~\cite{castelnovo2024marrying,vashishtha2023causal,markakis2024press} or through causal discovery methods~\cite{glymour2019review,shimizu2006linear,chickering2002optimal, wiering2002evolving,zhu2019causal}. This elicitation process is costly, error-prone~\cite{oates2017repair}, and time-consuming. Causal discovery methods, while useful, are fundamentally restrictive as they identify a class of DAGs compatible with observed data rather than a singular, definitive model~\cite{glymour2019review}. Moreover, existing discovery methods often do not perform well on real-world data and require significant human intervention for verification~\cite{singh2017comparative, huegle2021manm, constantinou2021large}. The problem is even worse for high-dimensional data, increasing the need for efficient methods to simplify and verify causal models while retaining essential information~\cite{o2006incorporating}. We illustrate this with an example:

\begin{example}
\label{ex:ex1} 
Consider the application of performance diagnosis for a cloud-based data warehouse service. Specifically, consider a dataset collected from the monitoring views in Amazon Redshift Serverless~\cite{redshift-serverless}, including performance metrics and query-extracted features, such as the number of unique tables and columns \revc{referenced in the executed query}.
This dataset enables answering crucial causal queries for optimizing performance. For example, understanding the impact of caching on latency (i.e., \verb|Result Cache Hit| on \verb|Elapsed Time|) can help tune caching mechanisms, or analyzing the effect of join complexity on the query planner's performance (i.e., \verb|Num Joins| on \verb|Plan Time|) can optimize query execution strategies. 
However, the necessary causal DAG to answer such questions is not readily available, and getting it right is non-trivial.


Figure~\ref{fig:example_causal_dag} shows an example causal DAG covering variables from just one monitoring view~\cite{sysqueryhistory} and a few query features, chosen for illustration. This is just a small part \common{of the overall high dimensional dataset.}
\revc{Edges in causal DAGs represent potential cause-effect relationships. In our example, for instance, the edge from} \texttt{Num Columns} \revc{to} \texttt{Exec. Time} \revc{suggests that the number of columns referenced in a query may influence the query's execution time.}

To answer the above causal queries, \verb|Query Template| is a critical confounder that must be adjusted for because it influences both the performance metrics (e.g., \verb|Elapsed| \verb|Time|, \verb|Plan Time|) and the analyzed mechanisms (e.g., \verb|Result| \verb|Cache| \verb|Hit|, \verb|Num Joins|). Failing to adjust for this variable can lead to biased estimations and incorrect conclusions. Hence, any possible misspecification in the causal DAG that would fail to identify this variable as a confounder would result in incorrect effect estimations. Such sensitivity to graph errors makes domain expert verification essential for each existing or missing edge. This task can be overwhelming, even in this small example with only $12$ nodes, as it involves inspecting $66$ potential edges, one per pair of nodes. In the full dataset, the number of variables would be much higher, further complicating the task. \qed




\label{example:example_causal_dag} \end{example}

Graph summarization is a logical next step, as it reduces the number of nodes and edges, making it easier for users to verify and inspect causal DAGs in high-dimensional datasets.
Graph summarization has been extensively studied, with state-of-the-art methods designed to efficiently generate concise representations aimed at minimizing reconstruction errors~\cite{lee2020ssumm, yong2021efficient}, or facilitating accurate query answering~\cite{shin2019sweg, maccioni2016scalable}. 
However, we argue that while general-purpose methods are adept at managing massive graphs, they are inadequate for summarizing causal graphs, a task that demands the preservation of causal information crucial for reliable inference.

In this paper, we propose a graph summarization technique tailored for causal inference. It simplifies high-dimensional causal DAGs into manageable forms without compromising essential causal information, thereby improving interpretability. 
Our approach introduces a causal DAG summarization objective, which balances simplifying the graph for enhanced comprehensibility and retaining essential causal information. 
Using our technique, one can summarize an initial causal DAG (constructed using partial domain knowledge or causal discovery) for simpler verification and elicitation. \common{Additionally, the summary causal DAG can be directly used for causal inference and is more robust to misspecification of assumptions.} Our approach thereby improves {\em interpretability}, {\em verifiability}, and {\em robustness} in causal inference, facilitating the adoption of these techniques in practice.
We illustrate this with an example:

\begin{example}
Consider Fig.~\ref{fig:summary_graphs_ssumm}, which shows the summary graph generated by SSumM~\cite{lee2020ssumm} for the causal DAG of Fig.~\ref{fig:example_causal_dag}. SSumM is a top-performing general-purpose graph summarization method that effectively balances conciseness and reconstruction accuracy. However, the generated graph can no longer be interpreted as a causal DAG, since it exhibits cycles and self-loops.
For example, computing the causal effect of \verb|Num Joins| on \verb|Plan time| is impossible due to the bidirectional edge between their cluster nodes.
Other methods (e.g., ~\cite{yong2021efficient}) exhibit similar weaknesses, making them unsuitable for summarizing causal DAGs.
An in-depth comparison with another graph summarization method~\cite{tian2008efficient} is provided in Section \ref{sec:exp}. We show that although this method can be adapted to generate summary DAGs compatible with causal inference principles, it does not optimally preserve critical causal information, reducing the accuracy of the inference over the summary DAG.


In contrast, Fig.\ref{fig:summary_graphs_cagres} shows the 5-node summary DAG generated by our approach, which preserves critical causal information, offering a more interpretable summary that can be directly used for inference. This summary DAG makes it easier to verify the soundness of assumptions it encodes.
Furthermore, this summary DAG is inherently more robust to misspecification, because our summarization process creates a summary DAG compatible with a \emph{set} of possible initial DAGs. Hence, even if the original causal DAG missed an edge, our summarization algorithm can still create the necessary connections and maintain causal integrity. 
Using the summary DAG for inference intuitively leads to a more conservative set of confounders: it may lead to adjusting for redundant attributes, but they will only be ones that do not hurt the analysis. \qed

\label{example:example_causal_dag_summary} \end{example}

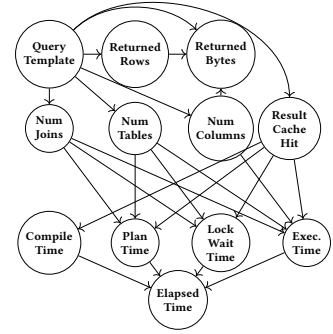
\begin{figure}
\vspace{-5mm}
\centering

 \scriptsize
 \resizebox{0.5\linewidth}{!}{
       \begin{tikzpicture}[
    every node/.style={
        minimum size=0.3cm, 
        font=\bfseries, 
        align=center, 
    },
    vertex/.style={draw, circle}
]

\node[vertex] at (0,0) (Temp) {\scriptsize Query\\Template};
\node[vertex] at (1.5,0) (Rows) {\scriptsize Returned\\Rows};
\node[vertex] at (3,0) (Bytes) {\scriptsize Returned\\Bytes};

\node[vertex] at (0,-1.3) (Joins) {\scriptsize Num\\Joins};
\node[vertex] at (1.5,-1.3) (Tables) {\scriptsize Num\\Tables};
\node[vertex] at (3,-1.3) (Columns) {\scriptsize Num\\Columns};
\node[vertex] at (4.2,-1.3) (Columns) (Cache) {\scriptsize Result\\Cache\\Hit};

\node[vertex] at (0,-3.3) (Compile) {\scriptsize Compile\\Time};
\node[vertex] at (1.5,-3.3)(Plan) {\scriptsize Plan\\Time};
\node[vertex] at (3,-3.3)(Lock) {\scriptsize Lock\\Wait\\Time};
\node[vertex] at (4.5,-3.3) (Execute) {\scriptsize Exec.\\Time};

\node[vertex] at (2.25,-4.3) (Elapsed) {\scriptsize Elapsed\\Time};

\draw[->] (Temp) -- (Rows);
\draw[->] (Temp) to[out=45, in=135, looseness=0.9] (Bytes);
\draw[->] (Rows) -- (Bytes);
\draw[->] (Temp) -- (Tables);
\draw[->] (Temp) -- (Columns);
\draw[->] (Temp) -- (Joins);

\draw[->] (Tables) -- (Plan);
\draw[->] (Joins) -- (Plan);
\draw[->] (Tables) -- (Lock);
\draw[->] (Joins) -- (Lock);
\draw[->] (Tables) -- (Execute);
\draw[->] (Joins) -- (Execute);
\draw[->] (Columns) -- (Execute);
\draw[->] (Columns) -- (Bytes);

\draw[->] (Temp) to[out=45, in=90, looseness=0.95]  (Cache);
\draw[->] (Cache) -- (Compile);
\draw[->] (Cache) -- (Plan);
\draw[->] (Cache) -- (Execute);
\draw[->] (Cache) -- (Lock);

\draw[->] (Compile) -- (Elapsed);
\draw[->] (Plan) -- (Elapsed);
\draw[->] (Execute) -- (Elapsed);
\draw[->] (Lock) -- (Elapsed);

\end{tikzpicture}}
\vspace{-3mm}
\caption{Example causal DAG} \label{fig:example_causal_dag}
\vspace{-5mm}
\end{figure}

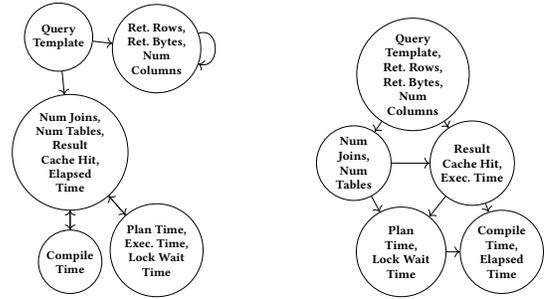
\begin{figure}[t] \scriptsize \centering
    \begin{subfigure}[b]{0.23\textwidth} \centering
    \resizebox{0.7\linewidth}{!}{
            \begin{tikzpicture}[
                node distance=0.5cm,
                every node/.style={
                    minimum size=0.8cm, 
                    font=\bfseries, 
                    align=center, 
                },
                vertex/.style={draw, circle}
            ]
            
            \node[vertex] at (-0.2,0.2) (Temp) {\scriptsize Query\\Template};
            \node[vertex] at (1.5,0) (Rows) {\scriptsize Ret. Rows,\\Ret. Bytes,\\ Num\\Columns};
            \node[vertex]  at (0, -1.8) (Joins) {\scriptsize Num Joins,\\Num Tables,\\Result\\ Cache Hit,\\Elapsed\\Time};
            \node[vertex] at (0,-3.7) (Compile) {\scriptsize Compile\\Time};
            \node[vertex] at (1.5, -3.5) (Plan) {\scriptsize Plan Time,\\Exec. Time,\\Lock Wait\\Time};
            
            \draw[->] (Temp) -- (Rows);
            \draw[->] (Temp) -- (Joins);                        
            \draw[->] (Rows) to[out=20, in=-20, looseness=2] (Rows);
            \draw[->] (Joins) -- (Compile);
            \draw[->]  (Compile) -- (Joins);         
            \draw[->] (Joins) -- (Plan);
            \draw[->] (Plan) -- (Joins);
        \end{tikzpicture}  }
        \caption{Problematic Summary Graph}
        \label{fig:summary_graphs_ssumm}
    \end{subfigure}
    \begin{subfigure}[b]{0.23\textwidth} \centering
    \resizebox{0.75\linewidth}{!}{
       \begin{tikzpicture}[
            node distance=0.5cm,
            every node/.style={
                minimum size=0.7cm, 
                font=\bfseries, 
                align=center, 
            },
            vertex/.style={draw, circle}
        ]

            \node[vertex] at (1,0) (Temp) {\scriptsize Query\\Template,\\Ret. Rows,\\Ret. Bytes,\\ Num\\Columns};
            \node[vertex] at (0, -1.5) (Joins) {\scriptsize Num\\Joins,\\Num\\Tables};
            \node[vertex] at (2, -1.5)(Cache) {\scriptsize Result\\Cache Hit,\\Exec. Time};
            \node[vertex] at (0.8, -3) (Plan) {\scriptsize Plan\\Time,\\Lock Wait\\Time};
            \node[vertex] at (2.5,-3) (Comp) {\scriptsize Compile\\Time,\\Elapsed\\Time};
            
            \draw[->] (Temp) -- (Joins);
            \draw[->] (Temp) -- (Cache);
            \draw[->] (Joins) -- (Cache);
            \draw[->] (Joins) -- (Plan);
            \draw[->] (Cache) -- (Plan);
            \draw[->] (Cache) -- (Comp);
            \draw[->] (Plan) -- (Comp);
        \end{tikzpicture}}
        \caption{Our  Summary DAG}
        \label{fig:summary_graphs_cagres}
    \end{subfigure}
    \vspace{-3mm}
    \caption{$5$-node summary graphs for the DAG in Fig. \ref{fig:example_causal_dag}. }
    \label{fig:summary_graphs}
    \vspace{-4mm}
\end{figure}

Our main contributions are summarized as follows.

\noindent
{\bf Causal DAG Summarization}. We introduce the problem of summarizing causal DAGs in a way that preserves their utility for reliable causal inference (Section \ref{sec:problem}). This necessitates preserving the causal information encoded in the input DAG. Causal DAGs encode information through missing edges, which imply Conditional Independence (CI) constraints. 
We therefore formalize causal DAG summarization as finding a summary DAG that preserves CI statements to the greatest extent possible, while meeting a node number constraint. We prove that this problem is NP-hard.

\noindent
{\bf Summary Causal DAGs}.
We introduce the concept of \emph{summary causal DAGs}, derived by grouping nodes within the original DAG via \emph{node contractions}. Despite inherently leading to information loss, node contraction enables summary DAGs to compactly encapsulate potential causal DAGs from which the summary DAG could have originated. 
We show that contracting nodes is akin to adding edges to the input causal DAG. 
Based on this connection, we develop a sound and complete algorithm for identifying all \ci\ encoded by a summary DAG. \common{This connection is crucial for utilizing summary causal DAGs for causal inference. (Section \ref{sec:csep}).}

\noindent
{\bf The \algoName\ Algorithm}.
\reva{We devise an efficient heuristic greedy algorithm called \algoName}. 
A key feature of \algoName\ is its approach to choosing which node pair to contract. This process is informed by the connection between node contraction and the addition of edges to the input DAG,
prioritizing node pairs that add the fewest edges upon contraction. Additionally, \algoName\ incorporates several optimizations, including caching mechanisms, making it a practical tool for generating summary causal DAGs (Section~\ref{sec:algorithm}).


\noindent
{\bf Causal Inference over Summary Causal DAGs}
We show that summary causal DAGs can be directly utilized for causal inference. We establish that Pearl's \emph{do-calculus} framework~\cite{pearl_causality_2000}, which provides a set of sound and complete rules for reasoning about the effects of interventions using causal DAGs, remains sound and complete for summary DAGs. 
By examining the connection between node contractions and the addition of edges, we offer clear insights into how these modifications affect the soundness and completeness of do-calculus within the framework of summary DAGs 
(Section \ref{sec:identifiability}).

\noindent
{\bf Experimental Evaluation}
We demonstrate how summary DAGs offer robustness against errors in the input causal DAG (Section \ref{sec:robustness}).
We conduct extensive experiments over six datasets demonstrating the effectiveness of \algoName\ compared to existing solutions and two variations of \algoName. The results show the efficiency of \algoName\ in handling high-dimensional datasets and its ability to generate summary DAGs that ensure reliable inference (Section \ref{sec:exp}).

\section{Background}
\label{sec:background}




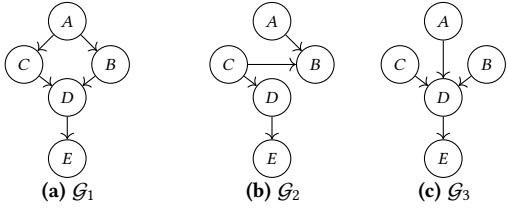
\begin{figure}[t] \scriptsize \centering
    \begin{subfigure}[b]{0.15\textwidth} \centering
        \begin{tikzpicture}[node distance=0.5cm,every node/.style={minimum size=0.5cm}]
            \tikzset{vertex/.style = {draw, circle}}

            \node[vertex] (A) {$A$};
            \node[vertex, below right=0.3cm of A] (B) {$B$};
            \node[vertex, below left=0.3cm of A] (C) {$C$};
            \node[vertex, below=0.5cm of A] (D) {$D$};
            \node[vertex, below=0.3cm of D] (E) {$E$};

            \draw[->] (A) -- (B);
            \draw[->] (A) -- (C);
            \draw[->] (C) -- (D);
            \draw[->] (B) -- (D);
            \draw[->] (D) -- (E);
        \end{tikzpicture}
         \vspace{-2mm}
        \caption{$\mathcal{G}_1$}
        \label{fig:g1}
    \end{subfigure}
    \begin{subfigure}[b]{0.15\textwidth} \centering
        \begin{tikzpicture}[node distance=0.5cm,every node/.style={minimum size=0.5cm}]
            \tikzset{vertex/.style = {draw, circle}}

            \node[vertex] (A) {$A$};
            \node[vertex, below right=0.3cm of A] (B) {$B$};
            \node[vertex, below left=0.3cm of A] (C) {$C$};
            \node[vertex, below=0.5cm of A] (D) {$D$};
            \node[vertex, below=0.3cm of D] (E) {$E$};

            \draw[->] (A) -- (B);
            \draw[->] (C) -- (B);
            \draw[->] (C) -- (D);
            \draw[->] (D) -- (E);
        \end{tikzpicture}
         \vspace{-2mm}
        \caption{$\mathcal{G}_2$}
        \label{fig:g2}
    \end{subfigure}
    \begin{subfigure}[b]{0.10\textwidth} \centering
        \begin{tikzpicture}[node distance=0.5cm,every node/.style={minimum size=0.5cm}]
            \tikzset{vertex/.style = {draw, circle}}
    
            \node[vertex] (A) {$A$};
            \node[vertex, below right=0.3cm of A] (B) {$B$};
            \node[vertex, below left=0.3cm of A] (C) {$C$};
            \node[vertex, below=0.5cm of A] (D) {$D$};
            \node[vertex, below=0.3cm of D] (E) {$E$};
    
            \draw[->] (A) -- (D);
            \draw[->] (C) -- (D);
            \draw[->] (B) -- (D);
            \draw[->] (D) -- (E);
        \end{tikzpicture}
        \vspace{-2mm}
        \caption{$\mathcal{G}_3$}
        \label{fig:g3}
    \end{subfigure}
    \vspace{-4mm}
    \caption{
   Three causal DAGs over the same set of nodes.
    }
    \label{fig:compatible_graphs}
    \vspace{-6mm}
\end{figure}

\vspace{1mm}
We consider a single-relation database over a schema $\mathbb{A}$.
We use upper case letters to denote a variable from $\mathbb{A}$ and bold symbols for sets of variables. 
The broad goal of causal inference is to estimate the effect of an {\em exposure variable} $T {\in} \mathbb{A}$ on an \emph{outcome variable} $O {\in} \mathbb{A}$.
We use Pearl's model for causal inference on observational data~\cite{pearl_causality_2000}.

To get an unbiased estimate for the causal effect of the exposure $T$ on the outcome $O$, one must mitigate the effect of {\em confounding variables}, i.e., variables that can affect the exposure assignment and outcome~\cite{pearl_causality_2000}.
For instance, when estimating how query execution time affects the elapsed time, one would avoid a source of \emph{confounding bias} by considering the number of columns and tables. 
Pearl's model provides ways to account for confounding variables to get an unbiased causal estimate using \emph{causal DAGs}~\cite{pearl_causality_2000}. Causal DAGs provide a simple way of representing causal relationships within a set of variables.
A causal DAG $\cg$ for the variables in $\mathbb{A}$ is a specific type of a Bayesian network and is formally defined as follows:

\noindent
{\bf Causal DAG}. 
A Bayesian network is a DAG $\cg$ in which nodes represent random variables and
edges express direct dependence between the variables. Each node $X_i$ is associated
with the conditional distribution $\pr(X_i
|\pi(X_i))$, where $\pi(X_i)$ is the set of parents of $X_i$ in $\cg$.
The joint distribution over all variables $\pr(X_1,{\dots},X_n)$,
is given by the product of all conditional distributions. That is,
\begin{equation}
	\label{eq:BN}
\setlength{\abovedisplayskip}{-8pt}
\setlength{\belowdisplayskip}{-8pt}
 \small
	\pr(X_1,\dots,X_n)= \prod_{i=1}^n \pr(X_i|\pi(X_i))
\end{equation}
A causal DAG is a Bayesian network where edges signify direct causal influence rather than statistical dependence. We say that $X$ is a potential cause of $Y$ if there is a directed path from $X$ to $Y$. 

\noindent
\subsubsection*{\bf \revb{$d$-Separation}}
\revb{$d$-separation is a criterion in a causal DAG that determines whether two sets of nodes are conditionally independent, given a third set, by checking whether all paths between the sets are ``blocked'' based on specific structural rules. If two sets of nodes are $d$-separated, by definition it means that all paths connecting them are blocked by other nodes. Formally, }
a \e{trail} $t{=}(X_1,{\dots}$, $X_n)$ is a sequence of nodes s.t. there is a a distinct edge between $X_i$ and $X_{i{+}1}$ for every $i$. That is, $(X_i{\rightarrow} X_{i+1}){\in} \edges(\cg)$ or \reva{$(X_i {\leftarrow} X_{i+1}){\in} \edges(\cg)$} for every $i$. A node $X_i$ is said to be \e{head-to-head} with respect to $t$ if $(X_{i-1}{\rightarrow} X_i){\in} \edges(\cg)$ and $(X_{i}{\leftarrow} X_{i+1}){\in} \edges(\cg)$.
A trail $t=(X_1,{\dots},X_n)$ is \e{active} given $\mb{Z}{\subseteq} \mathcal{X}$ if (1) every $X_i$ that is a head-to-head node with respect to $t$ either belongs to $\mb{Z}$ or has a descendant in $\mb{Z}$, and (2) every $X_i$ that is not a head-to-head node w.r.t. $t$ does not belong to $\mb{Z}$. If a trail $t$ is not active given $\mb{Z}$, then it is \e{blocked} given $\mb{Z}$~\cite{pearl_causality_2000}. \revb{In a DAG, two sets of nodes $\mb{X}$ and 
$\mb{Y}$ are $d$-separated by a third set of nodes $\mb{Z}$ if all trails connecting $\mb{X}$ and $\mb{Y}$ are blocked by $\mb{Z}$.}

\subsubsection*{\bf Conditional Independence}
Causal DAGs encode a set of
Conditional Independence statements (\ci) that can be read off the graph using $d$-separation~\cite{pearl_causality_2000}.
These statements describe the absence of an active trail between two sets of variables when conditioning on other variables. \revb{If two sets of nodes $\mb{X}$ and 
$\mb{Y}$ are $d$-separated by $\mb{Z}$, then $\mb{X}$ and 
$\mb{Y}$ are conditionally independent given $\mb{Z}$.}

\begin{example}
\label{ex:CIs}
   \common{ Examples of \ci\ encoded in the causal DAG depicted in Fig. \ref{fig:compatible_graphs}(a) include: ($B \indep_{d} C \mid A$), and ($D \indep_{d} A \mid BC$). \qed}
\end{example}

\subsubsection*{\bf CIs \& Missing Edges}
In causal DAGs, the information encoded by missing edges implies the set of \ci\ the DAG represents. 
Namely, removing edges can undermine the causal model as it implies \ci\ that do not necessarily hold in the distribution. On the other hand, existing edges indicate \emph{potential} causal dependence. This implies that adding edges to a causal DAG, provided acyclicity is maintained, does not necessarily compromise validity~\cite{pearl_causality_2000}.


\subsubsection*{\bf The Recursive Basis}
The \emph{Recursive Basis} (RB)~\cite{DBLP:journals/networks/GeigerVP90} for a causal DAG comprises a set of at most $n$ \ci, signifying that each node is conditionally independent of its non-descendants nodes given its parents. This succinct set of \ci\ holds significance, as it can be used for constructing the causal DAG, and all other \ci\ encoded in the DAG can be deduced from it.
Formally, given a causal DAG $\cg$, let $\langle X_1,{\dots},X_n \rangle$ denote a complete topological order over $\nodes(\cg)$.
Equation~\ref{eq:BN} implicitly encodes a set of $n$ \ci, called the RB for $\cg$, defined as follows:
\begin{equation}
	\label{eq:recursiveSet}
	\sigmarb(\cg)\eqdef\set{(X_i  \indep X_1\dots X_{i-1}{\setminus} \pi(X_i) \mid \pi(X_i)): i\in [n]}
\end{equation}
It has been shown~\cite{DBLP:journals/networks/GeigerVP90,DBLP:conf/uai/GeigerP88,DBLP:conf/uai/VermaP88} that both the semi-graphoid axioms (see Appendix~\ref {app:semigraphoid}) and $d$-separation are sound and complete for inferring \ci\ from the RB, which matches the \ci\ encoded by the causal DAG.

\begin{example}
Consider the causal DAG $\cg_1$ in Fig.~\ref{fig:compatible_graphs}(a). In the nodes' topological order, $A$ precedes $B$ and $C$, which in turn, precedes $D$. The last node is $E$. The RB of $\cg_1$ is given in Table~\ref{tab:recursive_basis}. Given the topological order over the nodes and the RB, $\cg_1$ can be fully constructed. Further, any CI statement encoded in $\cg_1$ can be implied from this RB by using the semi-graphoid axioms. \qed 
\end{example}

\subsubsection*{\bf ATE\& do-Calculus}
The $d$o-operator, a fundamental concept in causal inference, is used to denote interventions on variables in a causal model. 
It represents the intervention on a variable to observe the resulting change in an outcome variable while holding the external factors constant.
In computing the \emph{Average Treatment Effect} (ATE)~\cite{pearl_causality_2000}, a popular measure of causal estimate, the $do$-operator is applied to represent the treatment assignment for treatment and control groups. The ATE quantifies the average causal effect of a treatment $T$ on an outcome variable $O$ in a population:
\begin{equation}
    {\small ATE(T,O) = \mathbb{E}[O \mid do(T=1)] -  
    \mathbb{E}[O \mid do(T=0)]}
\label{eq:ate}
\end{equation}

To compute the causal effect of $T$ on $O$, it is crucial to identify and adjust for confounders. The backdoor criterion~\cite{pearl_causality_2000} provides a sufficient condition by identifying a set of variables $\mathbf{Z}$ that blocks all backdoor paths between $T$ and $O$, enabling confounder adjustment within the causal DAG framework. However, it is part of the $do$-calculus system, an axiomatic framework designed for reasoning about interventions and their effects within causal models. The $do$-calculus comprises three rules that facilitate the substitution of probability expressions containing the $do$-operator with standard conditional probabilities~\cite{pearl_causality_2000}. It provides a systematic method for deriving causal relationships from observational data. Due to its soundness and completeness, the framework offers a broad toolkit for causal inference. Since these concepts are not directly used in this paper, we omit a detailed review.

\section{Problem Formulation}
\label{sec:problem}

Our goal is to distill a causal DAG\footnote{For a discussion of other causal graph formats like mixed graphs, see Appendix \ref{subsec:mix_graphs}} into an interpretable summary by grouping nodes while preserving its utility for causal inference. Thus, the summary DAG should meet the following criteria:\\
\textbf{Size Constraint:} 
The summary DAG should be concise to reduce the cognitive load on analysts~\cite{bhowmick2022data}. We therefore impose a size constraint to enhance the summary DAG's intelligibility, ensuring the core complexity of the original DAG is maintained.

\noindent
\textbf{Preserving Causal Information:} The summary DAG must maintain the causal dependencies present in the original DAG: If variable $A$ has a directed causal path to $B$ in the original DAG, this relationship should be faithfully preserved in the summary DAG. The summary DAG should also preserve the CIs represented in the original DAG. If variables $A$ and $B$ are conditionally independent, this lack of dependence should be reflected in the summary DAG. Lastly, the summary DAG should not introduce any spurious conditional independencies that the original causal DAG does not imply.

\common{Our objective is to preserve the utility of the summary causal DAG for causal inference.}
As mentioned, in causal DAGs, the information encoded by missing edges implies the set of \ci\ the DAG represents. Therefore, removing edges can undermine the causal model as it implies \ci\ that do not necessarily hold in the original DAG. On the other hand, existing edges indicate \emph{potential} causal dependence. This implies that adding edges to a causal DAG, provided acyclicity is maintained, does not necessarily compromise validity. 
We, therefore, rigorously enforce conditions on the summary DAG to ensure that the directionality is faithfully preserved and assert that the summary DAG should preserve, to the greatest extent possible, a subset of the independence assumptions encoded in the original DAG. We show that, with these considerations, the summary causal DAG remains a viable tool for causal inference. 

We first formalize the concept of a summary causal DAG, then rigorously formalize the problem of causal DAG summarization.




\subsection{Summary Causal DAGs}
\label{subsec:qcg}
A \emph{summary graph} is obtained by applying \emph{node contraction} operations~\cite{pemmaraju2003computational}. The resulting graph retains the essential connectivity information of the original graph with a reduced number of nodes.

Given a graph $\cg$, the contraction of a pair of nodes $U, V {\in} \nodes(\cg)$ is the operation that produces a graph $\mathcal{H}$ in which the two nodes $U$ and $V$ are replaced with a single node $\mb C {=} \{U,V\} {\in} \nodes(\qg)$,
where $\mb C$ is now neighbors with nodes that $U$ and $V$ were originally adjacent to (edge directionality is preserved).
If $U$ and $V$ were connected by an edge, the edge is removed upon contraction.

\begin{definition}[Summary-DAG]
	\label{def:summaryDAG}
	A summary DAG of a DAG $\cg$ is a pair $(\qg,f)$, where $\qg$ is a DAG with nodes $\nodes(\qg)$, edges $\edges(\qg)$, and $f:\nodes(\cg) {\rightarrow} \nodes(\qg)$ is a function that partitions the nodes $\nodes(\cg)$ among the nodes $\nodes(\qg)$, such that: If $(U,V){\in} \edges(\cg)$, then $f(U){=}f(V)$ or $(f(U),f(V)){\in} \edges(\qg)$.
We define the inverse $f^{{-}1}{:}\nodes(\qg) {\rightarrow} 2^{\nodes(\cg)}$ as follows:
$f^{{-}1}(X) {\eqdef}\set{V{\in} \nodes(\cg): f(V){=}X}$
\end{definition}
To simplify the notations, we omit $f$ whenever possible.

\begin{figure}
\scriptsize
        \begin{subfigure}[b]{0.06\textwidth} \centering
        \small
        \begin{tikzpicture}[node distance=0.3cm,every node/.style={minimum size=0.5cm}]
            \tikzset{vertex/.style = {draw, circle}}

            \node[vertex] (A) {\scriptsize$A$};
            \node[vertex, below=0.3cm of A] (B) {\scriptsize$\mb{BC}$};
            \node[vertex, below=0.3cm of B] (D) {\scriptsize$D$};
            \node[vertex, below=0.3cm of D] (E) {\scriptsize$E$};

            \draw[->] (A) -- (B);
            \draw[->] (B) -- (D);
            \draw[->] (D) -- (E);
        \end{tikzpicture}
        \caption{$\mathcal{H}_{1}$}
     
    \end{subfigure}
        \begin{subfigure}[b]{0.1\textwidth} \centering
        \begin{tikzpicture}[node distance=0.5cm,every node/.style={minimum size=0.5cm}]
            \tikzset{vertex/.style = {draw, circle}}

            \node[vertex] (A) {\scriptsize$A$};
            \node[vertex, below right=0.3cm of A] (B) {\scriptsize$\mb{BD}$};
            \node[vertex, below left=0.3cm of A] (C) {\scriptsize$C$};
            \node[vertex, below=0.5cm of B] (E) {\scriptsize$E$};

            \draw[->] (A) -- (B);
            \draw[->] (A) -- (C);
            \draw[->] (C) -- (B);
            \draw[->] (B) -- (E);
        \end{tikzpicture}
        \caption{ $\mathcal{H}_{2}$}
     
    \end{subfigure}
     \begin{subfigure}[b]{0.06\textwidth} \centering
     \small
        \begin{tikzpicture}[node distance=0.5cm,every node/.style={minimum size=0.5cm}]
            \tikzset{vertex/.style = {draw, circle}}

            \node[vertex] (A) {\scriptsize$\mb{ABC}$};
            \node[vertex, below=0.5cm of A] (B) {\scriptsize$D$};
      
            \node[vertex, below=0.5cm of B] (E) {\scriptsize$E$};

            \draw[->] (A) -- (B);
            \draw[->] (B) -- (E);
        \end{tikzpicture}
        \caption{$\mathcal{H}_{3}$}
      
    \end{subfigure}
      \begin{subfigure}[b]{0.06\textwidth} \centering
        \begin{tikzpicture}[node distance=0.5cm,every node/.style={minimum size=0.5cm}]
            \tikzset{vertex/.style = {draw, circle}}

            \node[vertex] (A) {\scriptsize$A$};
            \node[vertex, below=0.5cm of A] (D) {\scriptsize$\mb{BC}$};
      
            \node[vertex, below=0.5cm of D] (E) {\scriptsize$\mb{DE}$};

            \draw[->] (A) -- (D);
            \draw[->] (D) -- (E);
        \end{tikzpicture}
        \caption{$\mathcal{H}_{4}$}
       
    \end{subfigure}
      \begin{subfigure}[b]{0.10\textwidth} \centering
        \begin{tikzpicture}[node distance=0.5cm,every node/.style={minimum size=0.5cm}]
            \tikzset{vertex/.style = {draw, circle}}

           \node[vertex] (A) {\scriptsize$\mathcal{H}_1$};
            \node[vertex, below=0.5cm of A] (D) {\scriptsize$\mathcal{H}_3$};
       \node[vertex, below right=0.5cm of A] (B) {\scriptsize$\mathcal{H}_4$};
            \node[vertex, below=0.5cm of D] (E) {\scriptsize$\mathcal{H}_2$};

            \draw[->] (A) -- (D);
            \draw[->] (A) -- (B);
            \draw[->] (D) -- (E);
        \end{tikzpicture}
        \caption{Order}
       
    \end{subfigure}
        \vspace{-4mm}
    \caption{
 Summary causal DAGs for $\cg_1$ and the partial order among them.
    }
    \label{fig:quotient_dags}
      \vspace{-3mm}
\end{figure}
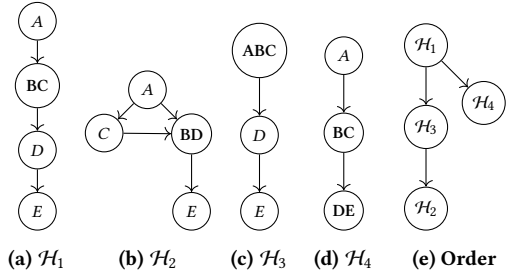




\begin{example}
Consider Fig. \ref{fig:compatible_graphs}(a) which depicts a DAG $\cg_1$. After contracting $B$ and $C$, the resulting summary DAG $\qg_1$ is displayed in Fig. \ref{fig:quotient_dags}(a). In $\qg_1$, the nodes $B$ and $C$ have been contracted into the node $\mb{BC}$. Namely, $f(B) = f(C) = \mb{BC}$, and $f^{-1}(\mb{BC}) = \set{B,C}$.\qed
\end{example}

A causal DAG $\cg$ is said to be \emph{compatible} with a summary DAG $\qg$, if, there exists a function $f$ that partitions the nodes $\nodes(\cg)$ among the nodes $\nodes(\qg)$, such that: If $(U,V){\in} \edges(\cg)$, then $f(U){=}f(V)$ or $(f(U),f(V)){\in} \edges(\qg)$. Namely, $\qg$ is a summary DAG of $\cg$. 

\begin{definition}[Compatibility]
\label{def:compatSD}
	Let $(\qg,f)$ be a summary DAG. A DAG $\cg$ is \e{compatible} with $\qg$ if $\qg$ is a summary DAG for $\cg$. We use $\comp$ to denote the set of all causal DAGs compatible with $\qg$. 
\end{definition}

We also use the term compatibility to describe the relationship between two causal DAGs sharing the same set of nodes, where the edges of one are fully contained in the set of edges of another graph.
Let $\cg$ be a causal DAG and let $\cg'$ be a causal DAG where $\nodes(\cg){=}\nodes(\cg')$. We say that $\cg'$ is a \emph{supergraph} of $\cg$ if $\edges(\cg){\subseteq} \edges(\cg')$. In this case, we also say that $\cg$ is \emph{compatible} with $\cg'$.

\begin{example}
Consider again Fig. \ref{fig:compatible_graphs}. Both $\cg_1$ and $\cg_2$ are compatible with the summary DAG $\qg_1$ shown in Fig. \ref{fig:quotient_dags}(a) (achieved by contracting $B$ and $C$).
However, $\cg_3$ is not compatible with $\qg_1$, since the edge between $D$ and $A$ is not preserved. \qed
\end{example}

We aim to find acyclic summary graphs. Thus, we prove a simple lemma characterizing node contractions that preserve acyclicity.

\begin{lemma}
	\label{lem:contractDAG}
Let $\cg$ be a DAG, and let $V,U {\in} \nodes(\cg)$. Let $\qg_{VU}$ denote the summary graph that results from $\cg$ by contracting $V$ and $U$. Then $\qg_{VU}$ contains a directed cycle if and only if $\cg$ contains a directed path $P$ from $V$ to $U$ (or $U$ to $V$), where $|P|{\geq} 2$.
\end{lemma}

A \emph{summary causal DAG} is a specific type of summary graph obtained through node contraction operations over a causal DAG $\cg$ and ensures acyclicity.

As mentioned, the RB of a causal DAG, as defined by Eq. (\ref{eq:recursiveSet}), comprises a set of at most $n$ \ci\ (where $n {=} |\nodes(\cg)|$), signifying that each node is
conditionally independent of its preceding nodes\footnote{According to a given full topological order of the nodes} given its parents.
This succinct set of \ci\ holds significance,  because it enables the derivation of all other \ci\ represented in the causal DAG.
The RB of a summary causal DAG is defined in a similar manner. The only difference is that in a summary causal DAG, a node may represent a subset of nodes of the original DAG. 

\begin{table}
    \centering
    \scriptsize
        \caption{The recursive bases of the summary DAGs in Figure~\ref{fig:quotient_dags}}
    \label{tab:recursive_basis}
    \vspace{-4mm}
    \begin{tabular}{cc}
         \toprule
        \textbf{Graph} & \textbf{Recursive Basis} \\
        \midrule
        $\cg_1$ & $(C \indep B | A), (D \indep A | BC), (E \indep ABC |D)$  \\
     
           $\mathcal{H}_{1}$ &  $(D \indep A | BC), (E \indep ABC |D)$ \\
     
     $\mathcal{H}_{2}$ &   $(E \indep AC |BD)$ \\
       
        $\mathcal{H}_{3}$ &   $(E \indep ABC |D)$ \\
        
              $\mathcal{H}_{4}$ &   $(DE \indep A |BC)$ \\
          \bottomrule
    \end{tabular}
\vspace{-5mm}
\end{table}

\begin{example}
\label{ex:RB_summary}
    \common{Figure \ref{fig:quotient_dags} displays four summary graphs for the causal DAG in Figure \ref{fig:compatible_graphs}(a). Table \ref{tab:recursive_basis} shows the RBs of these summary causal DAGs. In $\qg_4$, for instance, there are only three nodes and therefore the RB includes a single CI statement}. \qed
\end{example}



\subsection{The Causal DAG Summarization Problem} 
\label{sec:sum}
As mentioned, 
we aim to reduce an input causal DAG by contracting its nodes, retaining maximal causal information.
We covered the two criteria of our problem before proceeding with formalizing it.

\vspace{1mm}
\noindent
{\bf Size Constraint}
A size constraint is a key motivating constraint for graph summarization work and may be imposed on the number of nodes, storage space, minimum description length, etc.~\cite{liu2018graph}. 
We focus on a node-based size constraint as limited-size graphs are generally more accessible for inspection~\cite{bhowmick2022data, huang2009measuring}. \revb{Additionally, setting and adjusting a limit on the number of nodes is shown to be relatively straightforward for analysts~\cite{tian2008efficient,DBLP:journals/pvldb/KhanBB17}. We also observe that other summarization problems share similar hyperparameters, such as many clustering algorithms or k-nearest neighbors~\cite{kramer2013k,likas2003global}.} 



\vspace{1mm}
\noindent
{\bf Causal Information Preservation}
As mentioned, if two variables have a directed path between them in
the original DAG, then this relationship should be faithfully
preserved in the summary DAG. Indeed, this follows from the definition of a summary DAG (Def.~\ref{def:summaryDAG}).

Given two summary DAGs derived from the same causal DAG $\cg$, both adhering to the size constraint,
we prefer the one that preserves, to a larger degree, the set of \ci\ represented in $\cg$. 
To this end, we devise a measure to compare summary DAGs based on their RBs.  
When comparing two summary DAGs $\qg_1$ and $\qg_2$, we assert that $\qg_1$ is \emph{superior} to $\qg_2$ if the RB of $\qg_2$ is implied by the RB of $\qg_1$. Namely, all the \ci\ encoded by $\qg_2$ can also be deduced from $\qg_1$. We are searching for a maximal summary causal DAG, namely, that its RB is not implied by any other valid summary DAG. 

\revb{As mentioned, a summary DAG should not introduce any spurious CIs that the original causal DAG does not imply. However, it may overlook some CIs that are present in the original DAG. We refer to this property as an I-Map. Formally,} 
Let $\Omega{\eqdef} \set{X_1,{\dots},X_n}$ be a set of jointly distributed random variables with distribution $\pr$ (i.e., nodes of the original DAG). Formally,
\begin{definition}[I-Map]
\label{def:imap}
    A DAG $\cg$ is an \e{I-Map} for $\pr$ if for every disjoint sets $\mb{X},\mb{Y}$, and $\mb{Z}$ it holds that $(\mb{X}\indep_{d} \mb{Y} \mid \mb{Z})_{\cg}$ only if $(\mb{X}\indep_{\pr} \mb{Y} \mid \mb{Z})$. 
\end{definition}
Let $\cg_1$ and $\cg_2$ be two DAGs that are I-Maps for $\pr$.
We say that $\cg_2$ is superior to $\cg_1$, in notation $\cg_2 {\succ} \cg_1$, if for every $\sigma{\in} \sigmarb(\cg_1)$, it holds that $\sigmarb(\cg_2) {\implies} \sigma$.
Note that the relation $\succ$ does not necessarily form a complete order. We say that $\cg$ is \e{maximal for $\pr$} if $\cg$ is an I-Map for $\pr$, and there does not exist any $\cg'{\in} \mathcal{G}(\pr)$ such that $\cg'{\succ} \cg$. 
\revb{Our goal is to find a summary DAG that is an I-Map for $\pr$ and maximal for $\pr$, given a constraint on the number of nodes}. 

\begin{example}
Consider the causal DAG $\cg_1$ in Fig.~\ref{fig:compatible_graphs}(a). Fig.~\ref{fig:quotient_dags}(a) presents a $4$-size summary DAG $\qg_1$ for $\cg_1$. 
The RBs of both DAGs are shown in Table \ref{tab:recursive_basis}. 
Clearly, $\sigmarb(\qg_1){\subset} \sigmarb(\cg_1)$, and hence $\qg_1$ is an I-Map for $\pr$. 
Fig.~\ref{fig:quotient_dags}(b) presents $\qg_2$, another $4$-size summary DAG for $\cg_1$, where $\sigmarb(\qg_2){=}\set{(E \indep AC|BD)}$. From the semi-graphoid axioms, it holds that $(E \indep ABC|D) {\implies} (E \indep AC|BD)$. Thus, $\qg_1{\succ} \qg_2$. Hence, $\qg_1$ is a superior summary DAG. 
Similarly, Figures~\ref{fig:quotient_dags}(c) and ~\ref{fig:quotient_dags}(d) illustrate $\qg_3$ and $\qg_4$, $3$-size summary DAGs for $\cg_1$. Their RBs are given in Table~\ref{tab:recursive_basis}. The partial order among all summary DAGs is presented in Fig.~\ref{fig:quotient_dags}(e). Despite $\qg_3$ having only three nodes, it surpasses $\qg_2$. 
However, $\qg_3$ and $\qg_4$ are incomparable, i.e., neither $\sigmarb(\qg_3) {\implies} \sigmarb(\qg_4)$ nor $\sigmarb(\qg_4) {\implies} \sigmarb(\qg_3)$.\qed
\end{example}



\begin{problem}[Causal DAG Summarization]
\label{causal-graph-summarization-problem}
Given a causal DAG $\cg$ defined over a joint distribution $\pr$, and a bound $k$, find a summary causal DAG $\qg$ s.t. (i) the number of nodes in $\qg$ is $\leq k$; (ii) $\cg$ is compatible with $\qg$, \revb{is an I-Map for $\pr$} and is maximal for $\pr$. 
\end{problem}

\begin{example}
\label{ex:problem}
Consider again the causal DAG in Fig. \ref{fig:example_causal_dag}. We set $k {=} 5$. 
Fig. \ref{fig:summary_graphs_cagres} depicts an optimal summary causal DAG. 
Namely, the RB of any other summary causal DAG with 5 or fewer nodes is not superior to RB of this 5-node summary causal DAG.\qed
\end{example}


We show that the \probName\ problem is $NP$-hard via a reduction from the $k$-Max-Cut problem~\cite{hartmanis1982computers}.

\begin{theorem}
\label{theorem:problem-is-np-hard}
    Causal DAG summarization is an NP-hard problem.
\end{theorem}

\revb{As the proof relies on the relationship between node contractions and the addition of edges, established in the next section, we will explain the intuition behind this theorem in Section \ref{subsec:canonical_DAG}.} 


\section{Node-Contraction as Edge Addition}
\label{sec:csep}
Next, we establish the connection between node contractions and the addition of edges to the input causal DAG. This connection will be used to read off, from a given summary causal DAG, all the \ci\ it encodes. It also serves as a pivotal factor in guiding our algorithm for selecting promising node pairs to merge. Additionally, in Section \ref{sec:identifiability}, we will leverage this connection to demonstrate how causal inference can be directly conducted over summary DAGs.

\revb{We note that the \groundedDAG\ is not an objective of our problem; rather, it serves as a tool to formally define causal inference over summary DAGs and to guide our algorithm in identifying node contractions that minimize information loss.}

\subsection{The Canonical Causal DAG}
\label{subsec:canonical_DAG}
Given a summary causal DAG $\qg$, we define its corresponding \groundedDAG, denoted as $\groundedQG$. In this causal DAG, cluster nodes are decomposed into distinct nodes connected by edges. 
We show that the RB of the \groundedDAG\ is \emph{equivalent} to that of $\qg$.
We first define the notion of equivalence for sets of \ci.

\begin{definition}[CI Sets Equivalence]
   Let $\mb S$ and $\mb T$ denote two sets of \ci\ over the variable-set $\set{X_1,{\dots},X_n}$.
We say that $\mb{S}{\implies} \mb{T}$ if $\mb{S} {\implies} \sigma$ for every CI $\sigma {\in} \mb{T}$. We say that $\mb{S}$ and $\mb{T}$ are \emph{equivalent}, in notation $\mb{S}{\equiv} \mb{T}$, if $\mb{S}{\implies} \mb{T}$ and $\mb{T} {\implies} \mb{S}$. 
\end{definition}

Next, we formally define the notion of the \emph{\groundedDAG} for a given summary DAG.  

\begin{definition}[Canonical Causal DAG]
	\label{def:groundedDAG}
	Let $(\qg,f)$ be a summary DAG for a causal DAG $\cg$. Let $\langle X_1,\dots,X_n \rangle$ denote a complete topological order over $\nodes(\cg)$. We define the \groundedDAG\ associated with $(\qg,f)$, denoted $\groundedQG$ as follows: $\nodes(\groundedQG)=\nodes(\cg)$, and
	\begin{align*}
(X_i,X_j)\in \edges(\groundedQG) &\text{ if and only if }&  (X_i,X_j)\in \edges(\cg) \\
		&\text{ or }& (f(X_i),f(X_j))\in \edges(\qg) \\
		&\text{ or }& f(X_i)=f(X_j) \text{ and }i<j
	\end{align*}
\end{definition}

We observe that, by definition, $\groundedQG$ is compatible with the summary DAG $(\qg,f)$.

\begin{example}
\label{ex:canonical_DAG}
\common{
    Consider Figures \ref{fig:canonical_DAG}(a) and \ref{fig:canonical_DAG}(b) that depict an input causal DAG, and a 3-node summary. Fig. \ref{fig:canonical_DAG}(c) depicts the corresponding \groundedDAG. In the topological order $A$ precedes $B$ which in turn precedes $C$. All nodes within the node $\mb{ABC}$ are connected by edges in the canonical causal DAG, 
    according to the topological order. Since $\mb{ABC}$ is the parent of $D$ in the summary DAG, in the \groundedDAG\ all $A, B$ and $C$ are parents of $D$. Note that Fig. \ref{fig:canonical_DAG}(c) contains two more edges than Fig. \ref{fig:canonical_DAG}(a), which represents conditional independence relationships which are not captured in the 3-node summary graph Figure \ref{fig:canonical_DAG}(b). }\qed
\end{example}

\begin{figure}
 \centering
 \scriptsize
  \begin{subfigure}[b]{0.12\textwidth} \centering
        \begin{tikzpicture}[node distance=0.5cm,every node/.style={minimum size=0.5cm}]
            \tikzset{vertex/.style = {draw, circle}}

            \node[vertex] (A) {$A$};
            \node[vertex, below right=0.3cm of A] (B) {$B$};
            \node[vertex, below left=0.3cm of A] (C) {$C$};
            \node[vertex, below=0.5cm of A] (D) {$D$};
            \node[vertex, below=0.3cm of D] (E) {$E$};

            \draw[->] (A) -- (B);
            \draw[->] (A) -- (C);
            \draw[->] (C) -- (D);
            \draw[->] (B) -- (D);
            \draw[->] (D) -- (E);
        \end{tikzpicture}
         \vspace{-2mm}
        \caption{\common{Causal DAG }}
        \label{fig:g1}
    \end{subfigure}
    \begin{subfigure}[b]{0.13\textwidth} \centering
        \begin{tikzpicture}[node distance=0.5cm,every node/.style={minimum size=0.5cm}]
            \tikzset{vertex/.style = {draw, circle}}

            \node[vertex] (A) {\scriptsize$\mb{ABC}$};
            \node[vertex, below=0.5cm of A] (B) {\scriptsize$D$};
      
            \node[vertex, below=0.5cm of B] (E) {\scriptsize$E$};

            \draw[->] (A) -- (B);
            \draw[->] (B) -- (E);
        \end{tikzpicture}
        \caption{\common{Summary DAG}}
    \end{subfigure}
     \begin{subfigure}[b]{0.18\textwidth} \centering
        \begin{tikzpicture}[node distance=0.5cm,every node/.style={minimum size=0.5cm}]
            \tikzset{vertex/.style = {draw, circle}}

            \node[vertex] (A) {$A$};
            \node[vertex, below right=0.3cm of A] (B) {$B$};
       \node[vertex, below left=0.3cm of A] (C) {$C$};
            \node[vertex, below=0.5cm of A] (D) {$D$};
              \node[vertex, below=0.3cm of D] (E) {$E$};

            \draw[->] (A) -- (B);
            \draw[->] (A) -- (C);
            \draw[->] (B) -- (C);
                 \draw[->] (A) -- (D);
            \draw[->] (B) -- (D);
            \draw[->] (C) -- (D);
             \draw[->] (D) -- (E);
        \end{tikzpicture}
        \caption{\common{Canonical causal DAG}}
      
    \end{subfigure}
        \vspace{-4mm}
    \caption{
\common{A causal DAG, its summary DAG, and the corresponding \groundedDAG}}.
    \label{fig:canonical_DAG}
    \vspace{-3mm}
\end{figure}
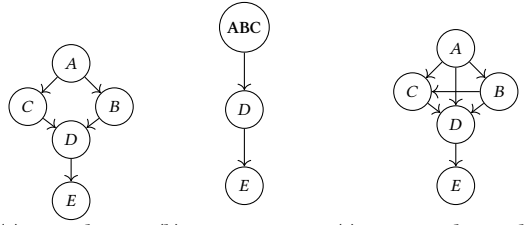

We show that the RB of the \groundedDAG\ $\groundedQG$ is equivalent to that of the summary DAG $\qg$ obtained by node contractions to a causal DAG $\cg$. In other words, node contractions can be conceptualized as the addition of edges to the input causal DAG.


\begin{theorem}
	\label{thm:contractionEdgeAddition}
Let $\qg$ be a summary causal DAG, and $\groundedQG$ is its corresponding \groundedDAG. We have: $\Sigma_{\text{RB}}(\qg)\equiv \Sigma_{\text{RB}}(\groundedQG)$.
\end{theorem}

Continuing with Example \ref{ex:canonical_DAG}, the RB of $\mathcal{G}_{\mathcal{H}_3}$ is $(E \indep ABC|D)$, which is identical to that of $\qg_3$ (see Table \ref{tab:recursive_basis}).

\vspace{1mm}
\noindent
\revb{{\bf Proof Intuition for Theorem~\ref{theorem:problem-is-np-hard}}.
In our proof we rely on the connection between a summary DAG and it \groundedDAG. Specifically, 
Theorem~\ref{theorem:problem-is-np-hard} establishes that finding a summary DAG $(\qg,f)$ whose canonical causal DAG $\cg_{\qg}$ results in the smallest number of added edges $|\edges(\cg_{\qg})|-|\edges(G)|$ is NP-Hard. Specifically, our proof shows that finding a summary DAG $(\qg,f)$ where $|\nodes(\qg)|=k$ and $|\edges(\cg_{\qg})|-|\edges(G)|\leq \tau$ for some threshold $\tau >0$ is an NP-complete problem. In fact, we prove the stronger claim that finding a summary DAG $(\qg,f)$ where $|\nodes(\qg)|=k$ and 
\begin{equation}
\label{eq:specialedges}
    \left| \set{(X_i,X_j)\in \edges(\cg_{\qg}){\setminus}\edges(G) : f(X_i)=f(X_j} \right|\leq \tau
\end{equation}
is NP-hard. If $|\edges(\cg_{\qg})|-|\edges(G)|\leq \tau$, then~\eqref{eq:specialedges} must hold as well. We establish that finding a summary DAG where~\eqref{eq:specialedges} holds is NP-Hard, and hence finding a summary DAG where $|\edges(\cg_{\qg})|-|\edges(G)|\leq \tau$ is NP-Hard as well. The full proof is provided in the Appendix.
}


\subsection{$s$-Separation}
\label{sec:ssep}
We introduce the notion of \emph{$s$-separation}, an extension of $d$-separation, tailored to identify \ci\ encoded by a summary DAG. Intuitively, a summary DAG represents a collection of causal DAGs that are compatible with it, meaning that it could have been obtained from any of those DAGs (similar to \emph{possible worlds}~\cite{dalvi2007efficient}). Each of these DAGs encodes a different set of \ci. The set of \ci\ encoded by a summary DAG is the intersection of \ci\ that holds in all compatible DAGs. In this way, we can ensure we restrict ourselves only to \ci\ that are certainly present in a particular context and can be reliably used for inference. 
\revb{$s$-separation extends $d$-separation, which allows the identification of valid CIs within a \emph{summary} DAG. We also introduce a sound and complete $s$-separation algorithm that leverages the standard $d$-separation algorithm.}


The validity of a CI statement, as derived from summary DAG $\qg$, is given by the following definition:

\begin{definition}[Validity of a CI in a summary DAG]
A CI statement is deemed \emph{valid} in a summary causal DAG $\qg$ if and only if it is implied by all causal DAGs within $\comp$.
\end{definition}

\noindent
$s$-separation captures all certain \ci\ that hold across all DAGs in $\comp$. 
We propose the following criterion for $s$-separation to encapsulate this notion of validity. 
\begin{definition}[$s$-separation]
Given a summary DAG $(\qg,f)$ and disjoint subsets $\mathbf{X}, \mathbf{Y}, \mathbf{Z} {\subseteq} \nodes(\qg)$, we say that $\mathbf{X}$ and $\mathbf{Y}$ are $s$-separated in $\qg$ by $\mathbf{Z}$, denoted by $(\mathbf{X} \indep_S \mathbf{Y} \mid \mathbf{Z})_{\qg}$, iff $f^{-1}(\mathbf{X})$ and $f^{-1}(\mathbf{Y})$ are $d$-separated by $f^{-1}(\mathbf{Z})$ in every causal DAG within $\comp$.
\end{definition}

We say that $\mathbf{X}$ and $\mathbf{Y}$ are $s$-connected in $(\qg,f)$ by $\mathbf{Z}$, if there exists a causal DAG $\cg \in \comp$, such that $f^{-1}(\mathbf{X})$ and $f^{-1}(\mathbf{Y})$ are $d$-connected in $\cg$ by $f^{-1}(\mathbf{Z})$.

\subsubsection{\bf $s$-separation Algorithm}
Given a summary causal DAG $\qg$, we aim to derive the set of \ci\ it encodes.
A naive approach would be to employ $d$-separation algorithms~\cite{pearl_causality_2000}. However, $\qg$ can potentially encompass more \ci\ than those discerned through $d$-separation alone, as demonstrated in the following example. 

\begin{example}
\label{ex:without-csep}
Referring back to Fig. \ref{fig:compatible_graphs}, $(B \indep_{d} E \mid D)$, $(C \indep_{d} E $ | $B,D)$, and $(B, C \indep_{d} E \mid D)$ all hold in $\cg_1$ and $\cg_2$.
Likewise, $(\mb{BC} \indep_{d} E \mid D)$ holds in $\qg_1$ (Fig. \ref{fig:quotient_dags}(a)).
However, since $\qg_1$ does not contain $B$ or $C$ as separate nodes, we cannot establish $(B \indep_{d} E \mid D)$ or $(C \indep_{d} E \mid B,D)$ from $\qg_1$ using $d$-separation. \qed
\end{example}
To address this, a simple solution is to find the set of \ci\ shared across all DAGs compatible with $\qg$. 
However, this approach is costly. 
We, therefore, present a simple algorithm for $s$-separation that leverages the connection between a summary DAG and its canonical causal DAG.
This algorithm operates as follows: Given a summary DAG $\qg$, establish a topological order for its nodes.\footnote{The order of nodes within a cluster is considered arbitrary, or it may be determined based on the topological order of the input causal DAG if such information is preserved.} Using this order, construct the \groundedDAG\ $\groundedQG$. Next, apply $d$-separation over $\groundedQG$ and return the resulting CI set.  
We demonstrate that this algorithm is sound and complete.

\eat{
\begin{theorem}[Soundness and Completeness of $d$-separation in supergraphs]
	\label{thm:dSep}
	In \sout{an ADMG} a causal DAG $G'$, let $\mb{X},\mb{Y},\mb{Z}\subseteq \nodes(G')$ be disjoint sets of variables. If $\mb{X}$ and $\mb{Y}$ are $d$-separated by $\mb{Z}$ in $G'$, then in any \sout{ADMG} causal DAG $G$ compatible with $G'$, $\mb{X}$ and $\mb{Y}$ are $d$-separated by $\mb{Z}$ in $G$. That is:
	\begin{align}
		(\mb{X}\indep_d\mb{Y}|\mb{Z})_{G'} \implies (\mb{X}\indep_d\mb{Y}|\mb{Z})_{G}
	\end{align}
	If $\mb{X}$ and $\mb{Y}$ are $d$-connected by $\mb{Z}$ in $G'$, then there exists a compatible \sout{ADMG} causal-DAG $G$, compatible with $G'$, where $X$ and $Y$ are $d$-connected by $Z$ in $G$.
\end{theorem}
}
\begin{theorem}[Soundness and Completeness of $s$-separation]
\label{cor:dsep}
		In a summary DAG $(\qg,f)$, let $\mb{X},\mb{Y},\mb{Z}{\subseteq} \nodes(\qg)$ be disjoint sets of nodes. If $\mb{X}$ and $\mb{Y}$ are $d$-separated by $\mb{Z}$ in $\qg$, then in any causal DAG $\cg {\in} \comp$, $f^{{-}1}(\mb{X})$ and $f^{{-}1}(\mb{Y})$ are $d$-separated by $f^{{-}1}(\mb{Z})$. That is:
  \small{
	\begin{align*} 
		(\mb{X}\indep_d \mb{Y}|\mb{Z})_{\qg} &{\implies} (f^{{-}1}(\mb{X})\indep_d f^{{-}1}(\mb{Y})|f^{{-}1}(\mb{Z}))_{\cg}
  {\implies} (\mb{X}\indep_s \mb{Y}|\mb{Z})_{\qg}
	\end{align*}}
	If $\mb{X}$ and $\mb{Y}$ are $d$-connected by $\mb{Z}$ in $\qg$, then there exists a DAG $\cg {\in} \comp$, s.t $f^{{-}1}(\mb{X})$ and $f^{{-}1}(\mb{Y})$ are $d$-connected by $f^{{-}1}(\mb{Z})$ in $\cg$.
\end{theorem}
\section{The \algoName\ Algorithm}
\label{sec:algorithm}

\reva{As demonstrated in Theorem \ref{theorem:problem-is-np-hard}, the \probName\ problem is NP-hard and therefore it is not trivial to devise an efficient algorithm with theoretical guarantees. We, therefore, introduce a heuristic algorithm, named \algoName\, for the \probName\ problem. Although lacking theoretical guarantees, \algoName\ effectively meets the size constraint and produces summary causal DAGs that can be directly used for sound causal inference.}
A brute force approach explores all summary DAGs with up to $k$ nodes.
It finds the optimal summary DAG, but runs in exponential time 
due to the exponential number of potential graphs. \algoName\ addresses this by estimating the merging effect on the \groundedDAG\ rather than iterating over all possible summary DAGs.

\vspace{1mm}
\noindent
{\bf Overview} \revb{The \algoName\ algorithm 
follows a previous line of work ~\cite{tian2008efficient,geisberger2012exact}, where a bottom-up greedy approach is used to identify promising node pairs for contraction. Its main contribution lies in \emph{how} it estimates merge costs, to preserve the causal interpretation of the graph}: It counts the number of edges to be added in the \groundedDAG\ for each node pair (a proxy for the RB's effect, as discussed in Section \ref{sec:csep}). In each iteration, the algorithm contracts the node pair resulting in the minimal number of additional edges. We also introduce optimizations for runtime efficiency, such as semantic constraint, fast low-cost merges, and caching mechanisms.

\vspace{1mm}
The \algoName\ algorithm is given in Algorithm \ref{algo:greedy_search}.
Given a bound $k$ and an input causal DAG, this algorithm iteratively seeks the next-best pair of nodes to be merged, until the size constraint is met (lines 4-15). 
The next-best pair of nodes to merge is the node pair whose contraction has the lowest cost (lines 10-12). The algorithm randomly breaks ties (lines 13-14). 
The GetCost procedure is shown in Algorithm \ref{algo:cost}. The cost of merging two (clusters of) nodes $\mb{U}$ and $\mb{V}$ is equal to the number of edges to be added in the corresponding \groundedDAG: (1) edges to be added between the nodes within the combined cluster $\mb{U} \bigcup \mb{V}$ (lines 3-4), (2) new parents for the nodes in $\mb{U}$ or $\mb{V}$ post-merge (lines 6-11), and (3) new children for the nodes in $\mb{U}$ or $\mb{V}$ after the merge (lines 13-18). 

\setlength{\textfloatsep}{2px}
\SetInd{1.0ex}{1.0ex}
\begin{algorithm}[t]
  \small
  \DontPrintSemicolon
  \SetKwInOut{Input}{input}\SetKwInOut{Output}{output}
  \LinesNumbered
  \Input{A causal DAG $\cg$ and a number $k$. }
  \Output{A summary causal DAG $\mathcal{H}$ with $k$ nodes.} \BlankLine
  \SetKwFunction{SimpleMerges}{\textsc{LowCostMerges}}
  \SetKwFunction{Merge}{\textsc{Merge}}
  \SetKwFunction{IsValidPair}{\textsc{IsValidPair}}
  \SetKwFunction{GetCost}{\textsc{GetCost}}

  $\mathcal{H}  \gets \cg$\\
  $\mathcal{H}  \gets \SimpleMerges(\mathcal{H} )$\\

  \While{ $\text{size}(\mathcal{H}.nodes )>k$ }
  {
    $min\_cost  \gets  \infty$ \\
    $(\boldsymbol{X},\boldsymbol{Y}) \gets Null$\\

    \For{$(\boldsymbol{U}, \boldsymbol{V}) \in \mathcal{H}.\text{nodes}$}
    {
      \If{\IsValidPair($\boldsymbol{U}, \boldsymbol{V}, \mathcal{H}$)}{

        $\text{cost}_{UV} \gets \GetCost(\boldsymbol{U}, \boldsymbol{V}, \mathcal{H})$\\
        \If{$\text{cost}_{UV} < \text{min\_cost}$}{
          min\_cost $ \gets \text{cost}_{UV}$\\
          $(\boldsymbol{X},\boldsymbol{Y}) \gets (\boldsymbol{U},\boldsymbol{V})$
        }

        \If{$\text{cost}_{UV} ==$ min\_cost}{
          Randomly decide if to replace $\boldsymbol{X}$ and $\boldsymbol{Y}$ with $\boldsymbol{U}$ and $\boldsymbol{V}$\\
        }
      }
    }
    \reva{ $\mathcal{H}.\text{Merge}(\boldsymbol{X}, \boldsymbol{Y})$}\\
  }
  \Return $\mathcal{H}$

  \caption{The \algoName\ Algorithm}\label{algo:greedy_search}
\end{algorithm}

\setlength{\textfloatsep}{2px}
\SetInd{1.0ex}{1.0ex}
\begin{algorithm}[t]
  \small
  \DontPrintSemicolon
  \SetKwInOut{Input}{input}\SetKwInOut{Output}{output}
  \LinesNumbered
  \Input{A summary causal DAG $\qg$ and a pair of nodes $\boldsymbol{U}$ and $\boldsymbol{V}$. }
  \Output{The cost of contracting  $\boldsymbol{U}$ and $\boldsymbol{V}$.} \BlankLine
  \SetKwFunction{RemoveSharedParents}{\textsc{RemoveSharedParents}}
  \SetKwFunction{RemoveSharedChildren}{\textsc{RemoveSharedChildren}}
  \SetKwFunction{Merge}{\textsc{Merge}}
  \SetKwFunction{IsValidPair}{\textsc{IsValidPair}}
  \SetKwFunction{HasEdge}{\textsc{HasEdge}}
  \SetKwFunction{GetPredecessors}{\textsc{GetPredecessors}}
  \SetKwFunction{GetSuccessors}{\textsc{GetSuccessors}}

  $\text{cost} \gets 0$\\
  \tcc{\textcolor{blue}{New edges among the nodes in the cluster}}
  \If{$\qg.\HasEdge(\boldsymbol{U},\boldsymbol{V})$ == False}{cost $\gets$ cost $+ size(\boldsymbol{U})\cdot size(\boldsymbol{V})$\\}

  \tcc{\textcolor{blue}{New parents}}
  \revc{
    $\text{parents}_U \gets \qg.\text{\GetPredecessors}(\boldsymbol{U})$\\
    $\text{parents}_V \gets \qg.\text{\GetPredecessors}(\boldsymbol{V})$\\
    $\text{parentsOnlyU} \gets \text{parents}_U \setminus \text{parents}_V$\\
    cost $\gets$ cost $+$ $\text{size}(\text{parentsOnlyU}) {\cdot} \text{size}(\boldsymbol{V})$\\

    $\text{parentsOnlyV} \gets \text{parents}_V \setminus \text{parents}_U$\\
    cost $\gets$ cost $+$ $\text{size}(\text{parentsOnlyV}) {\cdot} \text{size}(\boldsymbol{U})$\\
  }


  \tcc{\textcolor{blue}{New children}}

  \revc{
    $\text{children}_U \gets \qg.\text{\GetSuccessors}(\boldsymbol{U})$\\
    $\text{children}_V \gets \qg.\text{\GetSuccessors}(\boldsymbol{V})$\\
    $\text{childrenOnlyU} \gets \text{children}_U \setminus \text{children}_V$  \\
    cost $\gets$ cost $+$ $\text{size}(\text{childrenOnlyU}) \cdot \text{size}(\boldsymbol{V})$\\
    $\text{childrenOnlyV} \gets \text{children}_V \setminus \text{children}_U$  \\
    cost $\gets$ cost $+$ $\text{size}(\text{childrenOnlyV}) \cdot \text{size}(\boldsymbol{U})$
  }\\


  \Return cost

  \caption{The GetCost Procedure}\label{algo:cost}
\end{algorithm}

\vspace{1mm}
We next propose three optimizations to improve runtime.

\vspace{1mm}
\noindent
{\bf Semantic Constraint}:
We can reduce the search space and ensure that only semantically related variables are merged, thereby supporting semantic coherence in the summary DAG. To achieve this, the user may specify which node pairs are allowed to be merged by providing a semantic similarity matrix and a threshold that indicates the maximum distance between two nodes within a cluster. 
The user can assess the semantic similarity using previous work on semantic similarity~\cite{Harispe2015SemanticSF,mikolov2013efficient} or large language models~\cite{chatgpt}.

Assume a semantic similarity measure 
$sim(\cdot,\cdot)$ that assigns a value between 0 and 1 to a pair of variables.
For a summary DAG 
$\qg$ and a threshold $\tau$, we say that 
$\qg$ satisfies the semantic constraint if, for every cluster $\mb{C} \in \nodes(\qg)$, $sim(V_i, V_j) \geq \tau$ for every $V_i, V_j \in \mb{C}$. This condition is checked in line 8 of the \algoName\ algorithm when validating whether a pair of nodes is suitable for contraction.

\vspace{1mm}
\noindent
{\bf Caching}:
We use two caching mechanisms: one for storing invalid node pairs and another for cost scores. \revc{We demonstrated in Section \ref{subsec:ablation} that these caching mechanisms are effective in reducing runtime.} 

We initialize the invalid pairs cache during a preprocessing phase. An invalid pair is a node pair with semantic similarity \reva{below} the threshold or connected by a directed path of length above $2$ (Lemma \ref{lem:contractDAG}). During \algoName's run, invalid pairs are cached, and each iteration checks the validity of node pairs before computing costs.

The cost of a node pair $\mb{U}, \mb{V}$ remains unchanged after merging another node pair $\mb{X}, \mb{Y}$ if neither $\mb{U}$ nor $\mb{V}$ are neighbors of $\mb{X}$ or $\mb{Y}$. Following the merge of $\mb{X}$ and $\mb{Y}$, we update the cost cache by removing the cost scores of all node pairs involving one of their neighbors.
When calculating the cost for a node pair, we check if the score is in the cache. If not, we compute and add it, ensuring the cache reflects node pair mergers' impact on neighboring pairs.

\vspace{1mm}
\noindent
{\bf Low Cost Merges}:
As pre-processing, we contract node pairs with low costs (line~4). This involves merging nodes that share identical children and parents. Additionally, we merge nodes linked along non-branching paths, each having at most one parent and one child. \revc{In Section \ref{subsec:ablation}, we experimentally show that this optimization benefits small or low-density causal DAGs.}

\vspace{1mm}
\noindent
{\bf Time Complexity}
A single cost computation with $n {=} |\nodes(\cg)|$ takes $O(n)$ due to the maximum number of neighbors a node can have. 
The algorithm undergoes $n{-}k$ iterations, evaluating all node pairs (O($n^2$) such pairs) in the current summary DAG with no more than $n$ neighbors. Thus, the overall time complexity is $O((n{-}k) {\cdot} n^3)$.

\section{Do-Calculus in Summary Causal DAGs}
\label{sec:identifiability}


Next, we show that the rules of $do$-calculus are sound and complete in summary causal DAGs. This is vital to ensure that the summary causal DAGs are effective formats that support causal inference by enabling direct causal inference on the summary DAGs. 
Our proof relies on the equivalence between the RB of a summary DAG and its \groundedDAG\ (Theorem~\ref{thm:contractionEdgeAddition}). This result is not surprising because the \groundedDAG\ is a supergraph of the input causal DAG. Pearl already observed in~\cite{pearl_causality_2000} that: \emph{``The addition of arcs to a causal diagram can impede, but never assist, the identification of causal effects in nonparametric models. This is because such addition reduces the set of $d$-separation conditions carried by the diagram; hence, if causal effect derivation fails in the original diagram, it is bound to fail in the augmented diagram''}.


\vspace{1mm}
Given a causal DAG $\cg$, for a set of nodes $\mb{X} \subseteq \nodes(\cg)$, let $\cg_{\overline{\mb X}}$ denote the graph that results from $\cg$ by removing all incoming edges to nodes in $\mb X$, by $\cg_{\underline{\mb X}}$ the graph that results from $\cg$ by removing all outgoing edges from the nodes in $\mb X$. For a set of nodes $\mb{X}\subseteq \nodes(G){\setminus}\mb{Z}$, we denote by $\cg_{\overline{\mb X}\underline{\mb Z}}$ the graph that results from $\cg$ by removing all incoming edges into $\mb X$ and all outgoing edges from $\mb Z$.

\begin{theorem}[Soundness of Do-Calculus in summary causal DAGs]\label{thr:soundeness-do-calculus-summaryDAGs}
	Let $\cg$ be a causal DAG encoding an interventional distribution $P(\cdot \mid do(\cdot))$, compatible with the summary causal DAG $(\qg,f)$. For any disjoint subsets $\mb{X,Y,Z,W} {\subseteq} \nodes(\qg)$, the following rules hold:%
	{	\small
	\begin{align*}
		{\bf R}_1: & (\mb{Y}\indep \mb{Z}|\mb{X,W})_{\qg_{\overline{\mb{X}}}} \implies  P(\bY \mid do(\bX),\bZ,\bW)=P(\bY\mid do(\bX),\bW)   \\
		{\bf R}_2: & (\bY\indep \bZ|\bX,\bW)_{\qg_{\overline{\bX}\underline{\bZ}}} {\implies} P(\bY | do(\bX),do(\bZ),\bW){=}P(\bY | do(\bX),\bZ, \bW) \\
		{\bf R}_3: & (\bY\indep\bZ|\bX,\bW)_{\qg_{\overline{\bX}\overline{\bZ(\bW)}}} {\implies} P(\bY | do(\bX),do(\bZ),\bW){=}P(\bY| do(\bX), \bW)
	\end{align*}
}%
where, $\bU\eqdef(U)$ for every $U\in \nodes(\qg)$, and $Z(W)$ is the set of nodes in $Z$ that are not ancestors of any node in $W$. 

\end{theorem}
	


\begin{theorem}[Completeness of Do-Calculus in summary causal DAGs]
\label{thr:completness-do-calculus}
	Let $(\qg,f)$ be a summary causal DAG for $\cg$, and let $\bX,\bY,\bW,\bZ {\subseteq} \nodes(\qg)$ be disjoint sets of variables. If $\bY$ is $d$-connected to $\bZ$ in $\qg_{\overline{\bX}}$ w.r.t. $\bX\cup \bW$, then there exists a causal DAG $\cg'$ compatible with $\qg$, such that $f(\bY)$ is $d$-connected to $f(\bZ)$ in $\cg'_{\overline{f(\bX)}}$ w.r.t. $f(\bX\cup \bW)$.
\end{theorem}

\noindent
\textbf{ATE Computation over Summary DAGs}:
We outline how to compute ATE (see Section \ref{sec:background}) directly on the summary DAG. If the treatment or outcome is in a cluster node of $\qg$, we estimate $ATE(U,V)$ over the \groundedDAG\ $\groundedQG$. To minimize the adjustment set, $U$ is ordered before all nodes in its cluster in $\groundedQG$. Alternatively, an upper and lower bound can be derived by considering all subsets in $U$'s cluster in $\qg$.

\section{Robustness against
DAG quality}
\label{sec:robustness}

We evaluate the effectiveness of summary DAGs in providing robustness against a flawed input causal DAG. In a case study, we demonstrate that the summary DAG facilitates the handling of errors in the input DAG more effectively than directly examining the causal DAG (which may be overwhelming to the user). This study emphasizes that causal DAG summarization helps address quality issues and increases robustness against misspecifications.

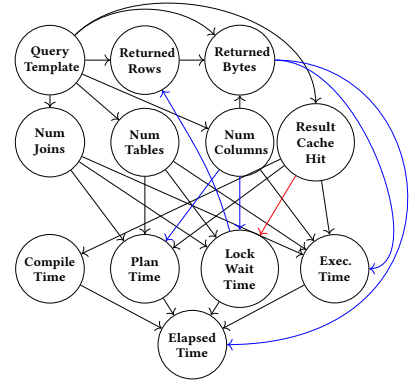
\begin{figure}
\centering

 \scriptsize
 \resizebox{0.7\linewidth}{!}{
       \begin{tikzpicture}[
    every node/.style={
        minimum size=0.5cm, 
        font=\bfseries, 
        align=center, 
         text width=0.8cm 
    },
    vertex/.style={draw, circle}
]
\clip (-1,1) rectangle (6, -5.2);

\node[vertex] at (0,0) (Temp) {\scriptsize Query\\Template};
\node[vertex] at (1.5,0) (Rows) {\scriptsize Returned\\Rows};
\node[vertex] at (3,0) (Bytes) {\scriptsize Returned\\Bytes};

\node[vertex] at (0,-1.3) (Joins) {\scriptsize Num\\Joins};
\node[vertex] at (1.5,-1.3) (Tables) {\scriptsize Num\\Tables};
\node[vertex] at (3,-1.3) (Columns) {\scriptsize Num\\Columns};
\node[vertex] at (4.2,-1.3) (Columns) (Cache) {\scriptsize Result\\Cache\\Hit};

\node[vertex] at (0,-3.3) (Compile) {\scriptsize Compile\\Time};
\node[vertex] at (1.5,-3.3)(Plan) {\scriptsize Plan\\Time};
\node[vertex] at (3,-3.3)(Lock) {\scriptsize Lock\\Wait\\Time};
\node[vertex] at (4.5,-3.3) (Execute) {\scriptsize Exec.\\Time};

\node[vertex] at (2.25,-4.5) (Elapsed) {\scriptsize Elapsed\\Time};

\draw[->] (Temp) -- (Rows);
\draw[->] (Temp) to[out=45, in=135, looseness=0.9] (Bytes);
\draw[->] (Rows) -- (Bytes);
\draw[->] (Temp) -- (Tables);
\draw[->] (Temp) -- (Columns);
\draw[->] (Temp) -- (Joins);

\draw[->] (Tables) -- (Plan);
\draw[->] (Joins) -- (Plan);
\draw[->] (Tables) -- (Lock);
\draw[->] (Joins) -- (Lock);
\draw[->] (Tables) -- (Execute);
\draw[->] (Joins) -- (Execute);
\draw[->] (Columns) -- (Execute);
\draw[->] (Columns) -- (Bytes);

\draw[->] (Temp) to[out=45, in=90, looseness=0.95]  (Cache);
\draw[->] (Cache) -- (Compile);
\draw[->] (Cache) -- (Plan);
\draw[->] (Cache) -- (Execute);
\draw[->, red] (Cache) -- (Lock);

\draw[->] (Compile) -- (Elapsed);
\draw[->] (Plan) -- (Elapsed);
\draw[->] (Execute) -- (Elapsed);
\draw[->] (Lock) -- (Elapsed);

\draw[->, blue] (Lock) to[out=105, in=-60] (Rows);
\draw[->, blue] (Columns) -- (Lock);
\draw[->, blue] (Columns) -- (Plan);
\draw[->, blue] (Bytes) to[out=0, in=0, in looseness=0.7] (Execute);
\draw[->, blue] (Bytes) to[out=0, in=0, out looseness=1.5, in looseness=2.2] (Elapsed);

\end{tikzpicture}}
\vspace{-5mm}
\caption{Modifications to the \redshift{} DAG (Fig. \ref{fig:example_causal_dag}).} \label{fig:editted_redshift_dag}
\end{figure}

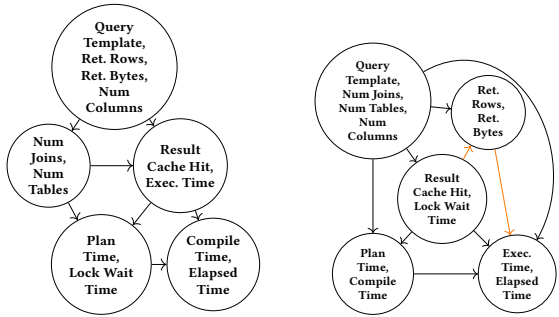
\begin{figure}[t] \scriptsize \centering
    \begin{subfigure}[b]{0.24\textwidth} \centering
    \resizebox{0.8\linewidth}{!}{
         \begin{tikzpicture}[
            node distance=0.5cm,
            every node/.style={
                minimum size=0.7cm, 
                font=\bfseries, 
                align=center, 
            },
            vertex/.style={draw, circle}
        ]

            \node[vertex] at (1,0) (Temp) {\scriptsize Query\\Template,\\Ret. Rows,\\Ret. Bytes,\\ Num\\Columns};
            \node[vertex] at (0, -1.5) (Joins) {\scriptsize Num\\Joins,\\Num\\Tables};
            \node[vertex] at (2, -1.5)(Cache) {\scriptsize Result\\Cache Hit,\\Exec. Time};
            \node[vertex] at (0.8, -3) (Plan) {\scriptsize Plan\\Time,\\Lock Wait\\Time};
            \node[vertex] at (2.5,-3) (Comp) {\scriptsize Compile\\Time,\\Elapsed\\Time};
            
            \draw[->] (Temp) -- (Joins);
            \draw[->] (Temp) -- (Cache);
            \draw[->] (Joins) -- (Cache);
            \draw[->] (Joins) -- (Plan);
            \draw[->] (Cache) -- (Plan);
            \draw[->] (Cache) -- (Comp);
            \draw[->] (Plan) -- (Comp);
        \end{tikzpicture}}
        \caption{Summary After Deletion}
        \label{fig:editted_redshift_dag_summary_deletion}
    \end{subfigure}
    \begin{subfigure}[b]{0.23\textwidth} \centering
    \resizebox{0.9\linewidth}{!}{
       \begin{tikzpicture}[
            node distance=0.5cm,
            every node/.style={
                minimum size=0.7cm, 
                font=\bfseries, 
                align=center, 
            },
            vertex/.style={draw, circle}
        ]

            \node[vertex] at (0,0) (Temp) {\scriptsize Query\\Template,\\ Num Joins,\\Num Tables,\\ Num\\Columns};
            \node[vertex] at (2, -0.2) (Returned) {\scriptsize Ret.\\Rows,\\Ret.\\Bytes};
            \node[vertex] at (1.2, -1.7)(Cache) {\scriptsize Result\\Cache Hit,\\Lock Wait\\Time};
            \node[vertex] at (0, -3) (Plan) {\scriptsize Plan\\Time,\\Compile\\Time};
            \node[vertex] at (2.5,-3) (Exec) {\scriptsize Exec.\\Time,\\Elapsed\\Time};
            
            \draw[->] (Temp) -- (Returned);
            \draw[->] (Temp) -- (Cache);
            \draw[->] (Temp) -- (Plan);
            \draw[->] (Temp) to[out=30, in=60, looseness=1.3] (Exec);
            \draw[->] (Cache) -- (Plan);
            \draw[->] (Cache) -- (Exec);
            \draw[->, orange] (Cache) -- (Returned);
            \draw[->] (Plan) -- (Exec);
            \draw[->, orange] (Returned) -- (Exec);
        \end{tikzpicture}}
        \caption{Summary After Additions}
        \label{fig:editted_redshift_dag_summary_addition}
    \end{subfigure}
    \vspace{-3mm}
    \caption{$5$-node summary DAGs after DAG modifications. }
    \label{fig:editted_redshift_dag_summary}
    
\end{figure}

We revisit the \redshift{} causal DAG (Fig. \ref{fig:example_causal_dag}). For each variable pair, we consulted GPT-4~\cite{openai2023gpt4} about the edge presence and direction, resulting in 55 detected edges. GPT-4 correctly identified 21 of the 23 original edges, inverted 1, and missed 1. It also generated 33 additional edges not in the original DAG. We will demonstrate how causal DAG summarization can reduce the impact of these errors.

\noindent
\textbf{Missing Edges}: Starting from the \redshift{} DAG (Fig.~\ref{fig:example_causal_dag}), we remove the edge GPT-4 failed to detect (\verb|Result Cache Hit| $\rightarrow$ \verb|Lock Wait Time|), marked in red in Fig.~\ref{fig:editted_redshift_dag}. 
As evident in Fig.~\ref{fig:editted_redshift_dag_summary_deletion}, \algoName{} produces the same summary DAG as in Fig.~\ref{fig:summary_graphs_cagres}. The information of which node in the cluster \verb|{Results Cache Hit|, \verb|Exec. Time}| has a directed edge to one of the nodes in the cluster \verb|{Plan Time|, \verb|Lock Time}| is lost upon summarization. Any causal estimation performed over the summary DAG considers all possible causal DAGs compatible with this summary DAG, including once where the edge is included. Thus, the impact of this error is reduced.

\noindent
\textbf{Extraneous Edges}: Starting again from the \redshift{} DAG, we add 5 random edges from the set of redundant edges produced by GPT-4, marked in blue in Fig.~\ref{fig:editted_redshift_dag}. These additional edges reduce the number of \ci\ entailed by the DAG, which can hurt causal inference accuracy. However, manually pruning extraneous edges would require having the user check each of the (now 28) edges in the DAG for correctness. If we instead summarize the DAG using \algoName{} with $k=5$, the user is faced with the simpler, 9-edge summary DAG shown in Fig.~\ref{fig:editted_redshift_dag_summary_addition}. It is sufficient for the user to detect the 2 suspicious orange edges among these 9 to discover 3 of the 5 extraneous edges. The remaining 2 extraneous edges (from \verb|Num Columns| to \verb|Plan Time| and \verb|Lock Wait Time|) are subsumed grouping \verb|Num Columns| together with other, highly semantically similar, query-related features. As such, graph summarization effectively helps address extraneous edges by facilitating their detection.

\section{Experimental Evaluation}
\label{sec:exp}


We empirically demonstrate the following claims:
(\textbf{C1}) Our summary DAGs support reliable causal inference. (\textbf{C2}) Our objective evaluation method effectively determines superior summary DAGs. (\textbf{C3}) \algoName\ outperforms other methods and achieves efficient performance. \revc{(\textbf{C4}) Our proposed optimizations help improve the runtime of \algoName\ without compromising quality.}

\begin{table}[t]
\centering
\footnotesize
\caption{Datasets}
\label{tab:data}
\vspace{-5mm}
\begin{tabular}{lccc}
\toprule
\textbf{Dataset} & \textbf{\# Nodes (Variables)} & \textbf{\# Edges}&\textbf{\# Tuples} \\
\midrule
\redshift& 12&23&9900\\
\flights & 11 & 15 &1M \\
\adult & 13& 48 &32.5K\\
\german & 21 & 43 &1000\\
\accidents & 41& 368 &2.8M\\
\revc{\urls} &60&310&1.7M\\
\bottomrule
\end{tabular}
\end{table}

\subsection{Experimental Setting}\label{sec:exp-setting}
All algorithms are implemented in Python 3.7. Causal effect computation was performed using the DoWhy library~\cite{dowhypaper}. 
The experiments were
executed on a PC with a
$4.8$GHz CPU, and $16$GB memory. Our code and datasets are available at \cite{full}.

\vspace{1mm}
\noindent
{\bf Datasets}.
We examine six datasets, as shown in Table \ref{tab:data}. Five of the datasets are publicly available, while \redshift was collected by running a publicly available benchmark on publicly available cloud resources. We use the DAG in Fig.~\ref{fig:example_causal_dag} for \redshift\ and build the causal DAGs using~\cite{youngmann2023causal} for the remaining datasets.
\textbf{\redshift}: A dataset collected by running queries from the TPC-DS benchmark~\cite{poess2002tpc} on Amazon Redshift Serverless~\cite{redshift-serverless}.
We execute 100 queries from the query benchmark and retrieve the associated entries in the monitoring view~\cite{sysqueryhistory}. 
\textbf{\flights}~\cite{flights}: a dataset describing domestic flight statistics in the US. We enriched it with attributes describing the weather, population, and the airline carriers.
\textbf{\adult}~\cite{adult}: a dataset comprises demographic information of individuals including their education, age, and income.
\textbf{\german}~\cite{asuncion2007uci}: a dataset that contains details of bank account holders, including demographic and financial information.
\textbf{\accidents}~\cite{moosavi2019countrywide}: This dataset includes key factors influencing car accident severity, such as weather and traffic signs.
\textbf{\urls}~\cite{urls}: a dataset containing descriptions of malicious and non-malicious URLs. It encompasses properties such as URL length, the number of digits, and the occurrence of sensitive words.

We also created \textbf{synthetic data} using the DoWhy package~\cite{dowhypaper}, enabling manipulation of node count, edge count, and data size.

\vspace{1mm}
\noindent
{\bf Baselines}.
We examine the following baselines:
\textcolor{red}{\textbf{\brutef}}: This algorithm implements an exhaustive search over all possible summary DAGs that satisfy the constraint yielding the optimal solution.
\textcolor{darkgreen}{\textbf{\ksnap}}~\cite{tian2008efficient}: A general-purpose graph summarization algorithm that employs bottom-up node contractions (akin to \algoName). The primary distinction lies in the objective function: \ksnap\ focuses on ensuring homogeneity among nodes within a cluster.
We have enhanced \ksnap\ to address acyclicity.
\textcolor{darkyellow}{\textbf{\tc}} 
In \cite{tikka2021clustering}, the authors proposed Transit Clusters as a specific type of summary causal DAG that maintains identifiability properties under certain conditions. They introduced an algorithm to identify all transit clusters for a graph. For a fair comparison, we consider the transit cluster that meets the constraints and has the maximal RB.
\textcolor{gray}{\textbf{\cic}}~\cite{niulearning} The authors of \cite{niulearning} proposed a Clustering Information Criterion (CIC) that represents various complex interactions
among variables in a causal DAG. Based on this criterion, they developed a greedy-based approach to learn clustered causal DAGs directly from the data. 
\textcolor{darkpurple}{\textbf{\rand}}: 
As a sanity check, this algorithm generates a random summary DAG that adheres to the size constraint.

\vspace{1mm}
\noindent
{\bf Metrics of evaluation}. 
In some cases, summary DAGs are incomparable, meaning their RBs do not strictly imply one another. To assess quality, we count additional edges in the \groundedDAG\ absent from the original DAG—fewer edges indicate a sparser summary DAG encoding more \ci.




\vspace{1mm}
As a default configuration, we set the size constraint $k$ to $\frac{n}{2}$, where $n$ is the number of nodes in the input causal DAG. The runtime cutoff was set at $1$ hour.

\subsection{Usability Evaluation (C1)}
\label{subsec:usability}

\subsubsection{The utility of the summary causal DAGs for causal inference}
\label{subsec:utility}
We assess the utility of the summary causal DAGs for causal inference. 
To this end, we compare the causal effects estimated within the original DAG with those computed within the summary causal DAGs.
Each causal effect estimation yields an interval (of 95\% confidence). We compare the intervals derived from the input DAG (the ground truth) with those obtained by the baselines. Given that the adjustment sets in the summary DAGs may differ from those in the original DAG, we anticipate getting different intervals.


\noindent
\textbf{Average Percentage Overlap}:
We report the average percentage of overlap of the causal interval across all node pairs connected by a causal path in the input DAG. A higher percentage overlap indicates greater robustness in causal inference. The results for \flights\ and \german\ are shown in Fig. \ref{fig:percentage_ovelap} (similar trends were observed for the other datasets).
\algoName's average percentage overlap is close to that of \brutef, suggesting a high degree of similarity between the two summary DAGs. \algoName\ surpasses all other competitors. This underscores the superior suitability of \algoName\ for causal inference compared to the baselines.  


\begin{figure}
  \centering
   \resizebox{0.44\linewidth}{!}{
  \begin{subfigure}[b]{0.22\textwidth}
    \centering
    \includegraphics[scale=0.21]{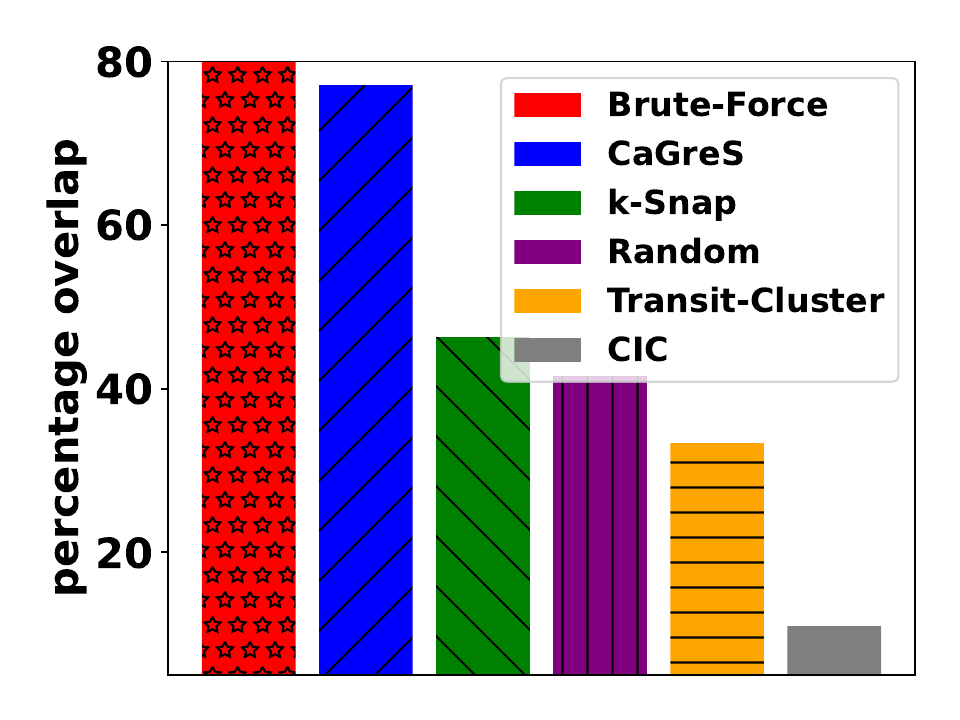}
    \vspace{-3mm}
    \caption{\flights}
    \label{fig:subfig1}
  \end{subfigure}}
 \resizebox{0.44\linewidth}{!}{
  \begin{subfigure}[b]{0.22\textwidth}
    \centering
    \includegraphics[scale=0.21]{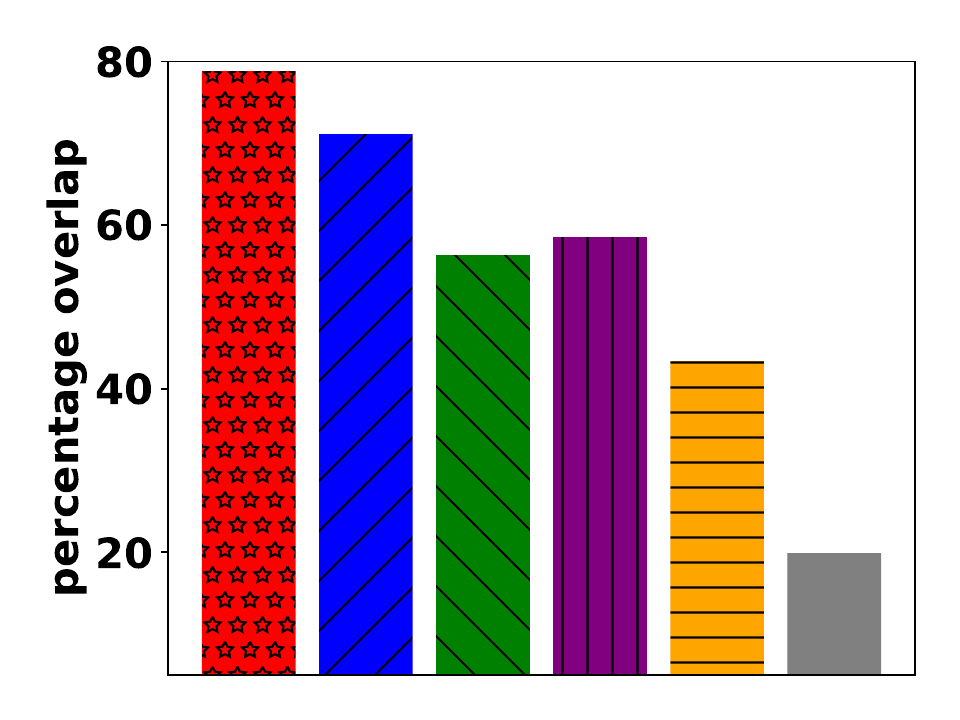}
    \vspace{-3mm}
    \caption{\german}
    \label{fig:subfig2}
  \end{subfigure}}
  \vspace{-4mm}
  \caption{Average percentage overlap with ground truth.  }
  \label{fig:percentage_ovelap}
\end{figure}

In what comes next, we use synthetic data, allowing us to manage the number of nodes in the input DAG and database tuples. We omit from presentation the \brutef, \tc, and \cic\ baselines as they exceeded our time limit cutoff.

\noindent
\textbf{\# of attributes}:
We examine how the number of nodes in the input causal DAG affects the performance. With a larger number of nodes, the task of finding the optimal summary DAG becomes harder. Here, the number of data tuples is fixed at $10$K. 
The results are depicted in Fig. \ref{fig:percentage_ovelap_2}(a). For all baselines, with more data attributes, their alignment with the input causal DAG diminishes. Nevertheless, \algoName\ consistently outperforms the competing methods. 

\begin{figure}
  \centering
   \resizebox{0.43\linewidth}{!}{
  \begin{subfigure}[b]{0.23\textwidth}
    \centering
    \includegraphics[scale=0.23]{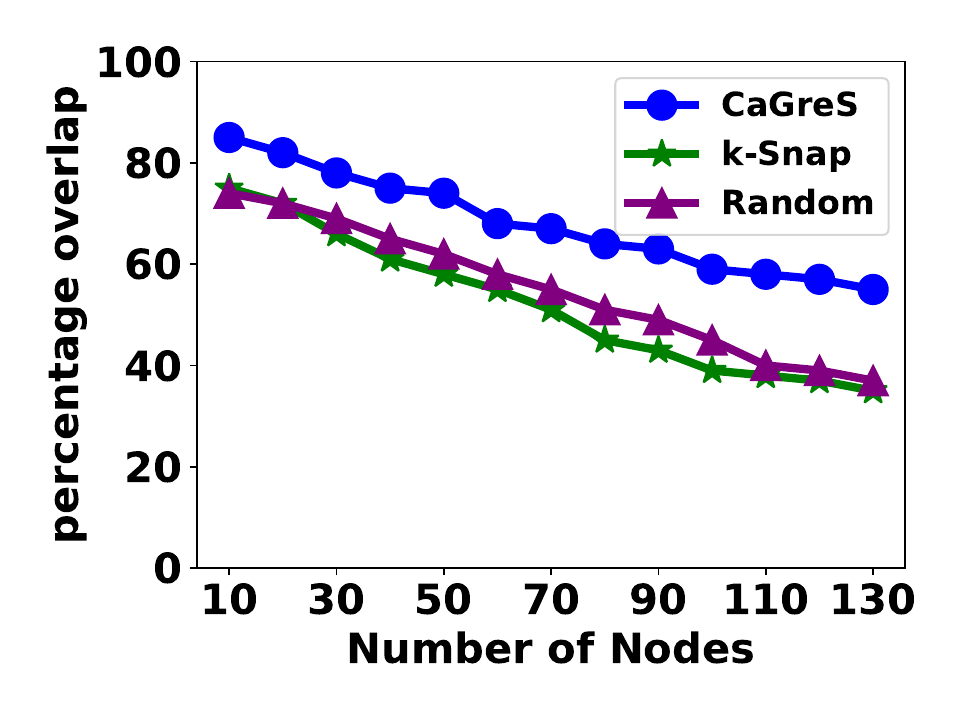}
    \vspace{-4mm}
    \caption{}
    \label{fig:subfig2}
  \end{subfigure}}
       \resizebox{0.43\linewidth}{!}{
  \begin{subfigure}[b]{0.23\textwidth}
    \centering
    \includegraphics[scale=0.23]{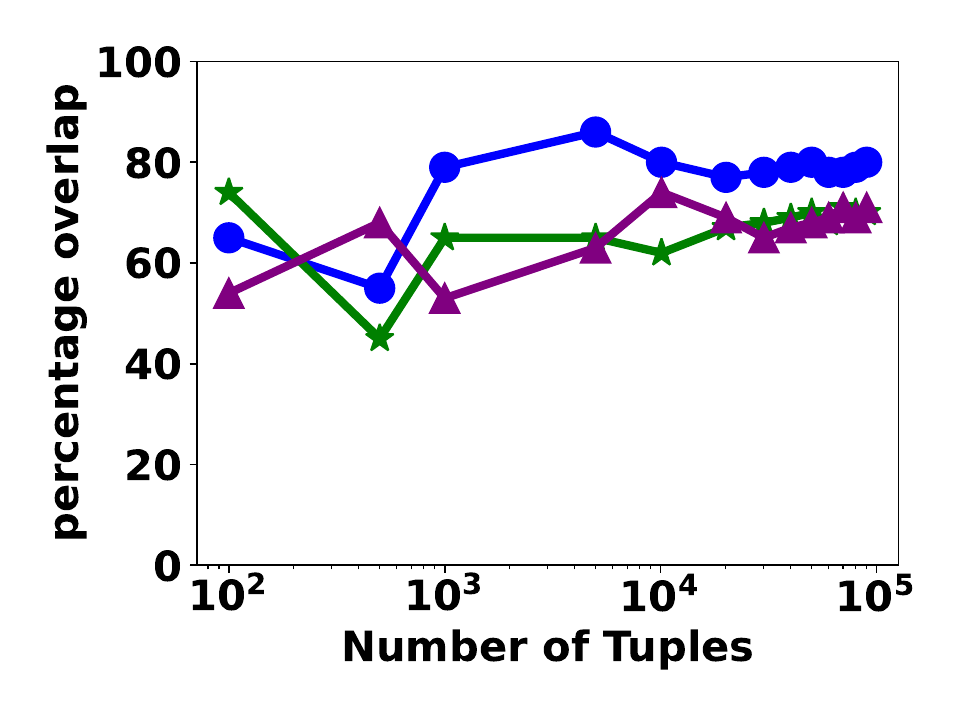}
        \vspace{-4mm}
    \caption{}
    \label{fig:subfig1}
  \end{subfigure}}
  \vspace{-4mm}
  \caption{Average percentage overlap vs. data properties.}
  \label{fig:percentage_ovelap_2}
\end{figure}

\noindent
\textbf{\# of tuples}: 
We analyze the impact of data size on performance, fixing the input causal DAG at 30 nodes. Since causal effects are sensitive to sample size, we expect larger datasets to yield effects on summary DAGs closer to those on the input DAG.  
As shown in Fig. \ref{fig:percentage_ovelap_2}(b), small data sizes produce noisy results, while larger sizes stabilize them.  
\emph{Again, \algoName\ outperforms its competitors}.



\subsection{Quality Evaluation (C2)}
\label{subsec:ci_eval}

\begin{figure}
  \centering
   \resizebox{0.44\linewidth}{!}{
    \begin{subfigure}[b]{0.23\textwidth}
    \centering
    \includegraphics[scale=0.23]{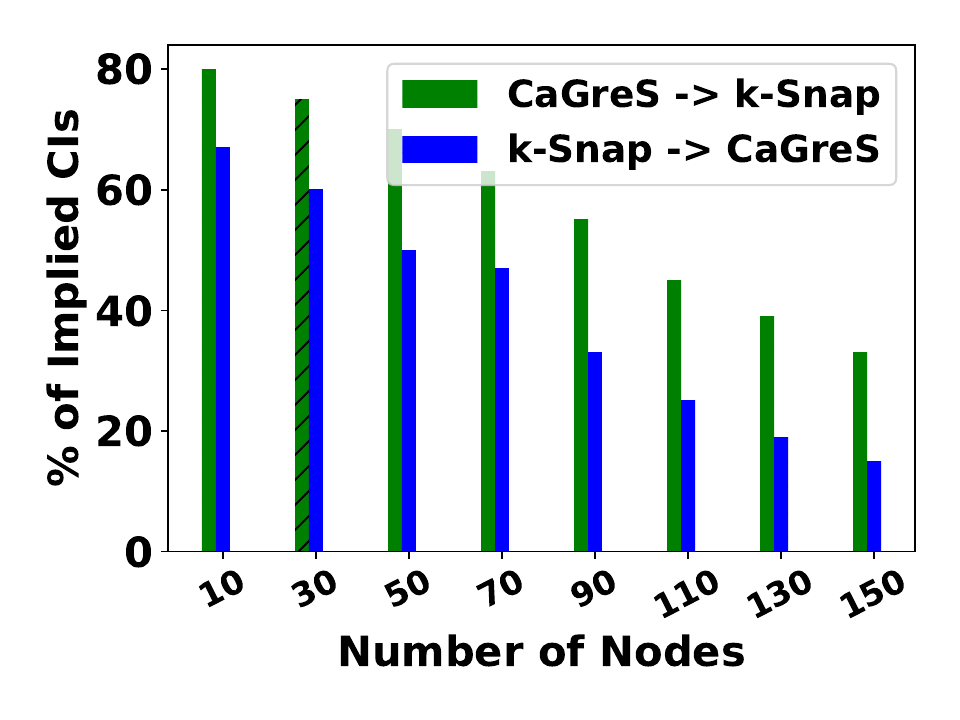}
      \vspace{-4mm}
    \caption{Higher bars are better}
    \label{fig:subfig1}
  \end{subfigure}}
     \resizebox{0.44\linewidth}{!}{
  \begin{subfigure}[b]{0.23\textwidth}
    \centering
    \includegraphics[scale=0.23]{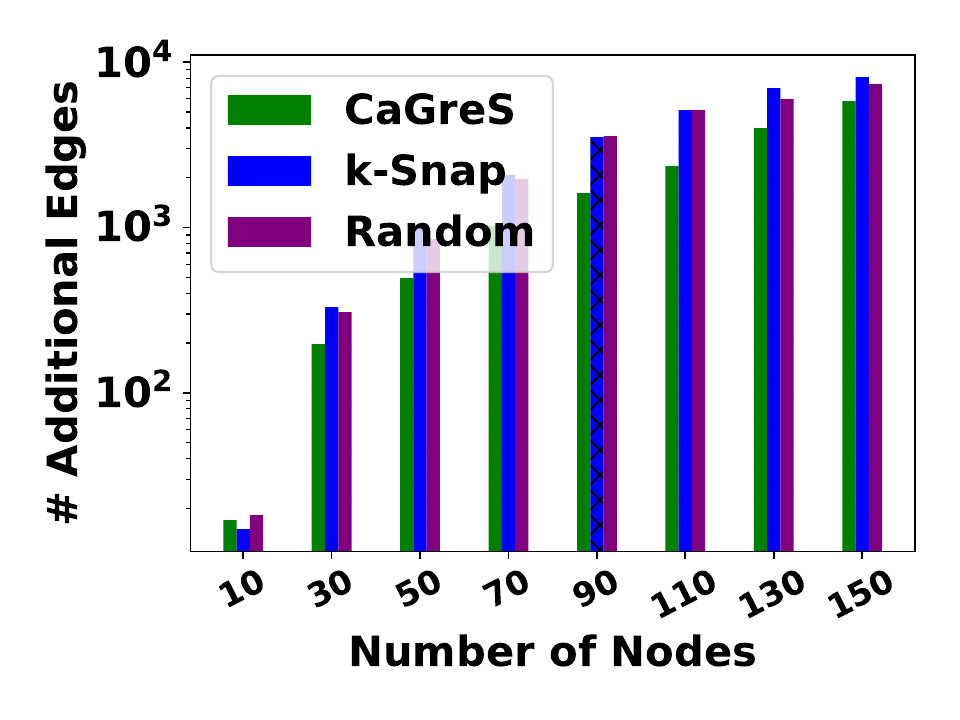}
      \vspace{-4mm}
    \caption{Lower bars are better}
    \label{fig:subfig1}
  \end{subfigure}}

  \vspace{-4mm}
  \caption{Quality metrics vs. the number of nodes.}
  \label{fig:metric}
\end{figure}

When multiple summary DAGs achieve maximal RBs, we use three metrics to identify a superior summary DAG: (1) the percentage of \ci\ in one summary DAG’s RB implied by another; (2) number of additional edges in the \groundedDAG; and (3) the size of adjustment sets in causal estimation, with smaller sets enhancing accuracy. As we show, these metrics are highly correlated.

We generated random causal DAGs with various number of nodes (five DAGs for each node count), while keeping all other parameters fixed.
We omit from presentation the \brutef, \tc, and \cic\ baselines as they exceeded our time cutoff. 
The results are depicted in Fig. \ref{fig:metric}. Fig. \ref{fig:metric}(a) depicts the percentage of \ci\ in the RB of \ksnap\ that are implied by that of \algoName\ and vice versa. Similar trends were observed for \rand. A higher percentage of \ksnap's \ci\ are implied by \algoName\ compared to the percentage of \algoName's \ci\ that are implied by \ksnap.
Hence, while no RB entirely implies the other, we can still conclude that the summary DAG of \algoName\ is superior to that of \ksnap. Fig. \ref{fig:metric}(b) depicts the number of additional edges in the \groundedDAG. \algoName\ consistently yields summary DAGs with fewer edges. 
We also considered the average size of the adjustment sets in the computation of causal estimations (omitted from the presentation).
We report that \algoName\ outperforms the competitors, consistently yielding smaller adjustment sets. 
\emph{Since these metrics are closely interrelated, we deduce that it is appropriate to use the count of additional edges for comparing quality.}


\subsection{Effectiveness Evaluation (C3)}
\label{subsec:exp_effectiveness}
We assess \algoName\ based on quality and runtime performance.

\vspace{1mm}
\noindent
{\bf Case Study: \flights}
We present the pairwise percentage of the \ci\ in the RB implied by all baseline pairs. The results are shown in Table \ref{tab:results_flights}. The summary DAGs obtained by \algoName\ and \ksnap\ are given in Fig. \ref{fig:flights_results} (The optimal summary DAG by \brutef\ is omitted from presentation). \brutef\ yields the most effective summary DAG, as it implies the highest percentage of \ci\ of any other baseline. While 60\% of the \ci\ of \ksnap\ are implied by the RB of \algoName, only 16\% of the \ci\ of \algoName\ are implied by the RB of \ksnap. This superiority of \algoName\ over \ksnap\ is further supported by a lower number of additional edges (7 for \algoName, 13 for \brutef, and 14 for \ksnap). 
Intuitively, this stems from \ksnap's decision to form two $3$-size clusters, connected by an edge. In the resulting \groundedDAG, every pair of nodes within and between the clusters is connected by an edge.

\begin{table}
\scriptsize
  \centering
 
  \caption{Pair-wise percentage of the RB's \ci\ implied.}
  \label{tab:results_flights}
   \vspace{-4mm}
  \begin{tabular}{lcccccc}
    \toprule
    & \brutef & \textcolor{blue}{\algoName} & \ksnap & \rand & \textcolor{darkyellow}{TC} & \cic \\
    \midrule
    \brutef & - & 83.3\% & 50\% & 50\% & 16.6\% & 16.6\% \\
   \textcolor{blue}{\algoName} & 50\%& - & 60\% & 16.6\% & 0\% & 16\% \\
    \ksnap &0\% &16.6\% & - & 50\% & 16.6\% & 0\% \\
    \rand &16.6\% & 0\%&50\% & - & 0\% & 16.6\% \\
    \textcolor{darkyellow}{TC} & 0\%&0\% &16.6\% & 16.6\%& - & 50\% \\
    \cic &0\% &0\% &0\% &16.6\% &0\% & - \\
    \bottomrule
  \end{tabular}
  \vspace{-1mm}
\end{table}

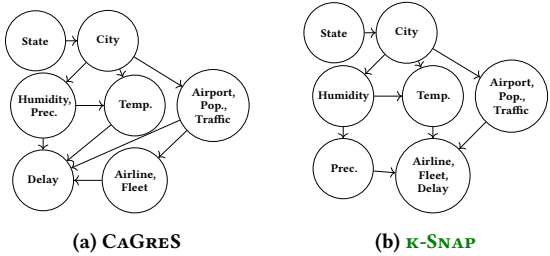
\begin{figure}[t] \scriptsize \centering






    \begin{subfigure}[b]{0.22\textwidth} \centering
 \resizebox{0.85\linewidth}{!}{
             \begin{tikzpicture}[
    node distance=0.5cm,
    every node/.style={
        minimum size=0.8cm, 
        font=\bfseries, 
        align=center, 
         text width=0.9cm 
    },
    vertex/.style={draw, circle}
]

\node[vertex] (State) {\scriptsize State};
\node[vertex, right=0.2cm of State] (City) {\scriptsize City};

\node[vertex, below left=0.5cm of City] (Humidity) {\scriptsize Humidity, Prec.};
\node[vertex, right=0.5cm of Humidity] (Temp) {\scriptsize Temp.};
\node[vertex, below=0.2cm of Humidity] (Prec) {\scriptsize Delay};
\node[vertex, right=0.2cm of Temp] (Airport) {\scriptsize Airport, Pop., Traffic};
\node[vertex, below=0.2cm of Temp] (Airline) {\scriptsize Airline, Fleet};

 \draw[->] (State) -- (City);
   \draw[->] (City) -- (Airport);
    \draw[->] (City) -- (Humidity);
     \draw[->] (City) -- (Temp);
        \draw[->] (Humidity) -- (Temp);
           \draw[->] (Humidity) -- (Prec);

            \draw[->] (Airline) -- (Prec);
                 \draw[->] (Airport) -- (Airline);
          \draw[->] (Temp) -- (Prec);
            \draw[->] (Airport) -- (Prec);
\end{tikzpicture}  }
        \caption{\algoName}
        \label{fig:ci_optimal_graph}
    \end{subfigure}
    \begin{subfigure}[b]{0.22\textwidth} \centering
 \resizebox{0.85\linewidth}{!}{
             \begin{tikzpicture}[
    node distance=0.5cm,
    every node/.style={
        minimum size=0.8cm, 
        font=\bfseries, 
        align=center, 
         text width=0.9cm 
    },
    vertex/.style={draw, circle}
]

\node[vertex] (State) {\scriptsize State};
\node[vertex, right=0.2cm of State] (City) {\scriptsize City};

\node[vertex, below left=0.5cm of City] (Humidity) {\scriptsize Humidity};
\node[vertex, right=0.5cm of Humidity] (Temp) {\scriptsize Temp.};
\node[vertex, below=0.2cm of Humidity] (Prec) {\scriptsize Prec.};
\node[vertex, right=0.2cm of Temp] (Airport) {\scriptsize Airport, Pop., Traffic};
\node[vertex, below=0.2cm of Temp] (Airline) {\scriptsize Airline, Fleet, Delay};

 \draw[->] (State) -- (City);
   \draw[->] (City) -- (Airport);
    \draw[->] (City) -- (Humidity);
     \draw[->] (City) -- (Temp);
        \draw[->] (Humidity) -- (Temp);
           \draw[->] (Humidity) -- (Prec);

            \draw[->] (Prec) -- (Airline);
                 \draw[->] (Airport) -- (Airline);
          \draw[->] (Temp) -- (Airline);
\end{tikzpicture}  }
        \caption{\ksnap}
        \label{fig:ci_optimal_graph}
    \end{subfigure}
    \vspace{-4mm}
    \caption{Summary causal DAGs for the \flights\ dataset.}
    \label{fig:flights_results}

\end{figure}


\begin{figure*}
 \centering
    \begin{subfigure}[b]{1\textwidth}
    \centering
    \includegraphics[width=\textwidth]{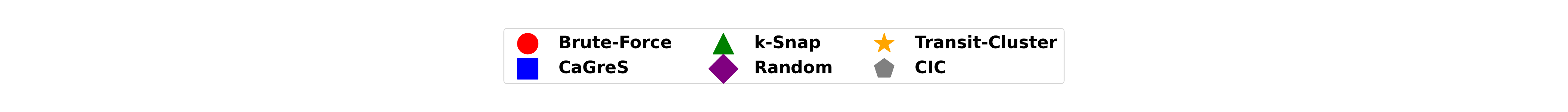}
       \vspace{-7mm}
  \end{subfigure}

  \centering
    \begin{subfigure}[b]{0.2\textwidth}
    \centering
    \includegraphics[width=\textwidth]{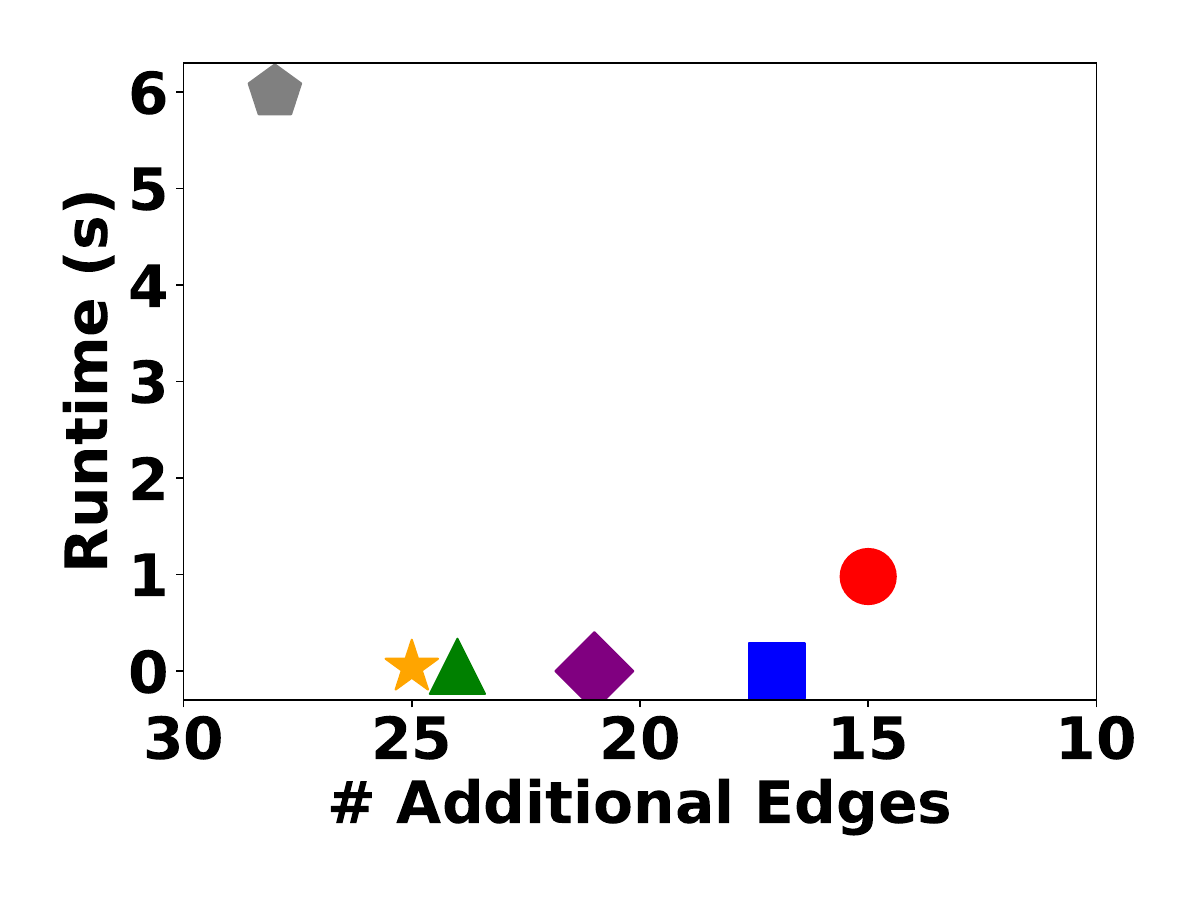}
      \vspace{-8mm}
    \caption{\redshift}
    \label{fig:subfig1}
  \end{subfigure}
    \begin{subfigure}[b]{0.2\textwidth}
    \centering
    \includegraphics[width=\textwidth]{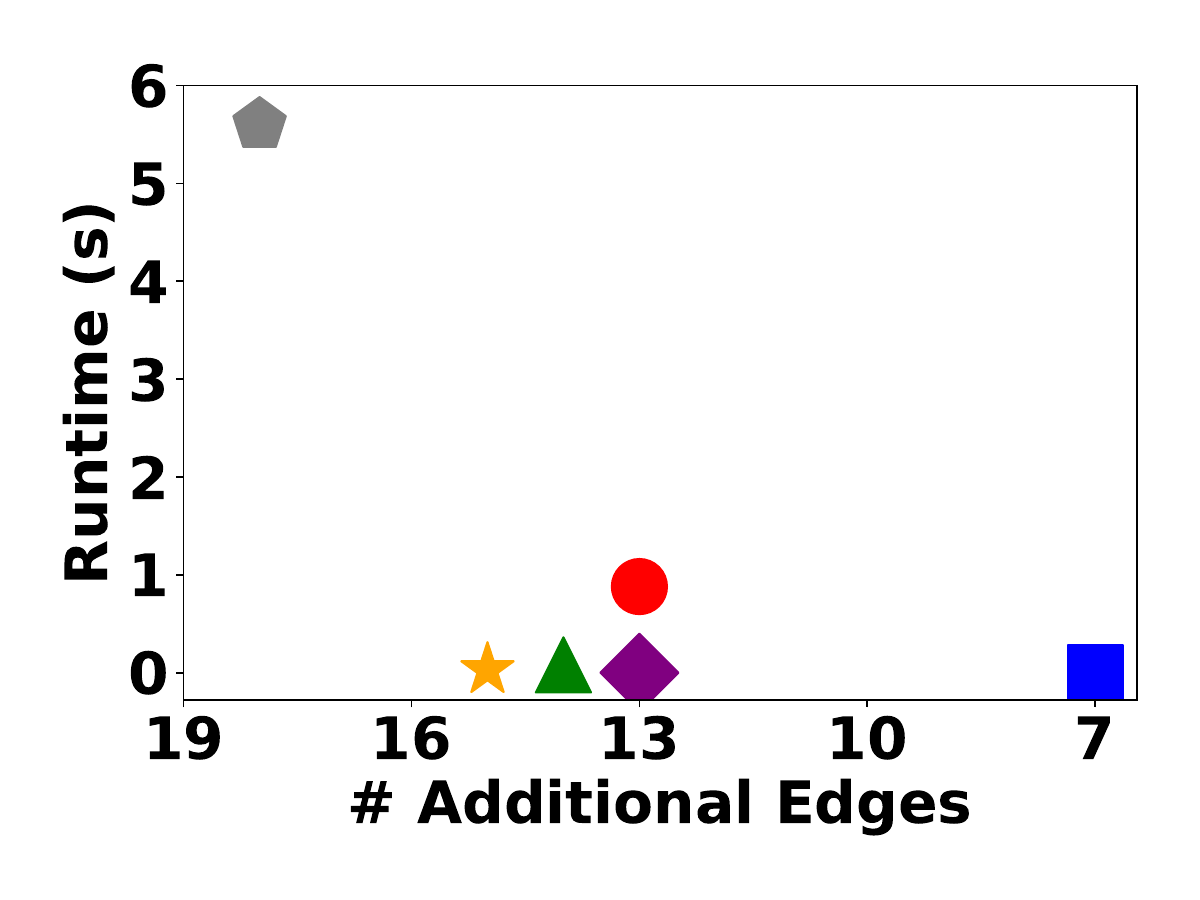}
      \vspace{-8mm}
    \caption{\flights}
    \label{fig:subfig1}
  \end{subfigure}
  \begin{subfigure}[b]{0.2\textwidth}
    \centering
    \includegraphics[width=\textwidth]{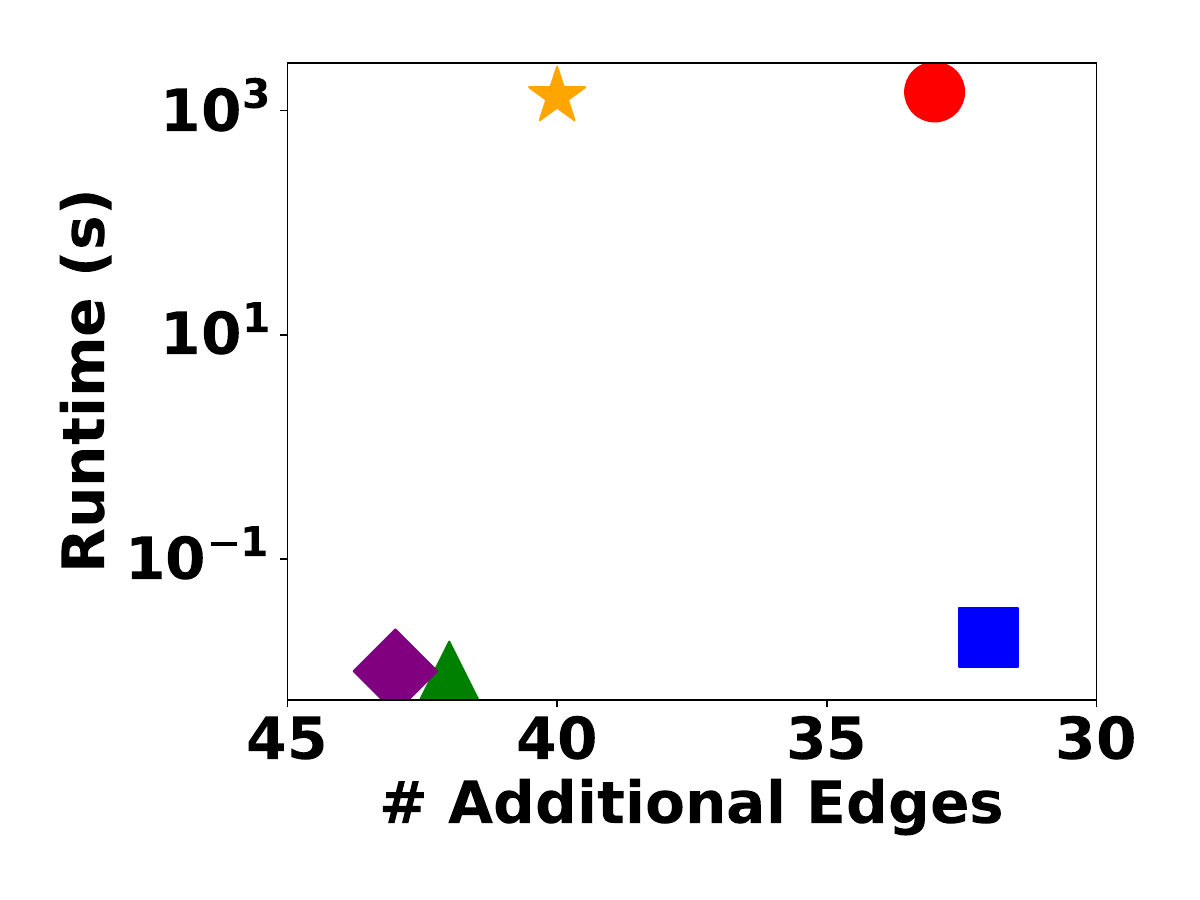}
      \vspace{-8mm}
    \caption{\adult}
    \label{fig:subfig2}
  \end{subfigure}

  \begin{subfigure}[b]{0.2\textwidth}
    \centering
    \includegraphics[width=\textwidth]{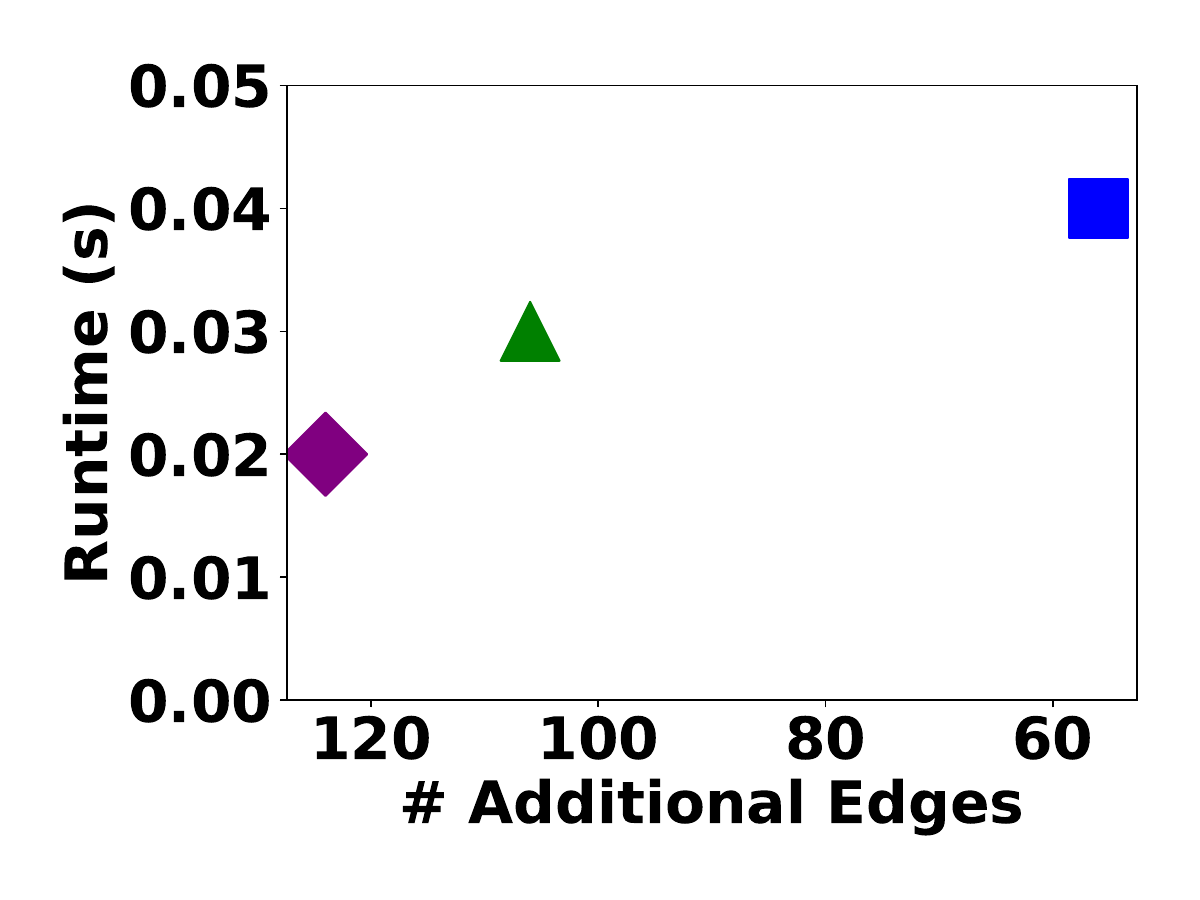}
      \vspace{-8mm}
    \caption{\german}
    \label{fig:subfig3}
  \end{subfigure}
  \begin{subfigure}[b]{0.2\textwidth}
    \centering
    \includegraphics[width=\textwidth]{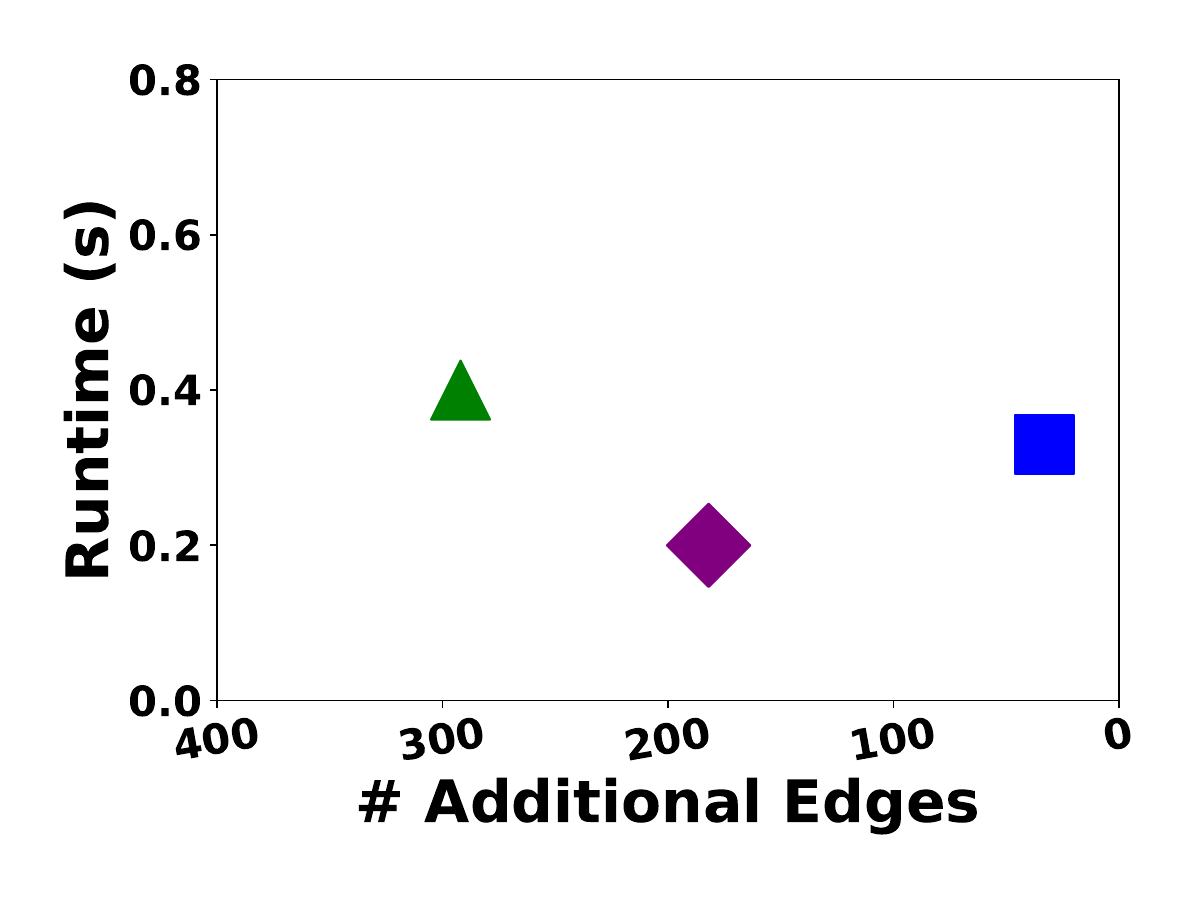}
    \vspace{-8mm}
    \caption{\accidents}
    \label{fig:subfig4}
  \end{subfigure}
   \begin{subfigure}[b]{0.2\textwidth}
    \centering
    \includegraphics[width=\textwidth]{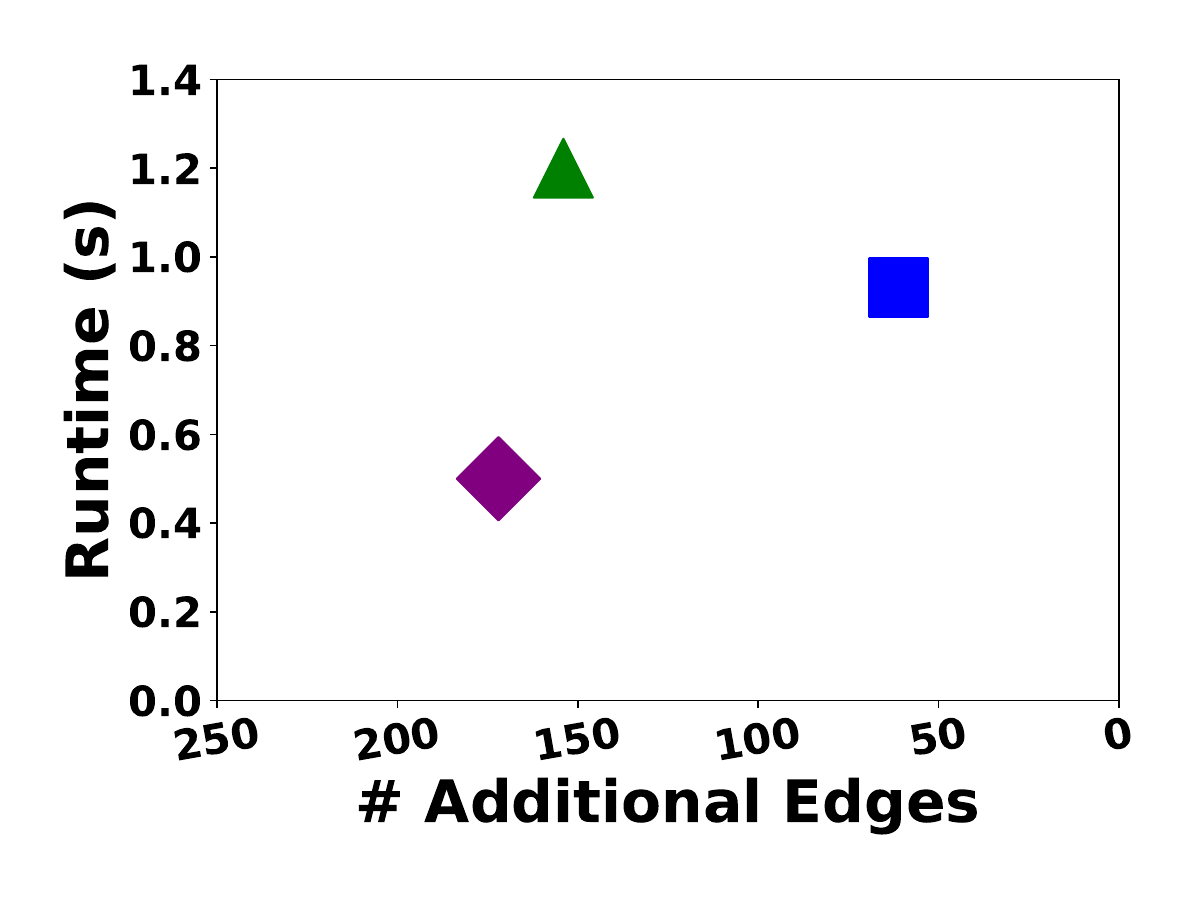}
      \vspace{-8mm}
    \caption{\urls}
    \label{fig:subfig1}
  \end{subfigure}
  \vspace{-4mm}
  \caption{Number of additional edges vs. runtime. The optimal solution should be located in the lower right region. }
  \label{fig:objective_runtime}
  \vspace{-5mm}
\end{figure*}

\vspace{1mm}
Next, for each dataset, we report the runtime and the number of additional edges. The results are depicted in Fig. \ref{fig:objective_runtime}. 
Only \algoName, \ksnap, and \rand\ can handle causal DAGs with more than 20 nodes within a responsive runtime. While \rand\ and \ksnap\ exhibit runtimes comparable to that of \algoName, \algoName\ consistently produces summary DAGs with fewer additional edges. As expected, \brutef\ outperforms \algoName\ in terms of quality but is impractical for interactive interaction. 
\cic\ exhibits relatively low performance, primarily due to a causal discovery component. \tc\ cannot handle large causal DAGs, as the algorithm materializes all transit clusters to select the maximal one.

\begin{figure}
  \centering
   \resizebox{0.44\linewidth}{!}{
  \begin{subfigure}[b]{0.23\textwidth}
    \centering
    \includegraphics[scale=0.23]{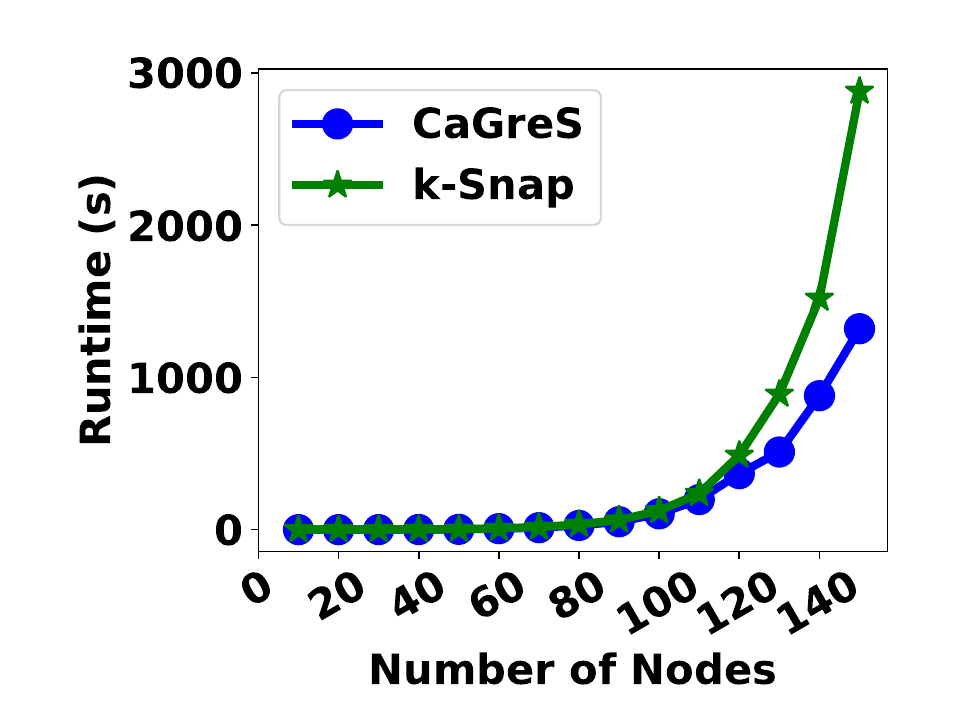}
    \vspace{-3mm}
    \caption{Runtime}
    \label{fig:subfig1}
  \end{subfigure}}
   \resizebox{0.44\linewidth}{!}{
  \begin{subfigure}[b]{0.23\textwidth}
    \centering
    \includegraphics[scale=0.23]{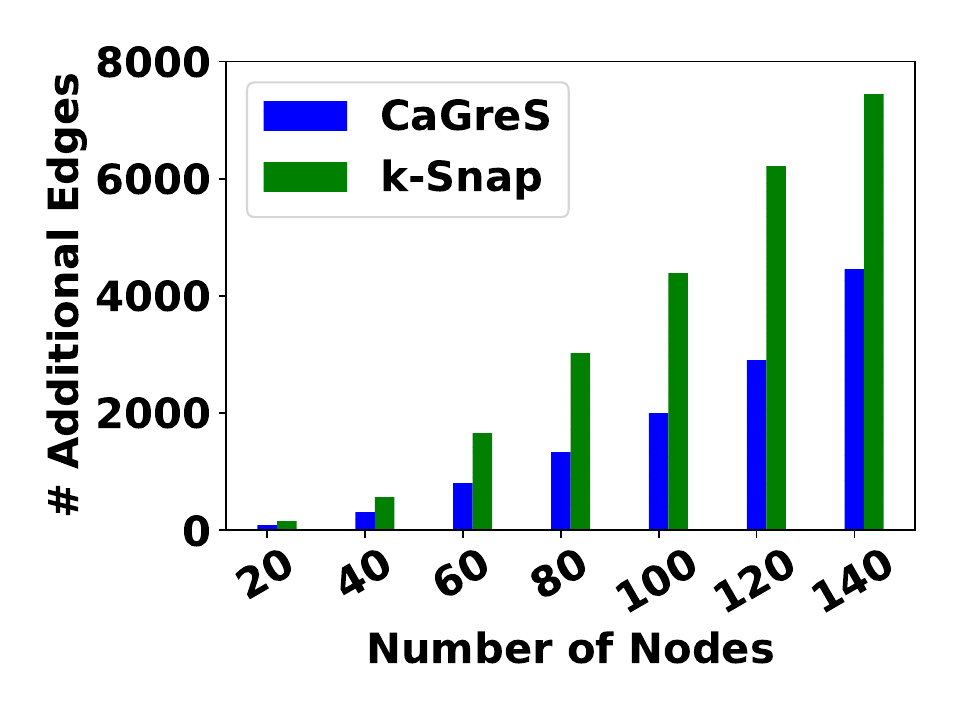}
       \vspace{-3mm}
    \caption{Quality}
    \label{fig:subfig2}
  \end{subfigure}}
     \vspace{-4mm}
  \caption{Number of nodes vs. running times and quality. }
  \label{fig:nodes}
   \vspace{-5mm}
\end{figure}

\begin{figure}
  \centering
    \resizebox{0.44\linewidth}{!}{
  \begin{subfigure}[b]{0.23\textwidth}
    \centering
    \includegraphics[scale=0.23]{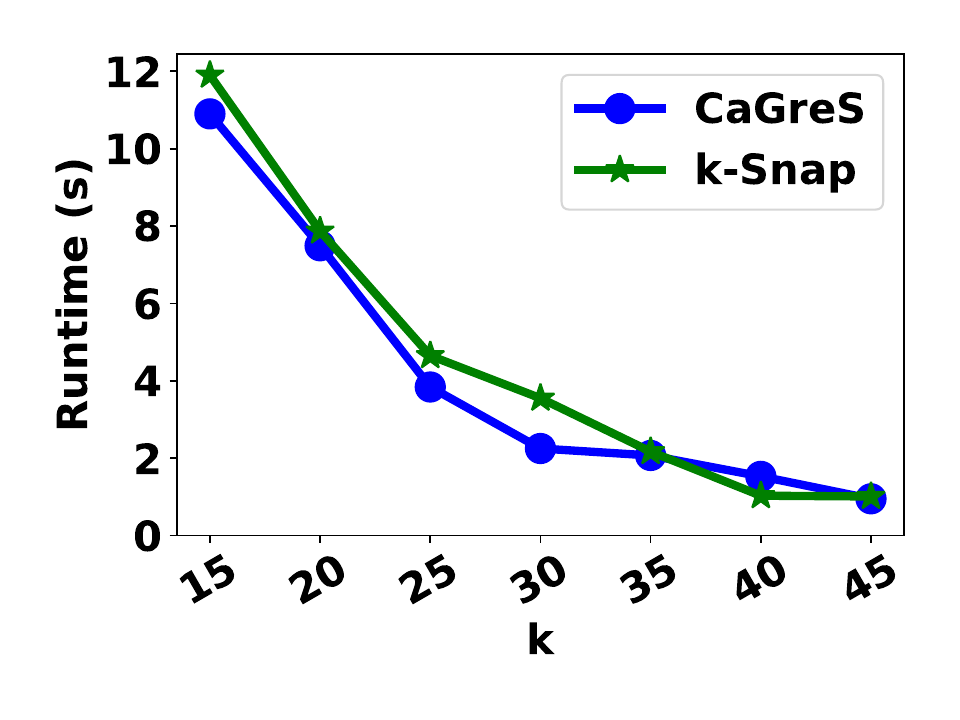}
        \vspace{-4mm}
    \caption{Runtime}
    \label{fig:subfig1}
  \end{subfigure}}
    \resizebox{0.44\linewidth}{!}{
  \begin{subfigure}[b]{0.23\textwidth}
    \centering
    \includegraphics[scale=0.23]{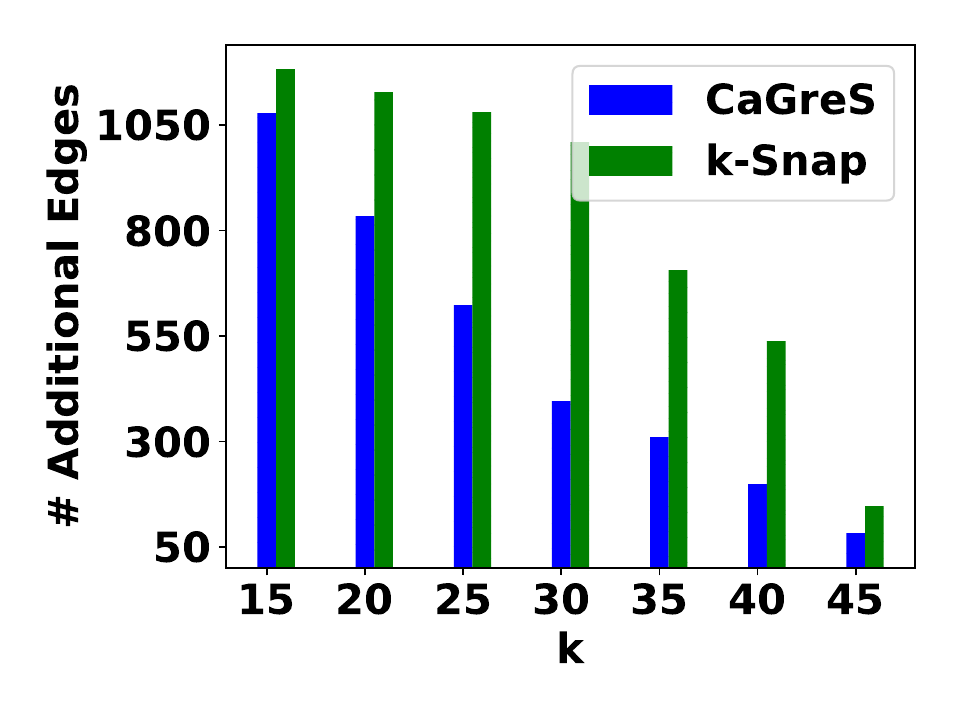}
    \vspace{-4mm}
    \caption{Quality}
    \label{fig:subfig2}
  \end{subfigure}}
  \vspace{-4mm}
  \caption{Summary size $k$ vs. running times and quality. }
  \label{fig:k}
   \vspace{-5mm}
\end{figure}

\begin{figure}
  \centering
   \resizebox{0.44\linewidth}{!}{
  \begin{subfigure}[b]{0.23\textwidth}
    \centering
    \includegraphics[scale=0.23]{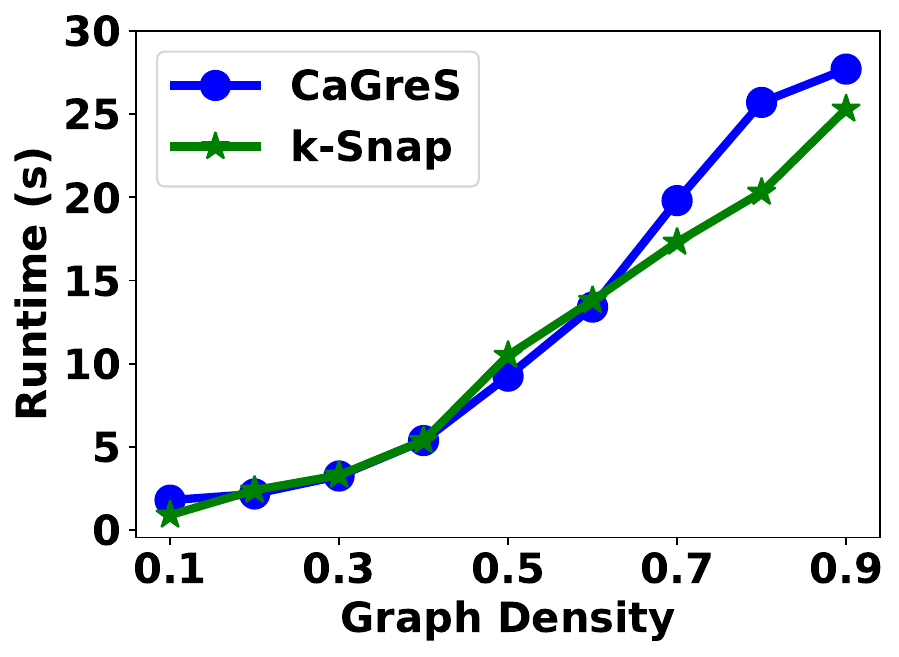}
    \vspace{-3mm}
    \caption{Runtime}
    \label{fig:subfig1}
  \end{subfigure}}
   \resizebox{0.44\linewidth}{!}{
  \begin{subfigure}[b]{0.23\textwidth}
    \centering
    \includegraphics[scale=0.23]{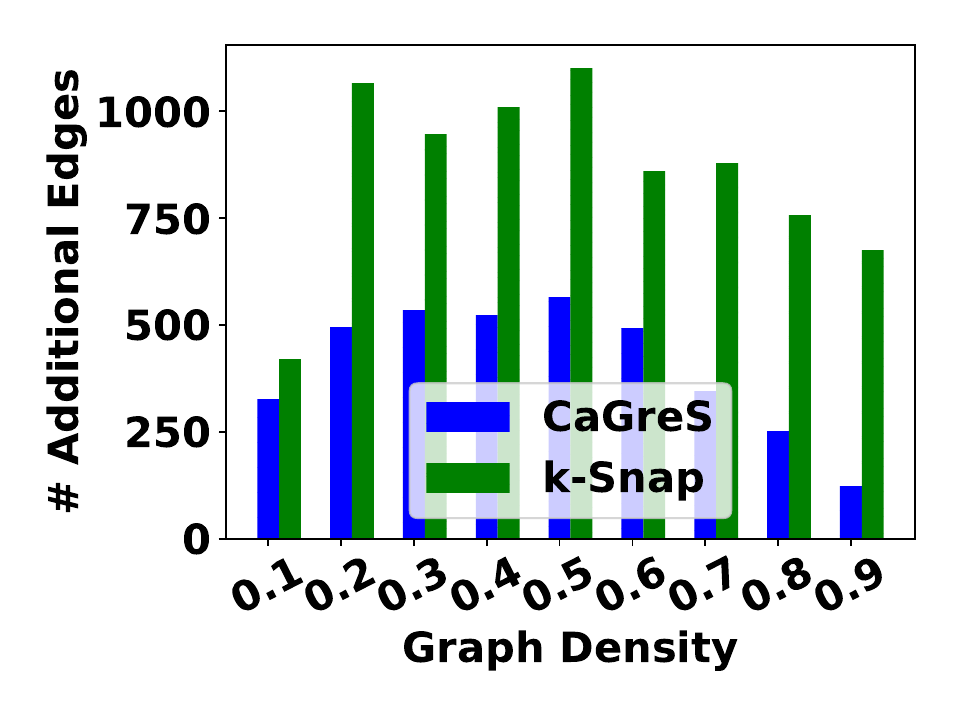}
       \vspace{-3mm}
    \caption{Quality}
    \label{fig:subfig2}
  \end{subfigure}}
     \vspace{-4mm}
  \caption{Graph density vs. running times and quality. }
  \label{fig:density}
     \vspace{-1mm}
\end{figure}

\vspace{0.5mm}
We next analyze the influence of different parameters on performance. In these experiments, our focus shifts to synthetic data, which enables us to manipulate data-related factors.

\noindent
{\bf Input DAG size}
We vary the number of nodes in the input DAG by generating a series of random DAGs varying the number of nodes (5 DAGs per node count) and keeping all other parameters constant. The results are shown in Fig. \ref{fig:nodes}.
As expected, \ksnap\ and \algoName\ exhibit a polynomial increase in runtime (Fig. \ref{fig:nodes}(a)). The improvement relative to \ksnap\ is attributed to our caching mechanisms.
\algoName\ consistently generates summary DAGs with fewer additional edges (Fig. \ref{fig:nodes}(b)), indicating better quality.

\noindent
{\bf Summary size}
We vary the size constraint $k$. Here, the node count is set to 
$50$, and the graph density is held constant at 
$0.3$. The results are depicted in Fig.\ref{fig:k}. The runtime of both \algoName\ and \ksnap\ demonstrate a linear increase with $k$ (Fig. \ref{fig:k}(a)). This is because larger 
$k$ values necessitate more merges. As expected, \algoName\ manages to generate summary DAGs corresponding to canonical causal DAGs with fewer edges (Fig. \ref{fig:k}(b)).

\noindent
{\bf Graph density}
We investigate the influence of graph density on performance. We observe a nearly linear increase in runtime for both \algoName\ and \ksnap\ as graph density rises (Fig. \ref{fig:density}(a)). This is because both algorithms examine neighboring nodes of each node pair, and higher density increases the number of neighbors.
As density increases, both algorithms add more edges. However, at high densities (above 0.7), fewer edges remain to be added, so the number of additional edges decreases (Fig. \ref{fig:density}(b)).


\noindent
{\bf Data size}
We report that the data size, i.e., number of tuples, has no effect on the performance of \algoName\ and \ksnap. This is because both algorithms only examine the input causal DAGs.

\revc{\subsection{Optimizations (C4)}
\label{subsec:ablation}
We assess the effect of our optimizations on the performance of \algoName. To this end, we examine three variants of \algoName: (I) No Cache, a version of \algoName\ without caching, (ii)  No Preprocessing, a version without the low-cost merge optimization, and (iii) No Optimizations: a variant without either optimization.} 


We used synthetic data to generate causal DAGs, varying node count and density as in Section 8.4. Figure \ref{fig:nodes_ablation} shows the results. The number of additional edges remained constant across baselines, so we omit that plot. This confirms that \emph{our optimizations enhance runtime without sacrificing quality}. Notably, \emph{caching improves runtime nearly threefold}, while preprocessing has a modest effect. For large, dense DAGs, preprocessing may slightly slow the algorithm but benefits smaller, sparser ones. 


\begin{figure}
  \centering
   \resizebox{0.48\linewidth}{!}{
  \begin{subfigure}[b]{0.25\textwidth}
    \centering
    \includegraphics[scale=0.18]{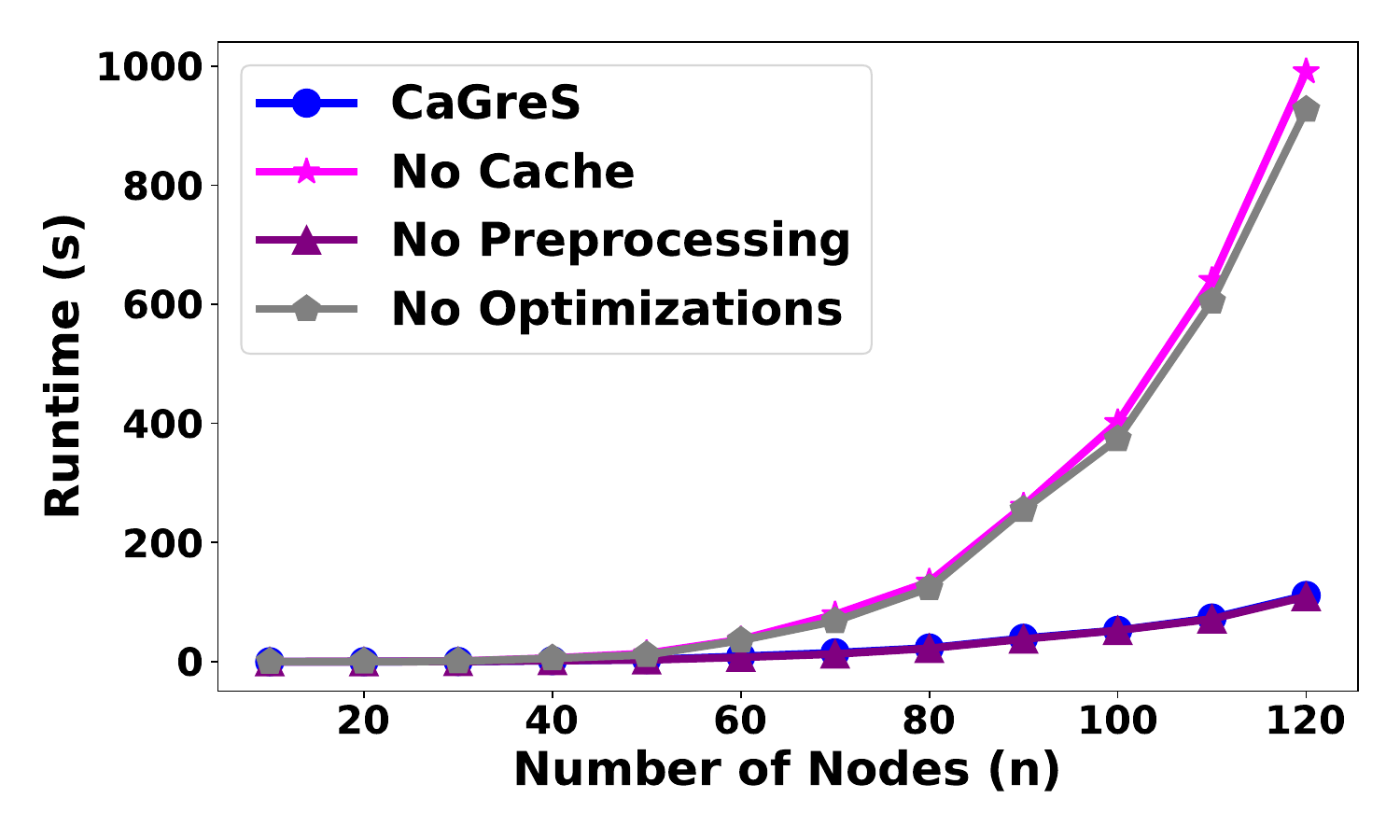}
    \vspace{-5mm}
    \caption{Varying \# nodes (density = 0.4)}
    \label{fig:subfig1}
  \end{subfigure}}
   \resizebox{0.48\linewidth}{!}{
  \begin{subfigure}[b]{0.25\textwidth}
    \centering
    \includegraphics[scale=0.18]{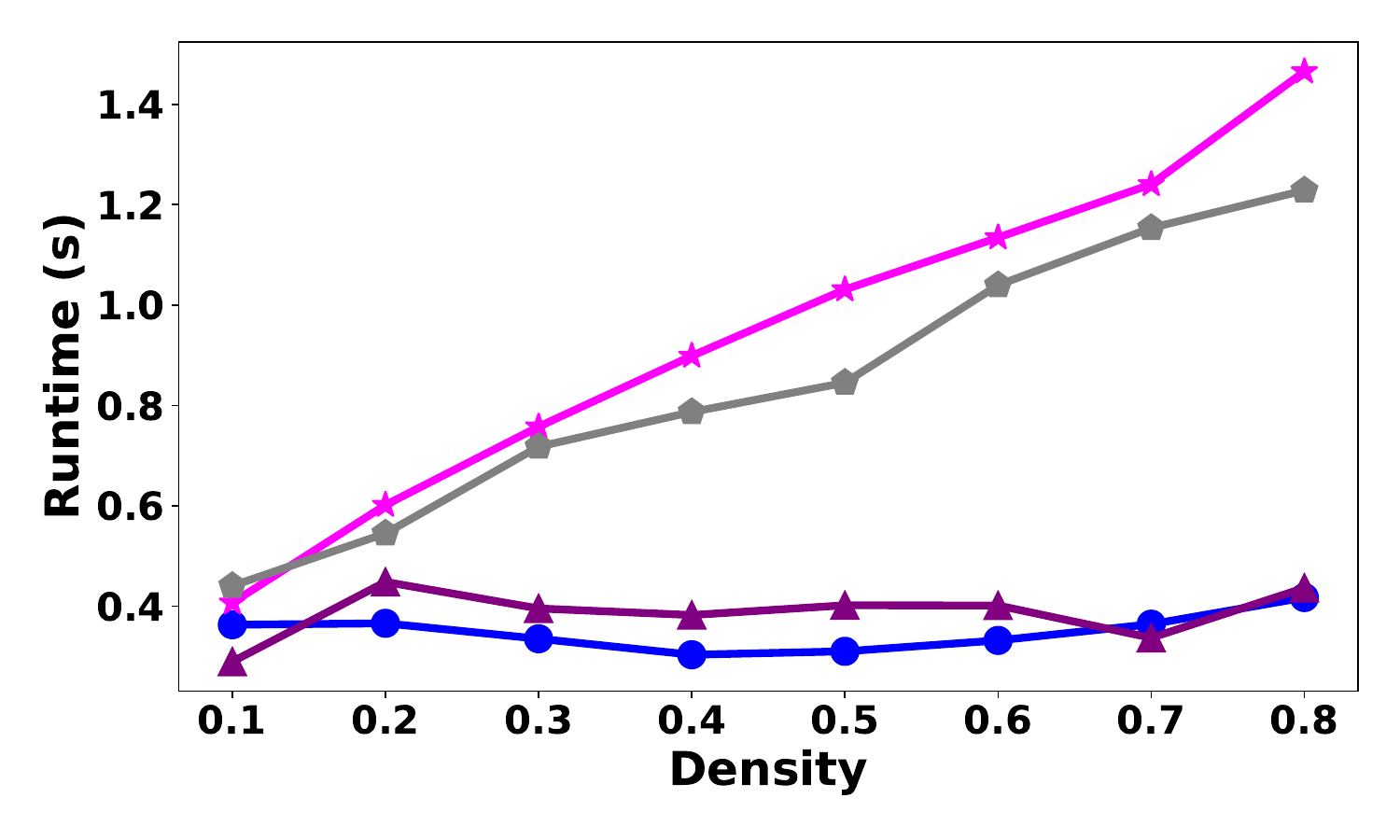}
       \vspace{-5mm}
    \caption{Varying density (n = 30)}
    \label{fig:subfig2}
  \end{subfigure}}
     \vspace{-4mm}
  \caption{\revc{Optimizations.}}
  \label{fig:nodes_ablation}

\end{figure}
\section{Related Work}
\label{sec:related_work}

\noindent
{\bf Summary Causal DAGs}. 
The abstraction of causal models has been studied in literature~\cite{scholkopf2021toward}. Previous work \cite{parviainen2016bayesian, beckers2019abstracting} investigated the problem of determining under what assumptions DAGs over sets of variables can represent the same \ci.
The authors of \cite{chalupka2014visual, chalupka2016multi, rubenstein2017causal} explored the problem of determining the causes of a target behavior (macro-variable) from micro-variables (e.g., image pixels). Other works consider chain or ancestral causal graphs \cite{zhang2008completeness,lauritzen2002chain}.
\cite{puente2013compressing} presented a method to compress causal graphs by removing nodes to eliminate redundancy. In contrast, our work addresses causal DAG summarization, where some causal information is lost but the summary DAG still supports reliable causal inference. The authors of \cite{anand2023causal} expanded the \emph{do-calculus} framework~\cite{pearl_causality_2000} to clustered causal graphs, a related but distinct concept. Our contribution lies in presenting a more streamlined proof of this principle, relying on the connection between node contraction and edge addition.





\vspace{1mm}
\noindent
{\bf General-Purpose Summary Graphs}
Graph summarization aims to condense an input graph into a more concise representation.
\reva{Graph summarization has been extensively studied within the data management community~\cite{bhowmick2022data,fan2012query,liu2014distributed,maneth2016compressing,navlakha2008graph}, as summary graphs not only reduce the graph's size but also enable efficient query answering~\cite{fan2012query,shin2019sweg,maccioni2016scalable,raghavan2003representing}, and enables enhanced data visualization and pattern discovery \cite{shen2006visual,jin2017ecoviz,koutra2014vog,cook1993substructure,dunne2013motif,li2015modulgraph}, and supports extraction of influence dynamics~\cite{mehmood2013csi}.}
Various techniques have been explored, including grouping nodes based on similarity  \cite{liu2014distributed,navlakha2008graph,shin2019sweg,shoaran2013zero,song2018mining,raghavan2003representing,zhou2009graph,tian2008efficient,lee2020ssumm, yong2021efficient,merchant2023graph}, reducing the number bits required to represent graphs \cite{boldi2004webgraph,rossi2018graphzip,shah2017summarizing,navlakha2008graph,maneth2016compressing}, and removing unimportant nodes/edges \cite{maccioni2016scalable,spielman2008graph}. 
We argue that existing techniques are ill-suited for the \probName\ problem. Graph summarization objectives differ across applications, often prioritizing minimizing the reconstruction error \cite{lee2020ssumm, yong2021efficient}, facilitating accurate query answering \cite{shin2019sweg,maccioni2016scalable}, \revb{or selecting contractions that preserve shortest paths to facilitate routing queries ~\cite{geisberger2012exact}}. 
Consequently, existing methods inadequately cater to the objective of preserving causal information, often yielding graphs unsuitable for causal inference, as shown in Section \ref{sec:intro}. 
\revb{Our algorithm 
follows a previous line of work ~\cite{tian2008efficient,geisberger2012exact}, where a bottom-up greedy approach is used to identify promising node pairs for contraction. Our main contribution lies in how it estimates merge costs related to the causal interpretation of the graph and the objective of preserving causal information. }

\vspace{1mm}
\noindent
{\bf Causal Discovery}. 
Causal discovery is a well-studied problem~\cite{glymour2019review,zhao2023causal,youngmann2023causal}, whose goal is to infer causal relationships among variables. While background knowledge is crucial~\cite{pearl2018book}, causal DAGs can be inferred from data under certain assumptions~\cite{glymour2019review,chickering2002optimal}. Existing methods include constraint-based~\cite{spirtes2000causation} and score-based algorithms~\cite{shimizu2006linear,chickering2002optimal, wiering2002evolving,zhu2019causal}. Pashami et al.~\cite{pashami2018causal} proposed a cluster-based conflict resolution mechanism to determine the causal relationship among variables.
Recent works \cite{castelnovo2024marrying, vashishtha2023causal} have explored the use of LLMs for causal discovery.
Our work serves as a complementary endeavor to existing research in causal discovery.

\section{Conclusions \& Limitations}
\label{sec:conc}
A mixed graph is a typical output of causal discovery algorithms~\cite{pena2016learning,spirtes2000causation,doi:10.1137/20M1362796}. For simplicity in exposition, we concentrated on regular causal DAGs throughout this paper. Nevertheless, our results and algorithms apply to mixed graphs as well.

This paper opens up promising future research directions.
This includes the development of compact representations of node sets tailored specifically for causal inference, addressing additional size constraints, and refining algorithms with theoretical guarantees. 

\revb{Lastly, we note that the size constraint can impact the generated summary DAG, and users may need to adjust it to obtain a desirable summary. Future research will explore methods to recommend an optimal value for this parameter. For example, a heuristic stopping condition could be added to the algorithm, signaling it to halt if the next merge would result in a significant loss of information.}



\begin{acks}
We sincerely appreciate the support from DARPA ASKEM Award HR00112220042, the ARPA-H Biomedical Data Fabric project, and Liberty Mutual. We also want to thank the NSF Award IIS-2340124, NIH Grant U54HG012510, and the Israel Science Foundation (ISF) Grant No. 1442/24 for their partial funding of this work.
\end{acks}

\clearpage
\bibliographystyle{ACM-Reference-Format}
\bibliography{references}

  \appendix

\section{Semigraphoid-Axioms}
\label{app:semigraphoid}
The semi-graphoid axioms are the following:

\begin{enumerate}
    \item Triviality: $I(A;\emptyset|C){=}0$.
    \item Symmetry: $I(A;B|C){=}0 {\implies} I(A;C|B){=}0$.
    \item Decomposition: $I(A;BD|C){=}0 {\implies} I(A;B|D){=}0$.
    \item Contraction: $I(A;B|C){=}0$, $I(A;D|BC){=}0 {\implies} I(A;BD|C){=}0$.
    \item Weak Union: $I(A;BD|C){=}0 {\implies} I(A;B|CD){=}0, I(A;D|BC){=}0$.
\end{enumerate}
The semi-graphoid axioms can be summarized using the following identity, which follows from the \e{chain-rule} for mutual information~\cite{Yeung:2008:ITN:1457455}.
\begin{equation}
	\label{eq:chainRule}
	I(A;BD|C){=}0 \text{ if and only if } I(A;B|C){=}0 \text{ and }I(A;D|BC){=}0
\end{equation}





\section{Proofs}
\label{app:proof}
Next, we provide the missing proofs.

We begin with some basic definitions used in the proofs.

Let $\cg$ be a causal DAG.
and let $U,V {\in} \nodes(\cg)$ be two nodes. We say that $U$ is a \e{parent} of $V$, and $V$ a \e{child} of $U$ if $(U \rightarrow V) \in \edges(\cg)$. 
A \e{directed path} $t{=}(V_1,{\dots},V_n)$ is a sequence of nodes such that there is an edge $(V_i {\rightarrow} V_{i+1}){\in} \edges(\cg)$ for every $i{\in} \set{1,{\dots},n{-}1}$. We say that $V$ is a \e{descendant} of $U$, and $U$ an \e{ancestor} of $V$ if there is a directed path from $U$ to $V$.
We denote the child-nodes of $V$ in $\cg$ as $\child_{\cg}(V)$ ; the descendants of $V$ (we assume that $V{\in} \desc_{\cg}(V)$) as $\desc_{\cg}(V)$, and the nodes of $\cg$ that are not descendants of $V$ as $\nondesc_{\cg}(V)$. For a set of nodes $\mb{S}{\subseteq} \nodes(\cg)$, we let $\desc_{\cg}(\mb{S}){\eqdef} \bigcup_{U\in \mb{S}}\desc_{\cg}(U)$, and by $\nondesc_{\cg}(\mb{S}){\eqdef} \bigcap_{U\in \mb{S}}\nondesc_{\cg}(U)$. 


A \e{trail} $t{=}(V_1,{\dots},V_n)$ is a sequence of nodes such that there is an edge between $V_i$ and $V_{i{+}1}$ for every $i{\in} \set{1,{\dots},n{-}1}$. That is, $(V_i{\rightarrow} V_{i+1}){\in} \edges(\cg)$ or $(V_i {\leftarrow} V_{i+1}){\in} \edges(\cg)$ for every $i{\in} \set{1,{\dots},n{-}1}$. A node $V_i$ is said to be \e{head-to-head} with respect to $t$ if $(V_{i-1}{\rightarrow} V_i){\in} \edges(\cg)$ and $(V_{i}{\leftarrow} V_{i+1}){\in} \edges(\cg)$.
A trail $t=(V_1,{\dots},V_n)$ is \e{active} given $\mb{Z}{\subseteq} \mathcal{V}$ if (1) every $V_i$ that is a head-to-head node with respect to $t$ either belongs to $\mb{Z}$ or has a descendant in $\mb{Z}$, and (2) every $V_i$ that is not a head-to-head node w.r.t. $t$ does not belong to $\mb{Z}$. If a trail $t$ is not active given $\mb{Z}$, then it is \e{blocked} given $\mb{Z}$. 

\subsection{Proofs for Section \ref{sec:problem}}

\begin{proof}[Proof of Lemma \ref{lem:contractDAG}]
	Let $P$ be a directed path from $A$ to $B$, such that $|P|\geq 2$. Let $X$ be $A$'s successor in $P$, and $Y$ be $B$'s predecessor in $P$. By our assumption that $|P|\geq 2$, $X\notin \set{A,B}$, but $Y$ may be the same as $X$. Now, consider the graph $H$. By definition, $H$ contains a node $AB$, with an incoming edge from $Y$, and an outgoing edge to $X$. If $X=Y$, we immediately get the cycle $AB\rightarrow X \rightarrow AB$. Otherwise, we consider the subpath $P'$ (of $P$) from $X$ to $Y$ ($X\underset{P'}{\rightsquigarrow} Y$) in $G$. This results in the following cycle in $H$: $Y\rightarrow AB \rightarrow X\underset{P'}{\rightsquigarrow} Y$.
	
	Now, suppose that $H$ contains the cycle $C$. $C$ must contain the node $AB$. Otherwise, the cycle is included in $G$, which leads to a contradiction that $G$ is a DAG. Let $Y$ and $X$ be the incoming and outgoing vertices, respectively, to $AB$ in $C$. Then, there is a directed path $P$ from $X$ to $Y$ in $H$ that avoids $AB$. That is, every vertex and edge on the path $P$ belongs to $G$ as well. Hence, $P$ is a directed path from $X$ to $Y$ in $G$. Since $Y$ is incoming to $AB$, then $Y$ is incoming to either $A$ or $B$ in $G$. Assume, wlog, that $Y\rightarrow A \in E$. Since $X$ is an outgoing vertex from $AB$, then it is outgoing from from either $A$ or $B$ (or both). If $X$ is outgoing from $A$, then we get the following cycle in $G$: $Y\rightarrow A \rightarrow X \underset{P}{\rightsquigarrow}Y$. Since $G$ is a DAG, this brings us to a contradiction.
	Therefore, $X$ must be outgoing from $B$ and not $A$.  But this gives us the following directed path from $B$ to $A$: $$A\leftarrow Y \underset{P}{\leftsquigarrow} X \leftarrow B.$$ This completes the proof.
\end{proof}

Next, we show that the \probName\ problem is $NP$-hard via a reduction from the $k$-Max-Cut problem~\cite{DBLP:journals/tcs/GareyJS76}. Let $G$ be an undirected graph with weighted edges (i.e., $w: \nodes(G) \rightarrow \mathbb{R}_{\geq 0}$).
The $k$-Max-Cut problem consists of
partitioning $\nodes(G)$ into $k$ disjoint clusters so as to maximize the sum of weights
of the edges joining vertices in different clusters. It is well-known that $k$-Max-Cut is NP-hard even if $k=2$~\cite{DBLP:journals/tcs/GareyJS76}, and the weight of every edge is $1$. In other words, deciding whether there exists a $k$-clustering of $\nodes(G)$ to clusters $\set{V_1,\dots,V_k}$, where the sum of weights of edges between vertices in distinct clusters is at least a given threshold $\gamma$ is NP-hard.

Let $\mathcal{S}\eqdef \set{V_1,\dots,V_k}$ be a $k$-clustering of $G$. We define $M_G(\mathcal{S})$ to be the number of edges of between vertices in a common cluster:
\begin{equation}
\label{eq:M(S)}
M_G(\mathcal{S})\eqdef \sum_{i=1}^k \sum_{u,v\in V_i} \mathbb{1}[(u,v)\in \edges(G)].
\end{equation}
We define $\overline{M}_G(\mathcal{S})$ to be the number of non-edges between vertices in a common cluster:
\begin{equation}
\label{eq:overM(S)}
\overline{M}_G(\mathcal{S})\eqdef \sum_{i=1}^k \sum_{u,v\in V_i} \mathbb{1}[(u,v)\notin \edges(G)].
\end{equation}

Similarly, we define $B_G(\mathcal{S})$ to be the number of edges of between vertices in distinct clusters:
\begin{equation}
\label{eq:B(S)}
B_G(\mathcal{S})\eqdef \sum_{1\leq i<j \leq k} \sum_{\substack{u\in V_i, \\v\in V_j}}\mathbb{1}[(u,v)\in \edges(G)].
\end{equation}
Since every edge $(u,v)\in \edges(G)$ is either between vertices in the same cluster or vertices in distinct clusters, then:
\[
M_G(\mathcal{S})+B_G(\mathcal{S})=|\edges(G)|.
\]
In particular, $M_G(\mathcal{S})\leq \tau$ if and only if $B_G(\mathcal{S})\geq |\edges(G)|-\tau$, for every $\tau\in [0,|\edges(G)|]$.

\begin{lemma}
\label{lem:MInimizeM(S)}
   Let $G$ be an undirected graph, $k\geq 2$, and $\tau \in [0,|\edges(G)|]$. Deciding whether there exists a $k$-clustering $\mathcal{S}$ of $G$ such that $M_G(\mathcal{S})\leq \tau$ is NP-Complete.
\end{lemma}
\begin{proof}
    The problem is clearly in NP because given a $k$-clustering $\mathcal{S}$, computing $M_G(\mathcal{S})$ (see~\eqref{eq:M(S)}) can be done in polynomial time.
    
    We prove hardness by reduction from $k$-Max-Cut. Suppose there exists a $P$-Time algorithm that given an undirected graph $G$, and a threshold value $\tau\in [0,|\edges(G)|]$, returns a $k$-clustering $\mathcal{S}$ such that $M_G(\mathcal{S})\leq \tau$ if one exists, and null otherwise. We show that such an algorithm can be applied to solve $k$-Max-Cut in polynomial time.

    Let $G$, $k$, and $\gamma$ be an instance of the $k$-max-Cut problem where $G$ is an undirected graph, $k\leq |\nodes(G)|$, and $\gamma\in [0,|\edges(G)|]$. We define $\tau\eqdef |\edges(G)|-\gamma$, and execute the algorithm for finding a $k$-clustering $\mathcal{S}$ such that $M_G(\mathcal{S})\leq \tau$. Since $M_G(\mathcal{S})+B_G(\mathcal{S})=|\edges(G)|$, then:
    \begin{align*}
        B_G(\mathcal{S})&=|\edges(G)|-M_G(\mathcal{S})&& \Rightarrow_{M_G(\mathcal{S})\leq \tau}\\
        &\geq |\edges(G)|-\tau && \Rightarrow_{\tau \eqdef |\edges(G)|-\gamma}\\
        &=|\edges(G)|-(|\edges(G)|-\gamma) &&\\
        &=\gamma &&
    \end{align*}
   Hence, we can decide in $P$-Time whether $G$ has a $k$-cut whose cardinality is at least $\gamma$. 
\end{proof}
\begin{lemma}
    \label{lem:overlineMNPHard}
   Let $G$ be an undirected graph, $k\geq 2$, and $\tau \in [0,|\edges(G)|]$. Deciding whether there exists a $k$-clustering $\mathcal{S}$ of $G$ such that $\overline{M}_G(\mathcal{S})\leq \tau$ is NP-Complete.
\end{lemma}
\begin{proof}
The problem is clearly in NP because given a $k$-clustering $\mathcal{S}$, computing $\overline{M}_G(\mathcal{S})$, and verifying $\overline{M}_G(\mathcal{S})\leq \tau$ can be done in polynomial time (see~\eqref{eq:overM(S)}).

So, suppose that there exists a $P$-Time algorithm that given an undirected graph $G$, and a threshold value $\tau\in [0,|\edges(G)|]$, returns a $k$-clustering $\mathcal{S}$ such that $\overline{M}_G(\mathcal{S})\leq \tau$ if one exists, and null otherwise. We show that such an algorithm can be applied to decide whether there exists a $k$-clustering of $G$ where $M_G(\mathcal{S})\leq \alpha$. By Lemma~\ref{lem:MInimizeM(S)}, this problem is NP-Hard, and hence this will prove that minimizing $\overline{M}_G(\mathcal{S})$ is NP-hard as well.

So let $G$, $k$, and $\alpha$ be the input to the problem for deciding whether there exists a $k$-clustering $\mathcal{S}$ such that $M_G(\mathcal{S})\leq \alpha$. Let $\overline{G}$ denote the complement graph of $G$. That is $\nodes(\overline{G})=\nodes(G)$ and $(u,v)\in \edges(G)$ if and only if $(u,v)\notin \edges(\overline{G})$. In particular, $\overline{M}_{\overline{G}}(\mathcal{S})=\alpha$ if and only if $M_G(\mathcal{S})=\alpha$. Therefore, if we can, in $P$-Time find a clustering that minimizes $\overline{M}_{\overline{G}}(\mathcal{S})$, then we have also found a clustering that minimizes $M_G(\mathcal{S})$. This completes the proof.
\end{proof}

We now show that the \probName\ problem is $NP$-hard. Specifically, we show that finding a summary DAG $(\qg,f)$ whose canonical DAG $\cg$ minimizes $|\edges(\cg_{\qg})|-|\edges(G)|$; that is, the canonical DAG $\cg_{\qg}$ results in the smallest number of added edges, is NP-Hard.
\begin{theorem}
    The \probName\ problem which minimizes the number of added edges to the canonical DAG is NP-Hard.
    \label{theorem:problem-is-np-hard-appendix}
\end{theorem}
\begin{proof}
    Given a DAG $D$ and a $k$-clustering $\mathcal{S}$ of $D$, it is straightforward to verify that $\overline{M}_D(\mathcal{S})\leq \tau$, and hence this problem is in NP.

    Let $G$ be an undirected graph, $k>0$, and $\gamma>0$ a threshold. Let $\nodes(G)=\set{v_1,\dots,v_n}$ denote a complete order of $\nodes(G)$. Define $D$ to be the directed graph where $\nodes(D)=\nodes(G)$ and $(v_i\rightarrow v_j)\in \edges(D)$ if and only if $(v_u,v_j)\in \edges(G)$ and $i<j$. Let $\mathcal{S}=\set{V_1,\dots,V_k}$ be a $k$-clustering of $\nodes(D)=\nodes(G)$ such that $\overline{M}_D(\mathcal{S})\leq \gamma$. By the definition of $\edges(D)$, we immediately get that $\overline{M}_G(\mathcal{S})\leq \gamma$. 
    By Lemma~\ref{lem:overlineMNPHard}, the \probName\ problem is $NP$-hard. 
\end{proof}

\eat{
Next, we show that the \probName\ problem is $NP$-hard via a reduction from the $k$-Min-Cut problem~\cite{hartmanis1982computers}.
Let $G$ be an undirected graph with weighted edges (i.e., $w: \nodes(G) \rightarrow \mathbb{R}_{\geq 0}$).
The $k$-Min-Cut problem consists of
partitioning $\nodes(G)$ into $k$ disjoint clusters so as to minimize the sum of weights
of the edges joining vertices in different clusters. It is well-known that $k$-Min-Cut is NP-hard even if $k=2$~\cite{DBLP:journals/tcs/GareyJS76}. In other words, deciding whether there exists a $k$-clustering of $\nodes(G)$ to clusters $\set{V_1,\dots,V_k}$, where the sum of weights of edges between vertices in distinct clusters is at most a given threshold $\gamma$ is NP-hard.

\begin{proof}[Proof of Theorem \ref{theorem:problem-is-np-hard}]
    In this proof we make the assumption that $InterSim(\mb S)=\sum_{u,v\in \mb S}InterSim(\set{u,v})$ where $\mb S\subseteq \nodes(G)$.
   Given a DAG $G$, a similarity threshold $\tau$, and its summary graph $\qg$, it is easy to verify that $\sum_{\mb S\in \nodes(\qg)}InterSim(\mb S)\geq \tau$. Therefore, the \probName~problem is in NP.

   We prove hardness by reduction from $k$-Min-Cut. Let $G$ be an undirected weighted graph, and let $\gamma>0$ be the threshold for the $k$-Min-Cut problem. For every pair of vertices $u,v \in \nodes(G)$, define $InterSim(\set{u,v})\eqdef w(u,v)$; define $\tau \eqdef \sum_{u,v\in \nodes(G)}w(u,v)- \gamma$. Let $G'$ be the DAG that results from $G$ by orienting its edges such that no directed cycles are generated (this is simple to do by fixing a complete order over the vertices, and orienting the edges accordingly). It is easy to see that there is a $k$-clustering $\set{V_1,\dots,V_k}$ of $G$ where the sum of weights of edges between vertices in distinct clusters is at least $\gamma$ if and only if $\sum_{i=1}^k InterSim(V_i)\geq \tau$ in $G'$. 
\end{proof}

}
\subsection{Proofs for Section \ref{sec:csep}}

Next, we prove some simple lemmas that will be useful later on. We denote by $\mathcal{H}_{UV}$ the summary DAG where nodes $U$ and $V$ are contracted. 
\begin{lemma}
\label{lemma:lemma1}
	The following holds:
	\begin{align}
		&\pi_{\mathcal{H}_{UV}}(\mb{X_{UV}})=\pi_{\cg_{UV}}(U)&& \pi_{\cg_{UV}}(V)=\pi_{\cg_{UV}}(U){\cup} \set{U}  \label{eq:lem:contractionAddedges1:1}\\
		&\child_{\mathcal{H}_{UV}}(\mb{X_{UV}})=\child_{\cg_{UV}}(V) && \child_{\cg_{UV}}(U)=\child_{\cg_{UV}}(V){\cup} \set{V} \label{eq:lem:contractionAddedges1:2}\\
		&\nondesc_{\mathcal{H}_{UV}}(\mb{X_{UV}}){=}\nondesc_{\cg_{UV}}(U) && \nondesc_{\cg_{UV}}(V){=}\nondesc_{\cg_{UV}}(U){\cup} \set{U} \label{eq:lem:contractionAddedges1:4}
	\end{align}
\end{lemma}

\begin{lemma}
	\label{lem:contractionAddedges2}
	Let $T\in \mathcal{V}$ such that $T\notin \set{U,V}\cup \child_{\cg}(U)\cup \child_{\cg}(V)$. Then it holds that:
	\begin{align}
		& \pi_{\cg_{UV}}(T)=\pi_{\mathcal{H}_{UV}}(T) \text{ and } \label{eq:sameParents}\\
		&\nondesc_{\cg_{UV}}(T){\setminus}\set{UV}=\nondesc_{\qg}(T){\setminus}\set{\mb{X_{UV}}} \text{ and }\label{eq:tNotChild} \\
		& \set{U,V}\subseteq \nondesc_{\cg_{UV}}(T)  \text{ if and only if }  \mb{X_{UV}}\in \nondesc_{\qg}(T) \label{eq:uvXuvparent} 
	\end{align}
	Now, let  $T\in \mathcal{V}$ such that $T\notin \set{U,V}\cup \pi_{\cg}(U)\cup \pi_{\cg}(V)$. Then:
	\begin{align}
		&\child_{\cg_{UV}}(T)=\child_{\mathcal{H}_{UV}}(T) \text{ and }  \label{eq:sameChildren}\\
		& \desc_{\cg_{UV}}(T){\setminus}\set{U,V}=\desc_{\mathcal{H}_{UV}}(T){\setminus}\set{\mb{X_{UV}}} \label{eq:tNotParent}\\
		& \set{U,V}\subseteq \desc_{\cg_{UV}}(T)  \text{ if and only if }  \mb{X_{UV}}\in \desc_{\mathcal{H}_{UV}}(T)
	\end{align}
\end{lemma}

\begin{proof}[proof of Lemma \ref{lemma:lemma1}]
By the definition of $G'$, it holds that $\pi_{G'}(u)=\pi_G(u)\cup \pi_G(v)$. Since $(u,v)\in \edges(G')$, then $\pi_{G'}(v)=\pi_{G'}(u) \cup \set{u}$.
Similarly, $\child_{G'}(v)=\child_G(u)\cup \child_G(v)$, and since $(u,v)\in \edges(G')$, then $\child_{G'}(u)=\child_{G'}(v) \cup \set{v}$.
By definition of edge contraction, it holds that $\pi_H(X_{uv})=\pi_G(u)\cup \pi_G(v)=\pi_{G'}(u)$, proving~\eqref{eq:lem:contractionAddedges1:1}.
Also, by definition of edge contraction, it holds that $\child_H(X_{uv})=\child_G(u)\cup \child_G(v)=\child_{G'}(v)$, proving~\eqref{eq:lem:contractionAddedges1:2}.

We now prove~\eqref{eq:lem:contractionAddedges1:4}. Let $t\in \nondesc_H(X_{uv})$. If $t \notin \nondesc_{G'}(u)$, then $t\in \desc_{G'}(u){\setminus}\set{v}$. This means that there is a directed path $P$ from $u$ to $t$ in $G'$. Let $s$ be the first vertex on this path (after $u$). Since $s\in \child_{G'}(u){\setminus}\set{v}$, then by the definition of $G'$, $s\in \child_G(u)\cup \child_G(v)$. By the definition of edge contraction, $s\in \child_H(X_{uv})$. Since $s\notin uv \cup \pi_G(u)\cup \pi_G(v)$, then every directed path starting at $s$ in $G$ remains a directed path in $H$. But this means that there is a directed path from $X_{uv}$ to $t$ (via $s$); contradicting the assumption that $t\in \nondesc_H(X_{uv})$. 
Now, let $t\in \nondesc_{G'}(u)$. If $t\notin \nondesc_H(X_{uv})$, then $t\in \desc_H(X_{uv}){\setminus}\set{X_{uv}}$.  This means that there is a directed path $P$ from $X_{uv}$ to $t$ in $H$. Let $s$ be the first vertex on this path (after $X_{uv}$). Since $s\in \child_{H}(X_{uv})$, then by the definition of $H$, $s\in \child_G(u)\cup \child_G(v)$. But then, by the definition of $G'$, it holds that $s\in \child_{G'}(u)$. Since no edges are removed by the transition from $G$ to $G'$, there is a directed path from $u$ to (via $s$) in $G'$; contradicting the assumption that $t\in \nondesc_{G'}(u)$. 
\end{proof}

\begin{proof}[Proof of Lemma \ref{lem:contractionAddedges2}]
By the definition of contraction, the only vertices in $G$ whose parent-set can potentially change following the contraction of $u$ and $v$ belong to the set $uv \cup \child_G(u) \cup \child_G(v)$.
By the definition of $\edges(G')$, the only vertices in $G$ whose parent-set can potentially change belong to the set $uv \cup \child_G(u) \cup \child_G(v)$.
Therefore, if $t\notin uv \cup \child_G(u) \cup \child_G(v)$, then $\pi_{G'}(t)=\pi_{H}(t)=\pi_G(t)$. This proves~\eqref{eq:sameParents}.

By the definition of contraction, the only vertices in $G$ whose child-set can potentially change following the contraction of $u$ and $v$ belong to the set $uv \cup \pi_G(u) \cup \pi_G(v)$.
By the definition of $\edges(G')$, the only vertices in $G$ whose child-set can potentially change belong to the set $uv \cup  \pi_G(u) \cup \pi_G(v)$.
Therefore, if $t\notin uv \cup \pi_G(u) \cup \pi_G(v)$, then $\child_{G'}(t)=\child_{H}(t)=\child_G(t)$. This proves~\eqref{eq:sameChildren}.

We now prove~\eqref{eq:tNotChild};
Let $s\in \nondesc_{H}(t){\setminus}\set{X_{uv}}$. If $s\notin \nondesc_{G'}(t)$, then $s\in \desc_{G'}(t)$. That is, there is a directed path $P$ from $t$ to $s$ in $G'$. Let us assume wlog that $P$ is the shortest directed path from $t$ to $s$ in $G'$. By this assumption, exactly one of the following holds: (1) $u,v \notin \nodes(P)$ (2) $u\in \nodes(P)$, $v\notin nodes(P)$ (3) $v\in \nodes(P)$, $u\notin nodes(P)$, or (4) $(u,v)\in \edges(P)$.
In the first case, every edge of $P$ is also an edge of $\edges(G)$, that does not enter or exit $\set{u,v}$. Therefore, $P$ is a directed path in $H$, a contradiction. In case (2), since $\pi_{G'}(u)=\pi_H(X_{uv})$ (see~\eqref{eq:lem:contractionAddedges1:1}), and $\child_{G'}(u)=\child_H(X_{uv}){\setminus}\set{v}$ (see~\eqref{eq:lem:contractionAddedges1:2}), then the path with nodes $X_{uv}\cup(\nodes(P)\setminus\set{u}) $, is a directed path in $H$ from $s$ to $t$; a contradiction.
In case (3), since $\child_{G'}(v)=\child_H(X_{uv})$ (see~\eqref{eq:lem:contractionAddedges1:2}), and $\pi_{G'}(v)=\pi_H(X_{uv}){\setminus}\set{u}$ (see~\eqref{eq:lem:contractionAddedges1:1}), then the path with nodes $X_{uv}\cup(\nodes(P)\setminus\set{v}) $, is a directed path in $H$ from $s$ to $t$; a contradiction. Finally, if $(u,v)\in \edges(P)$, then since $\pi_{G'}(u)=\pi_H(X_{uv})$ and $\child_{G'}(v)=\child_H(X_{uv})$, then the path with nodes $X_{uv}\cup (\nodes(P){\setminus}uv)$, is a directed path from $s$ to $t$ in $H$; a contradiction.
For the other direction, let $s\in \nondesc_{G'}(t){\setminus} uv$. If $s\notin \nondesc_H(t)$, then there is a directed path $P$ from $t$ to $s$ in $H$. If $X_{uv}\notin \nodes(P)$, then $\edges(P)\subseteq \edges(G)\subseteq \edges(G')$, and hence $P$ is a directed path from $t$ to $s$ in $G'$. Otherwise, if $X_{uv}\in \nodes(P)$, then since $\pi_H(X_{uv})=\pi_{G'}(u)$, $\child_H(X_{uv})=\child_{G'}(v)$, and $(u,v)\in \edges(G')$, then replacing $X_{uv}$, with the edge $(u,v)$ results in a directed $t,s$-path in $G'$; a contradiction.
\eat{
If $u,v \notin \nodes(P)$, then $P$ is a path in $H$, and hence $s\in \desc_H(t)$, which leads to a contradiction. If $u\in \nodes(P)$, then since $(u,v)\in \edges(G')$, and $\child_{G'}(u)=\child_{G'}(v)\cup \set{v}$, then $P$ can be made to go through both $u$ and $v$, and we can assume that $u,v \in \nodes(P)$. Likewise,  if $v\in \nodes(P)$, then since $(u,v)\in \edges(G')$, and $\pi_{G'}(u)\cup\set{u}=\pi_{G'}(v)$, then $P$ can be made to go through both $u$ and $v$, and we can assume that $u,v \in \nodes(P)$. So, we assume that $uv \subseteq \nodes(P)$. Since $H$ is a DAG, then by Lemma~\ref{lem:contractDAG}, there is no directed path between $u$ and $v$ whose length is greater than 2. But this means that $(u,v)\in \edges(P)$. But this means that $\nodes(P){\setminus}uv \cup X_{uv}$ form a directed path in $H$, contradicting the assumption that $s\in \nondesc_{H}(t)$. 
}
\end{proof}

\begin{proof} [Proof of Theorem \ref{thm:contractionEdgeAddition}]
	We first prove that $\Sigma_{\text{RB}}(G')\implies \Sigma_{\text{RB}}(H)$. We divide to cases.
	Let $(X_i; B_i|\pi_H(X_i))\in \Sigma_{\text{RB}}(H)$, where $X_{uv}\notin B_i\cup \pi_H(X_i)\cup \set{X_i}$.
	In particular, $X_i\in \nodes(G)$, and $X_i\notin \set{u,v}\cup \child_H(X_{uv})=\set{u,v}\cup \child_G(u)\cup \child_G(v)$.
	By~\eqref{eq:tNotChild}, we have that $\pi_{G'}(X_i)=\pi_H(X_i)$, and that $\nondesc_{G'}(X_i){\setminus}uv=\nondesc_{H}(X_i){\setminus}X_{uv}$.
	Therefore, we have that $(X_i;\nondesc_{H}(X_i){\setminus}\set{X_{uv}}|\pi_{G'}(X_i))_{G'}$. Since
	 $B_i\subseteq \nondesc_{H}(X_i){\setminus}\set{X_{uv}}$, then, by decomposition, we have that $\Sigma_{\text{RB}}(G')\implies (X_i; B_i| \pi_H(X_i))$.
	 
	Now, let $(X_i; X_{uv}B_i|\pi_H(X_i))\in \Sigma_{\text{RB}}(H)$. In this case as well $X_i\in \nodes(G)$, and $X_i\notin \set{u,v}\cup \child_H(X_{uv})=\set{u,v}\cup \child_G(u)\cup \child_G(v)$. By~\eqref{eq:tNotChild}, we have that $\pi_{G'}(X_i)=\pi_H(X_i)$, and that $\nondesc_{G'}(X_i){\setminus}uv=\nondesc_{H}(X_i){\setminus}X_{uv}$. 
	By , we have that $X_{uv}\in \nondesc_{H}(X_i)$ iff $uv \subseteq \nondesc_{G'}(X_i)$. Therefore, $B_iX_{uv}\subseteq \nondesc_{H}(X_i)$ iff $B_iuv\subseteq \nondesc_{G'}(X_i)$. This means that $\Sigma_{\text{RB}}(G')\implies (X_i;\nondesc_{G'}(X_i)|\pi_{G'}(X_i))$. By decomposition, we have that $\Sigma_{\text{RB}}(G')\implies (X_i;B_iuv|\pi_{H}(X_i))$ as required.

	Now, suppose that $X_{uv}\in \pi_H(X_i)$, or that $X_i \in \child_H(X_{uv})$. Since $X_i\in \nodes(G){\setminus}\set{u,v}$, then by~\eqref{eq:lem:contractionAddedges1:2}, we have that $X_i\in \child_{G'}(v){\setminus}\set{v}$. Therefore, $\pi_{G'}(X_i)=\pi_H(X_i){\setminus}\set{X_{uv}}\cup \set{u,v}$. 
	By~\eqref{eq:tNotChild}, we have that:
	\begin{align*}
		\nondesc_{G'}(X_i){\setminus}\pi_{G'}(X_i)&=\nondesc_{G'}(X_i){\setminus}(\pi_H(X_i){\setminus}\set{X_{uv}}\cup uv)\\
		&=(\nondesc_{G'}(X_i){\setminus}uv){\setminus}(\pi_H(X_i){\setminus}\set{X_{uv}})\\
		&\underbrace{=}_{\eqref{eq:tNotChild}}(\nondesc_{H}(X_i){\setminus}\set{X_{uv}}){\setminus}(\pi_H(X_i){\setminus}\set{X_{uv}}) \\
		&=\nondesc_{H}(X_i){\setminus}\pi_H(X_i)
	\end{align*} 
Therefore, $\Sigma_{\text{RB}}(G')\implies (X_i; \nondesc_H(X_i){\setminus}\pi_H(X_i) \mid \pi_H(X_i){\setminus}\set{X_{uv}}\cup \set{u,v})$. 
Finally, we consider the case where $X_i=X_{uv}$. By construction of $G'$, and by~\eqref{eq:lem:contractionAddedges1:4}, it holds that:
\begin{align}
	\Sigma_{\text{RB}}(G')&\implies (u; \nondesc_{G'}(u){\setminus}\pi_{G'}(u) \mid \pi_{G'}(u)) \label{eq:GImpliesHContractionAxiom:1}\\
	\Sigma_{\text{RB}}(G')&\implies (v; \nondesc_{G'}(u){\setminus}\pi_{G'}(u) \mid \pi_{G'}(u)\cup \set{u}) \label{eq:GImpliesHContractionAxiom:2}
\end{align}
By applying the contraction axiom on~\eqref{eq:GImpliesHContractionAxiom:1} and~\eqref{eq:GImpliesHContractionAxiom:2}, we get that 
$$\Sigma_{\text{RB}}(G') \implies (uv; \nondesc_{G'}(u){\setminus}\pi_{G'}(u) \mid \pi_{G'}(u)).$$
Using the fact that hat $\pi_H(X_{uv})=\pi_{G'}(u)$ (see~\eqref{eq:lem:contractionAddedges1:1}), and that $\nondesc_H(X_{uv})=\nondesc_{G'}(u)$ (see~\eqref{eq:lem:contractionAddedges1:4}), we get that 
$$\Sigma_{\text{RB}}(G') \implies (uv; \nondesc_{H}(X_{uv}){\setminus}\pi_{H}(X_{uv}) \mid\pi_{H}(X_{uv})).$$
Since $B_i \subseteq \nondesc_{H}(X_{uv}){\setminus}(\pi_{H}(X_{uv})\cup \set{X_{uv}})$, this proves the claim.

Now, for the other direction. Let $(X_i;B_i \mid \pi_{G'}(X_i))\in \Sigma_{\text{RB}}(G')$. If $u,v \notin X_i\cup B_i \cup \pi_{G'}(X_i))$, then $X_i\notin \set{u,v}\cup \child_G(u) \cup \child_G(v)$. By~\eqref{eq:tNotChild}, it holds that $\pi_H(X_i)=\pi_{G'}(X_i)$, and that $\nondesc_{G'}(X_i){\setminus}uv=\nondesc_H(X_i){\setminus}X_{uv}$. Since $B_i\subseteq \nondesc_{G'}(X_i){\setminus}uv=\nondesc_H(X_i){\setminus}X_{uv}$, then  $\Sigma_{\text{RB}}(H)\implies (X_i;B_i \mid \pi_{G'}(X_i))$.

If $uv \subseteq B_i$, then $u,v\notin \pi_{G'}(X_i)$, then $X_i\notin uv\cup \child_G(u)\cup \child_G(v)$. By~\eqref{eq:tNotChild}, we have that $\pi_H(X_i)=\pi_{G'}(X_i)$, and that $\nondesc_{H}(X_i){\setminus}X_{uv}=\nondesc_{G'}(X_i){\setminus}uv$. Therefore, $B_i{\setminus}uv\subseteq \nondesc_{H}(X_i)$, and by~\eqref{eq:uvXuvparent}, if $uv \subseteq B_i \subseteq \nondesc_{G'}(X_i)$, then $X_{uv}\in \nondesc_{H}(X_i)$.
Therefore, $\Sigma_{\text{RB}}(H)\implies (X_i;B_i{\setminus}uv \cup X_{uv} \mid \pi_{G'}(X_i))$, and since $X_{uv}=uv$, then $\Sigma_{\text{RB}}(H)\implies (X_i;B_i \mid \pi_{G'}(X_i))$. 

Since $(u,v)\in \edges(G')$, we are left with two other cases. 
First, that $(u;B_u \mid \pi_{G'}(u))$, and second $(v; B_v \mid \pi_{G'}(v))$. 
By $d$-separation in $H$, the following holds:
\begin{align}
\Sigma_{\text{RB}}(H) & \implies (X_{uv}; \nondesc_H(X_{uv}){\setminus}\pi_H(X_{uv}) \mid \pi_H(X_{uv})) \nonumber \\
& \underbrace{\implies}_{\eqref{eq:lem:contractionAddedges1:1},~\eqref{eq:lem:contractionAddedges1:4}}(X_{uv}; \nondesc_{G'}(u){\setminus}\pi_{G'}(u) \mid \pi_{G'}(u)) \nonumber \\
& \implies (uv ; \nondesc_{G'}(u){\setminus}\pi_{G'}(u) \mid \pi_{G'}(u)) \label{eq:HImpliesG1}
\end{align}

By~\eqref{eq:lem:contractionAddedges1:1}, it holds that $\pi_{G'}(v)= \pi_{G'}(u)\cup \set{u}$. 
By~\eqref{eq:lem:contractionAddedges1:4}, it holds that 
\begin{align*}
B_v \subseteq \nondesc_{G'}(v){\setminus}\pi_{G'}(v)&=(\nondesc_{G'}(u)\cup \set{u}){\setminus}(\pi_{G'}(u)\cup \set{u})\\
&=\nondesc_{G'}(u){\setminus}\pi_{G'}(u)
\end{align*}
Therefore, $B_v \cup B_u \subseteq \nondesc_{G'}(u){\setminus}\pi_{G'}(u)$.
In other words, by~\eqref{eq:HImpliesG1}, we have that;
\begin{align*}
&\Sigma_{\text{RB}}(H) \implies (uv ; B_u \cup B_v \mid \pi_{G'}(u))  \text{ if and only if } \\
&\Sigma_{\text{RB}}(H) \implies (u ; B_u \cup B_v \mid \pi_{G'}(u)) ,  (v ; B_u \cup B_v \mid \pi_{G'}(u)\cup \set{u}) 
\end{align*}
Since  $\pi_{G'}(v)= \pi_{G'}(u)\cup \set{u}$, then
overall, we have that $\Sigma_{\text{RB}}(H) \implies (u;B_u \mid \pi_{G'}(u))$, and $\Sigma_{\text{RB}}(H) \implies (v;B_v \mid \pi_{G'}(v))$. This completes the proof.
\end{proof}

To prove Theorem \ref{cor:dsep}, we first show the following lemma that establishes the connection between $d$-separation on the \groundedDAG\ and the original causal DAG. 
\begin{lemma}
\label{lem:dSepsupergraph}
Let $\cg$ and $\cg'$ be causal DAGs defined over the same set of nodes, i.e., $\nodes(\cg)=\nodes(\cg')$, where $\cg'$ is a supergraph of $\cg$ ($\edges(\cg')\supseteq \edges(\cg)$). Then, for any three disjoint subsets $\mb{X},\mb{Y},\mb{Z}\subseteq \nodes(\cg)$, it holds that:
$(\mb{X} \indep_d \mb{Y} |\mb{Z})_{\cg'}\implies (\mb{X} \indep_d \mb{Y} |\mb{Z})_{\cg}$
\end{lemma}

\begin{proof}[Proof of Lemma \ref{lem:dSepsupergraph}]
	Suppose that $X$ and $Y$ are $d$-separated by $Z$ in $G'$ (i.e., $(X;Y|Z)_{G'}$). If $X$ and $Y$ are $d$-connected by $Z$ in $G$, then let $P$ denote the unblocked path between $X$ and $Y$, relative to $Z$. Since $\edges(G')\supseteq \edges(G)$, then clearly $P$ is a path in $G'$ as well. Consider any triple $(x,w,y)$ on this path. If this triple has one of the forms $$\set{x\rightarrow w\rightarrow y, x\leftarrow w\leftarrow y, x \leftrightarrow w \rightarrow y, x\leftarrow w \leftrightarrow y},$$ then since $P$ is unblocked in $G$, relative to $Z$, then $w\notin Z$. Since $w\notin Z$, then the subpath $(x,w,y)$ is also unblocked in $G'$.
	If the triple has one of the forms:
	$$\set{x\leftrightarrow w\leftarrow y, x\rightarrow w\leftrightarrow y, x\rightarrow w \leftarrow y},$$
	then since $P$ is unblocked in $G$, relative to $Z$, then $\desc_G(w) \cap Z\neq \emptyset$. Since $\edges(G')\supseteq \edges(G)$, then $\desc_G(w)\subseteq \desc_{G'}(w)$. Therefore, $\desc_{G'}(w) \cap Z\neq \emptyset$. Consequently, we have, again, that the subpath $(x,w,y)$ is unblocked in $G'$.
	Overall, we get that every triple $(x,w,y)$ on the path $P$ is unblocked in $G'$, relative to $Z$, and hence $X$ and $Y$ are $d$-connected in $G'$, a contradiction.
	
	By definition, $G'$ is compatible with $G'$. Therefore, if $X$ and $Y$ are $d$-connected by $Z$ in $G'$, then by the completeness of $d$-separation, there exists a probability distribution that factorizes according to $G'$ in which the CI $(X;Y|Z)$ does not hold. This proves completeness.
\end{proof}

\begin{proof}[Proof of Theorem \ref{cor:dsep}]
	Let $\groundedQG$ denote the \groundedDAG\ corresponding to $\qg$. By Theorem~\ref{thm:contractionEdgeAddition}, $\Sigma_{\text{RB}}(\qg)\equiv \Sigma_{\text{RB}}(\groundedQG)$. Therefore, $(\mb{X}\indep_d \mb{Y}|\mb{Z})_{\qg} \Longleftrightarrow (f^{-1}(\mb{X})\indep_d f^{-1}(\mb{Y})|f^{-1}(\mb{Z}))_{\groundedQG}$
 Since $\edges(\cg)\subseteq \edges(\groundedQG)$, the claim immediately follows from Lemma~\ref{lem:dSepsupergraph}.
\end{proof}

\subsection{Proofs for Section \ref{sec:identifiability}}

We next show a smile lemma that will be useful for proving the soundness and completeness of do-calculus in summary graphs.


\begin{lemma}
	\label{lem:mutilatedGraphInclusion}
Let $G$ be ADMG, and let $G'$ be an ADMG where $\nodes(G')=\nodes(G)$, and $\edges(G')\supseteq \edges(G)$. Let $A,B,C\subseteq \nodes(G)$ be disjoint sets of variables, and let $X,Z \subseteq \nodes(G)$. Then:
\begin{align}
	(A;B|C)_{G'_{\overline{X}\underline{Z}}} \implies (A;B|C)_{G_{\overline{X}\underline{Z}}} 
\end{align}
\end{lemma}

\begin{corollary}
	\label{corr:mutilatedSummaryDAG}
	Let $G$ be ADMG, and let $(H,f)$ be a summary-DAG for $G$. 
    Let $A,B,C\subseteq \nodes(H)$ be disjoint sets of nodes, and let $X,Z \subseteq \nodes(H)$. Then:
	\begin{align}
		(A;B|C)_{H_{\overline{X}\underline{Z}}} \implies ({\bf A};{\bf B}|{\bf C})_{G_{\overline{{\bX}}\underline{{\bZ}}}} 
	\end{align}
where for $U \subseteq \nodes(H)$, we denote ${\bf U}\eqdef f(U)$. 
\end{corollary}

\begin{theorem}[Soundness of Do-Calculus in supergraphs]
\label{thr:do-calculus-supergraphs}
Let $\cg$ be a causal DAG encoding an interventional distribution $P(\cdot \mid do(\cdot))$. Let $\cg'$ be a causal DAG where $\nodes(\cg) = \nodes(\cg')$ and $\edges(\cg)\subseteq \edges(\cg')$. For any disjoint subsets $\mb{X,Y,Z,W} \subseteq \nodes(\cg)$, the following three rules hold:%

{\footnotesize
\begin{align*}
	{\bf R}_1: && (\mb{Y} \indep \mb{Z}|\mb{X,W})_{\cg'_{\overline{\mb{X}}}} &\implies P(\mb{Y}\mid do(\mb{X}),\mb{Z,W})=P(\mb{Y}\mid do(\mb{X}),\mb{W}) \\
	{\bf R}_2: && (\mb{Y}\indep \mb{Z}|\mb{X,W})_{\cg'_{\overline{\mb{X}}\underline{\mb{Z}}}} &\implies P(\mb{Y}\mid do(\mb{X}),do(\mb{Z}),\mb{W})=P(\mb{Y}\mid do(\mb{X}),\mb{Z}, \mb{W}) \\
	{\bf R}_3: && (\mb{Y} \indep \mb{Z}|\mb{X},\mb{W})_{\cg'_{\overline{\mb{X}}\overline{\mb{Z}(\mb{W})}}} &\implies P(\mb{Y}\mid do(\mb{X}),do(\mb{Z}),\mb{W})=P(\mb{Y}\mid do(\mb{X}), \mb{W}) 
\end{align*}}%
where $\mb{Z}(\mb{W})$ is the set of nodes in $\mb{Z}$ that are not ancestors of any node in $\mb{W}$. That is, $\mb{Z}(\mb{W})=\mb{Z}{\setminus}\ancs_{\cg'}(\mb{W})$ where $\ancs_{\cg'}(\mb{W})\eqdef \bigcup_{W\in \mb{W}}\ancs_{\cg'}(W)$.
\end{theorem}

\begin{theorem}[Soundness of Do-Calculus in supergraphs]
\label{thr:do-calculus-supergraphs}
Let $G$ be a causal BN (CBN) encoding an interventional distributions $P(\cdot \mid do(\cdot))$. Let $G'$ be an ADMG where $\edges(G)\subseteq \edges(G')$. For any disjoint subsets $X,Y,Z,W \subseteq \nodes(G)$, the following three rules hold:%

{\footnotesize
\begin{align*}
	{\bf R}_1: && (Y;Z|X,W)_{G'_{\overline{X}}} &\implies P(Y\mid do(X),Z,W)=P(Y\mid do(X),W) \\
	{\bf R}_2: && (Y;Z|X,W)_{G'_{\overline{X}\underline{Z}}} &\implies P(Y\mid do(X),do(Z),W)=P(Y\mid do(X),Z, W) \\
	{\bf R}_3: && (Y;Z|X,W)_{G'_{\overline{X}\overline{Z(W)}}} &\implies P(Y\mid do(X),do(Z),W)=P(Y\mid do(X), W) 
\end{align*}}%
where $Z(W)$ is the set of vertices in $Z$ that are not ancestors of any vertex in $W$. That is, $Z(W)=Z{\setminus}\ancs_{G'}(W)$ where $\ancs_{G'}(W)\eqdef \bigcup_{w\in W}\ancs_{G'}(w)$.
\end{theorem}

\begin{proof}[Proof of Lemman \ref{lem:mutilatedGraphInclusion}]
	We show that $\edges(G_{\overline{X}\underline{Z}})\subseteq \edges(G'_{\overline{X}\underline{Z}})$, and the claim then follows from Theorem~\ref{thm:dSep}.
	Let $(u,v)\in \edges(G_{\overline{X}\underline{Z}})\subseteq \edges(G) \subseteq \edges(G')$. By definition, $u\notin Z$ and $v \notin X$. But this means that $(u,v)\in \edges(G'_{\overline{X}\underline{Z}})$, which completes the proof.
\end{proof}

\begin{proof}[Proof of Corollary \ref{corr:mutilatedSummaryDAG}]
Let $G_{H_{\overline{X}\underline{Z}}}$ denote the grounded DAG corresponding to $H_{\overline{X}\underline{Z}}$. By Theorem~\ref{thm:contractionEdgeAddition}, $\Sigma_{\text{RB}}(H_{\overline{X}\underline{Z}})\equiv \Sigma_{\text{RB}}(G_{H_{\overline{\bX}\underline{\bZ}})}$, and hence $(A;B|C)_{H_{\overline{X}\underline{Z}}}$ if and only if $({\bf A};{\bf B}|{\bf C})_{G_{H_{{\overline{{\bf X}}\underline{{\bf Z}}}}}}$. Since $\edges(G_H)\supseteq \edges(G)$, then by Lemma~\ref{lem:mutilatedGraphInclusion}, it holds that if $({\bf A};{\bf B}|{\bf C})_{G_{H_{{\overline{{\bf X}}\underline{{\bf Z}}}}}}$, then $({\bf A};{\bf B}|{\bf C})_{G_{{\overline{{\bf X}}\underline{{\bf Z}}}}}$. Overall, we have that:
\begin{align}
	(A;B|C)_{H_{\overline{X}\underline{Z}}} \Leftrightarrow ({\bf A};{\bf B}|{\bf C})_{G_{H_{{\overline{{\bX}}\underline{{\bZ}}}}}} \implies ({\bf A};{\bf B}|{\bf C})_{G_{\overline{{\bX}}\underline{{\bZ}}}} 
\end{align}
which proves the claim.
\end{proof}

\begin{proof}[Proof of Theorem \ref{thr:do-calculus-supergraphs}]
	If $(Y;Z|X,W)_{G'_{\overline{X}}}$, then by Lemma~\ref{lem:mutilatedGraphInclusion}, it holds that $(Y;Z|X,W)_{G_{\overline{X}}}$. By the soundness of do-calculus for causal BNs, we get that $P(\bY \mid do(\bX),\bZ,\bW)=P(\bY\mid do(\bX),\bW) $. 
	If $ (Y;Z|X,W)_{G'_{\overline{X}\underline{Z}}}$, then by Lemma~\ref{lem:mutilatedGraphInclusion}, it holds that $ (Y;Z|X,W)_{G_{\overline{X}\underline{Z}}} $. By the soundness of do-calculus for causal BNs, we get that $P(Y\mid do(X),do(Z),W)=P(Y\mid do(X),Z, W) $. 
	Finally, if $(Y;Z|X,W)_{G'_{\overline{X}\overline{Z(W)}}}$, then by Lemma~\ref{lem:mutilatedGraphInclusion}, it holds that $(Y;Z|X,W)_{G_{\overline{X}\overline{Z(W)}}}$. By the soundness of do-calculus for causal BNs, we get that $P(Y\mid do(X),do(Z),W)=P(Y\mid do(X), W) $. 
\end{proof}

\begin{proof}[Proof of Theorem \ref{thr:soundeness-do-calculus-summaryDAGs}]
	If $(Y;Z|X,W)_{H_{\overline{X}}} $, then by Corollary~\ref{corr:mutilatedSummaryDAG}, it holds that $(\bY;\bZ|\bX,\bW)_{G_{\overline{X}}}$. By the soundness of do-calculus for causal BNs, we get that $P(\bY \mid do(\bX),\bZ,\bW)=P(\bY\mid do(\bX),\bW)$.
	If $(Y;Z|X,W)_{H_{\overline{X}\underline{Z}}}$,  then by Corollary~\ref{corr:mutilatedSummaryDAG}, it holds that $(\bY;\bZ|\bX,\bW)_{G_{\overline{X}\underline{Z}}}$. By the soundness of do-calculus for causal BNs, we get that $P(\bY\mid do(\bX),do(\bZ),w)=P(\bY\mid do(\bX),\bZ, \bW)$.
	Finally, if $(Y;Z|X,W)_{H_{\overline{X}\overline{Z(W)}}}$, then by Corollary~\ref{corr:mutilatedSummaryDAG}, it holds that $(\bY;\bZ|\bX,\bW)_{G_{\overline{X}\underline{Z(W)}}}$. By the soundness of do-calculus for causal BNs, we get that $P(\bY \mid do(\bX),do(\bZ),\bW)=P(\bY\mid do(\bX), \bW) $.
\end{proof}

\begin{proof}[Proof of Theorem \ref{thr:completness-do-calculus}]
	Consider $G_H$, the grounded-DAG of $(H,f)$, that is, by definition, compatible with $H$.
	If $Y$ is $d$-connected to $Z$ in $H_{\overline{X}}$ with respect to $X\cup W$, then by Definition~\ref{def:groundedDAG}, it holds that $y$ is $d$-connected to $z$ in $G_{H_{\overline{f(X)}}}$ with respect to $f(X\cup W)$, for every $y\in f(Y)$ and $z\in f(Z)$. 
	Therefore,  $f(Y)$ is $d$-connected to $f(Z)$ in $G_{H_{\overline{f(X)}}}$ with respect to $f(X\cup W)$.
\end{proof}

\section{Handling Mixed Graphs}
\label{subsec:mix_graphs}
While one dominant form of graph input for causal inference is a causal DAG, other graph representations are also used when a full causal DAG is not retrievable, say, by a causal discovery algorithm (e.g., \cite{pena2016learning}).
Many of these graph representations are referred to as \textit{mixed graphs} due to their inclusion of undirected, bidirected, and other types of edges \cite{chickering2002optimal, perkovic2020identifying}.

One commonly used of mix graph is an acyclic-directed mixed graph (ADMG), which consists of a DAG with bidirected edges. As mentioned in the introduction, all of our results apply to scenarios where the input graph is an ADMG. Subsequently, we present an extension to the \algoName\ algorithm to accommodate an ADMG.

In this scenario, we modify the cost function (Algorithm \ref{algo:cost}) as follows: When we remove a bidirected edge between nodes $U$ and $V$ (i.e., $U \leftrightarrow V$) by merging $U$ and $V$ into a single node, the cost incurred is doubled compared to removing a "standard" directed edge. For instance, in line 4 of Algorithm \ref{algo:cost}, if $\mathbf{U}$ and $\mathbf{V}$ were linked by a bidirected edge, the line would be updated to:
$$cost \gets cost + 2\cdot size(\mathbf{U})\cdot size(\mathbf{V})$$
This adjustment is necessary because losing a bidirected edge should carry a higher cost than losing a regular directed edge, given that more information is lost.

\end{document}